\documentclass[twoside]{article}

\usepackage[accepted]{"aistats2021"}

\usepackage[round]{natbib}

\usepackage[utf8]{inputenc} 
\usepackage[T1]{fontenc}    
\usepackage[hidelinks]{hyperref}       
\usepackage{url}            
\usepackage{booktabs}       
\usepackage{amsfonts}       
\usepackage{nicefrac}       
\usepackage{microtype}      

\usepackage{dsfont}
\usepackage{graphicx}
\usepackage{subfigure}
\usepackage{booktabs} 
\usepackage{amsmath}
\usepackage{amssymb}
\usepackage{amsthm}
\usepackage{bbm}
\usepackage{caption}
\usepackage[ruled,vlined]{algorithm2e}
\usepackage{chngcntr} 

\newcommand{\vertiii}[1]{{\left\vert\kern-0.25ex\left\vert\kern-0.25ex\left\vert #1 
    \right\vert\kern-0.25ex\right\vert\kern-0.25ex\right\vert}}

\newcommand{\Prob}{\mathbb{P}}
\newcommand{\bs}{\boldsymbol}

\newcommand\numberthis{\addtocounter{equation}{1}\tag{\theequation}}

\newcommand{\glasso}{graphical lasso}

\newcommand{\thav}{thAV}
\newcommand{\av}{AV}
\newcommand{\stars}{StARS}

\newcommand{\oracle}{oracle}

\DeclareMathOperator*{\argmin}{\arg\!\min}

\DeclareMathOperator*{\Sd+}{\mathcal{S}_d^+}

\DeclareMathOperator*{\tr}{\operatorname{tr}}

\usepackage{color}
\usepackage[normalem]{ulem}

\renewcommand{\cite}[1]{\citep{#1}}

\newtheorem{lem}{Lemma}
\newtheorem{cor}[lem]{Corollary}
\newtheorem{defi}[lem]{Definition}
\newtheorem{theorem}[lem]{Theorem}

\newtheorem{prop}[lem]{Proposition}

\newtheorem{assumption}[lem]{Assumption}

\title{Thresholded Adaptive Validation:\\Tuning the Graphical Lasso for Graph Recovery}

\begin{document}
%
\runningtitle{Thresholded Adaptive Validation: Tuning the Graphical Lasso for Graph Recovery}

%
\runningauthor{Mike Laszkiewicz, Asja Fischer, Johannes Lederer}

\twocolumn[

\aistatstitle{Thresholded Adaptive Validation:\\Tuning the Graphical Lasso for Graph Recovery}

\aistatsauthor{Mike Laszkiewicz \And Asja Fischer \And  Johannes Lederer }

\aistatsaddress{ Department of Mathematics, Ruhr University
Bochum, Germany } ]

\begin{abstract}
Many Machine Learning algorithms are formulated as regularized optimization problems, but their performance hinges on a regularization parameter that needs to be calibrated to each application at hand. In this paper, we propose a general calibration scheme for regularized optimization problems and apply it to the graphical lasso, which is a method for Gaussian graphical modeling.  
The scheme is equipped with theoretical guarantees and motivates a thresholding pipeline that can improve graph recovery. 
Moreover, requiring at most one line search over the regularization path, the calibration scheme is computationally more efficient than competing schemes that are based on resampling.
Finally, we show in simulations that our approach can improve on the graph recovery of other approaches considerably. 
\end{abstract}

\section{Introduction}
Over the last decades full of technical achievements, we experienced a revolutionary supply of data, confronting us with large-scale data sets. In order to handle and to infer new insights from the appearing wealth of data it presupposes us to put effort into the development of new, scaleable procedures.
One approach to address this problem is using graphical models, which proved to serve as an intuitive, easy-understanding visualization of the underlying interaction network that can then be further analyzed.
Typical applications for graphical models occur in several modern sciences, including genetics \cite{.2004}, the analysis of brain connectivity networks \cite{Bu.20170420}, and the investigation of complex financial networks \cite{Denev.2015}.
In all of these cases, graphical models can reduce the network of interactions to its relevant parts to lighten the challenge of high-dimensionality. While domain experts can analyze and interpret the structure of interactions across features, we can use this information for more accurate model building. Using a sparse representation of the relevant parts of the model, it is possible to develop more efficient inference algorithms and accelerate sampling from the model. 
A popular approach to face this challenge is to recover the network from the data using \textit{undirected graphical models}. We call this task \textit{graph recovery}. 

An important class of undirected graphical models are \textit{Gaussian graphical models}.
There are numerous estimators for Gaussian graphical models, including those that account for high dimensionality, e.g.~the graphical lasso \cite{Yuan.2007, Banerjee.20070704, Friedman.2008}, SCAD \cite{Fan.2009}, and MCP \cite{Zhang.2010}, which are based on the idea of regularized maximum likelihood estimation, and various other approaches such as neighborhood regression \cite{Meinshausen.2006, Sun}, TIGER \cite{Liu.20120911}, and SCIO \cite{scio}. 
These estimators reduce the effective dimensionality of the model through a regularization term that is adjusted to the setting at hand with a regularization parameter.

In this paper we generalize the theoretical framework  of \citet{Chichignoud.20161108} 
and utilize the large body of preliminary theoretical work \cite{Ravikumar.20081121} to verify that we can apply this general scheme to the \glasso{}.
Important features of the resulting data-driven calibrated estimator are that it comes with a finite sample upper bound on the approximation error 
and is computationally efficient as it requires at most one graphical lasso solution path. 
We equip the estimator with a simple, theory-based threshold and observe a significant improvement over other methods in its graph recovery performance. 

The rest of the paper is organized as follows. 
Section~\ref{sec:glasso} briefly reviews Gaussian graphical models and presents the proposed estimator for graph recovery, which we call \textit{thresholded adaptive validation \glasso{}} (\thav{}).
Based on an empirical analysis of toy data sets presented in Section~\ref{simus}, we demonstrate that the \thav{} outperforms existing methods on the graph recovery tasks. We apply the \thav{} to real-world data to recover biological networks in Section~\ref{sec:application}. Finally, we conclude with a discussion and present ideas for future research in Section~\ref{sec:discussion}. The supplement contains theoretical results, proofs, and further simulations. The code for our experiments is provided and can further be accessed through our public git repository\footnote{\url{https://github.com/MikeLasz/thav.glasso}}.

\section{Thresholded Adaptive Validation for the Graphical Lasso} \label{sec:glasso}
We begin by giving a brief review of Gaussian graphical models and describing the \glasso{} optimization problem in Section~\ref{sec:intro}. In Section~\ref{sec:method} we apply the adaptive validation (AV) calibration scheme, which was originally proposed by \citet{Chichignoud.20161108} and which we generalize to general regularized optimization problems in Section~A.1 of the supplement, to the \glasso{}. We obtain finite sample results for the $\ell_\infty$-loss on the off-diagonals of the \glasso{}, and employ these bounds to motivate a thresholded \glasso{} approach. 

\subsection{Brief Review of Gaussian Graphical Models} \label{sec:intro}
An undirected graphical model expresses the conditional dependence structure between components of a multivariate random variable. More precisely, given a high-dimensional random variable $\bs{z}\in\mathbb{R}^d$, the undirected graphical model depicts for each pair of components $z_i,\, z_j$ of $\bs{z}$ if these are independent given the remaining $d-2$ components of $\bs{z}$ (i.e.~ $z_i \perp z_j \vert \bs{z}_{\backslash  \{i,j\}}$).
Formally, an undirected graphical model is defined as a pair $(\bs{z},  \mathcal{G})$, where $\mathcal{G}:= ( \mathcal{V},  \mathcal{E})$ is a graph with vertices $\mathcal{V}:= \{1,...,d\}$ and edge set $\mathcal{E}:= \{ (i,j) \in \mathcal{V} \times \mathcal{V}: \; z_i \not\perp z_j \vert \bs{z}_{\backslash \{i,j\}}\}$. 
It is well-known that in a Gaussian graphical model, i.e.~in the case that $\bs{z}\sim\mathcal{N}_d(\bs{0}_d,\,\Sigma)$, where $\Sigma$ is the positive definite covariance matrix, we can find an elegant characterization of the conditional dependency structure. It can be seen as a special case of the \textit{Hammersley-Clifford Theorem} \cite{Grimmett.1973, Besag.1974, Lauritzen.1996}: for any $i\neq j \in \mathcal{V}$ it holds that
\begin{equation}\label{hammersley}
    z_i \perp z_j \vert \bs{z}_{\backslash \{i,j\}} \quad \Leftrightarrow \quad \Theta_{ij} = 0 \enspace,
\end{equation}
where $\Theta:= \Sigma^{-1}$ is the so called \textit{precision matrix}. 
Hence, in order to estimate the  conditional dependence graph $\mathcal{G}$, one can build on an estimate $\hat{\Theta}$ of
the precision matrix $\Theta$ and define $\hat{\mathcal{E}}:= \{ (i,j) \in \mathcal{V}\times \mathcal{V}: \; \hat{\Theta}_{ij} \neq 0 \}$.

Given $n$ samples $\bs{z}^{(1)},\dots ,\, \bs{z}^{(n)}$ drawn independently from $\mathcal{N}_d(\bs{0}_d,\,\Sigma)$, an evident approach to estimate $\Theta$ is to employ maximum likelihood estimation. 
But it is well-known that its performance suffers in the high-dimensional setting where $n \approx d$ or even $n<d$, and that it does not exist in the latter setting \cite{Wainwright.2019}.
A typical approach to overcome the burdens that come with high-dimensionality is to assume a sparsity structure on the target, that is, to assume $\Theta$ to have many zero-entries. This does not only improve theoretical guarantees but also makes the conditional dependence graph more interpretable. Moreover, imposing a sparsity structure is in accordance with the scientific beliefs in typical application areas in which graphical models are being used \cite{Thieffry.1998, Jeong.2001}.
The probably most-frequently used sparsity encouraging estimation procedure for Gaussian graphical models is the \textit{graphical lasso} \cite{Yuan.2007} 
\begin{align}
    \hat{\Theta}_r = \argmin_{\Omega \in \Sd+} \biggl\{  \operatorname{tr}\biggl[ &\frac{1}{n} \sum_{i=1}^n \bigl( \bs{z}^{(i)}\bigr)^\top \bs{z}^{(i)} \Omega \biggr] \nonumber \\ &- \log\left[ \det [ \Omega ] \right] + r\Vert \Omega\Vert_{1,\operatorname{off}} \biggr\} \enspace, \label{glasso} 
\end{align}
where $\Sd+$ is the set of positive definite and symmetric matrices in $\mathbb{R}^{d\times d}$, $\operatorname{tr}$ denotes the trace, $r$ is a problem-dependent regularization parameter, and $\Vert \Omega \Vert_{1,\operatorname{off}} := \sum_{i\neq j}\vert \Omega_{ij} \vert$ denotes the $\ell_1$-norm of $\Omega\in \Sd+$ on its off-diagonal.
Of course, the performance of the estimator hinges on the choice of~$r$, and while general theoretical results for the \glasso{} exist (e.g. those presented by \citet{Zhuang18}), to the best of our knowledge there are none that allow for a sophisticated choice of $r$ for graph recovery tasks that occur in practice.

\subsection{Thresholded Adaptive Validation}\label{sec:method}
In this section, we transfer the AV technique proposed by \citet{Chichignoud.20161108} for the lasso to the graphical lasso. As can be seen from our derivation in Section~A.1 of the supplement, the technique can be applied to tune any general regularized optimization problem over a set $\mathbb{S}$ of the form 
\begin{equation}
\label{eq:generalOP}
    \hat{\Theta}_r \in \argmin_{\Omega \in \mathbb{S}} \bigl\{ f(\bs{Z}, \Omega) + r h(\Omega)\bigr\} \enspace,
\end{equation}
where $r$ is a real-valued regularization parameter, $f$ is a function measuring the fit of the estimator $\Omega$ given the observed data $\bs{Z}$, and $h$ is some regularization function depending only on $\Omega$. 
Comparing~\eqref{glasso} with~\eqref{eq:generalOP} shows that the \glasso{} belongs to this class of
regularized optimization problems. We can therefore apply the calibration scheme proposed by \citet{Chichignoud.20161108} and which we generalized in the supplement to obtain the following definition: 
\begin{defi}[\av{}] \label{avdefi}
Let $\mathcal{R}$ be a finite and nonempty set of regularization parameters. Then, the adaptive validation (AV) calibration scheme selects the regularization parameter according to
\begin{align*}
    \hat{r} := \min \bigl\{r\in \mathcal{R}: \;  \ell\bigl(\hat{\Theta}_{r^\prime}&,\hat{\Theta}_{r^{\prime\prime}}\bigr) \leq C (r^\prime + r^{\prime \prime}) \\ & \quad \quad
    \forall \,r^\prime,r^{\prime\prime} \in \mathcal{R}\cap [r, \infty) \bigr\} \enspace ,
\end{align*}
where $\ell:\mathcal{S}_d^+ \times \mathcal{S}_d^+ \rightarrow \mathbb{R}$ is the $\ell_\infty$-distance on the off-diagonals, $C\in \mathbb{R}$ is a constant (specified in the following), and  $\hat{\Theta}_{r^\prime},\, \hat{\Theta}_{r^{\prime\prime}}$ are the \glasso{} estimators~\eqref{glasso} using regularization parameter $r^\prime$ and $r^{\prime\prime}$, respectively. We call $\hat{\Theta}_{\hat{r}}$ (resulting from inserting $\hat{r}$ into~\eqref{glasso}) the \av{} estimator.
\end{defi}

The constant $C$ stems from an assumption, which the theory of \citet{Chichignoud.20161108} relies on, namely
 that there exists this constant and a class of events $(\mathcal{T}_r)_{r\in \mathcal{R}}$, which are increasing in $r$, such that conditioned on $\mathcal{T}_r$ it holds
\begin{equation}
    \label{asu}
    \ell(\Theta, \hat{\Theta}_r)\leq Cr \enspace .
\end{equation}
Particularly, we only require the existence of the set of events $(\mathcal{T}_r)_{r \in \mathcal{R}}$ for some fixed $C$ and do not need to have access to it.
As demonstrated by Theorem~8 in Section~A.2 of the supplement, we can show
based on the investigations of \citet{Ravikumar.20081121}
that this assumption holds true for the \glasso{}\footnote{Note that even though we obtain a class of events $\bigl( \mathcal{T}_r\bigr)_{r\in \mathcal{R}}$ building on a similar interpretation as \citet{Chichignoud.20161108}, we have to resort to a more involved primal-dual-witness construction to prove the validity of this upper bound.}.
Motivated by this, the smallest regularization parameter $r_\delta^\ast$ that enables us to apply~\eqref{asu} with probability $1-\delta$, for some $\delta\in (0,1]$, can be seen as a natural candidate for $r$: 
\begin{equation*}
    r_\delta^\ast := \argmin_{r \in \mathcal{R}}\left\{ \Prob\bigl( \mathcal{T}_r\bigr) \geq 1 - \delta\right\} \enspace .
\end{equation*}
However, $r_\delta^\ast$ is inaccessible in practice, since we usually cannot measure $\Prob(\mathcal{T}_r)$.
Nonetheless, the AV estimator $\hat{\Theta}_{\hat{r}}$ also results in a good approximation of the precision  matrix as guaranteed by the following theorem, which 
is based on the generalization of the Theorem 3 of \citet{Chichignoud.20161108} \footnote{The generalized version we derived corresponds to Theorem 3 in Section~A.1 of the supplement.} applied onto the \glasso{}. 
\begin{theorem}[Finite Sample Bound for the \av]\label{avtheo}
    Suppose that 
    $\hat{r}$ is the regularization parameter selected by the AV calibration scheme and $C$ is the constant from~\eqref{asu}. Then, for any $\delta \in (0,1]$, it holds that 
    \begin{equation}
        \hat{r} \leq r_{\delta}^\ast \quad \textrm{ and } \quad \ell\bigl(\Theta, \hat{\Theta}_{\hat{r}}\bigr) \leq 3Cr_{\delta}^\ast \label{avbound}
    \end{equation}
    with probability at least $1-\delta$.
\end{theorem}

For simplicity of notation, let us denote the \av{} estimator by $\hat{\Theta}:=\hat{\Theta}_{\hat{r}}$ from here on.
The finite sample upper bound~\eqref{avbound} immediately implies that it holds with probability $1-\delta$ that 
\begin{enumerate}
    \item for any zero entry $\Theta_{ij}=0$ of the true precision matrix the corresponding  entry $\hat{\Theta}_{ij}$ of the AV estimate satisfies $\vert \hat{\Theta}_{ij}\vert \in [0, 3Cr_\delta^\ast]$;
    \item for any significant non-zero entry $\Theta_{ij}$ with $\vert \Theta_{ij}\vert > (3 + \lambda )Cr_\delta^\ast$, for some constant $\lambda$, the corresponding entry $\hat{\Theta}_{ij}$ of the AV estimate is also non-zero with $\vert \hat{\Theta}_{ij} \vert > \lambda Cr_\delta^\ast$.
\end{enumerate}
These observations suggest a strategy for efficient graph recovery: by including all edges $(i,j)$ to the edge set that satisfy $\vert \hat{\Theta}_{ij}\vert> \lambda Cr_\delta^\ast$, we make sure that we recover all significant entries (see 2.). Pursuing this strategy, we can also shrink the interval, in which \av{} missclassifies zero entries to $[\lambda C r_\delta^\ast,\, 3C r_\delta^\ast]$ (see 1.). 
However, as $r_\delta^\ast$ is inaccessible, we propose to replace it in the selection strategy by the AV regularization parameter, leading to the thresholded estimator defined in the following.
\begin{defi}[\thav{}]
    Let $\hat{\Theta}$ be the $\av{}$ estimator. Then, we define the thresholded adaptive validation \glasso{}  (\thav{}) estimator by 
    \begin{equation}\label{thresholdedglasso}
        \bigl( \hat{\Theta}^t \bigr)_{ij} := \bigl(\hat{\Theta}_{ij} \mathds{1}_{\{\vert \hat{\Theta}_{ij} \vert > t\}} \bigr)_{ij} \enspace,
    \end{equation}
    where $t:=\lambda C\hat{r}$ is the threshold, $\lambda\in (0,3]$, and $\mathds{1}_{A}$ is the indicator function over a set $A$. The resulting estimated edge set is then 
    \begin{align*}
        \hat{\mathcal{E}}:&= \bigl\{ (i,j) \in \mathcal{V} \times \mathcal{V}: \; \hat{\Theta}_{ij}^t \neq 0 \bigr\} \\ &= \bigl\{ (i,j)\in \mathcal{V} \times \mathcal{V}: \; \vert \hat{\Theta}_{ij}\vert > \lambda C \hat{r}\bigr\} \enspace.
    \end{align*}
\end{defi}
As we know from Theorem 2 that $\hat{r} \leq r_\delta^\ast$ with probability $1-\delta$, we can use the above observations to derive the following corollary that guarantees outstanding graph recovery properties of the \thav{}:
\begin{cor}[Finite Sample Graph Recovery] \label{supportrecov}
    Let $\hat{\Theta}$ be the \av{} estimator and $\hat{\Theta}^t$ the \thav{} estimator with $t=\lambda C\hat{r}$, where $C$ is the constant from Assumption~\eqref{asu} and $\lambda \in (0,3]$. Then, it holds with probability $1-\delta$ that 
    \begin{enumerate}
        \item for all $(i,j)\in\mathcal{V}$ such that $\Theta_{ij}=0$ it is ${\vert \hat{\Theta}_{ij}\vert \in [0,\, 3Cr_\delta^\ast]}$ and therefore
        \begin{equation*}
            (i,j)\in \hat{\mathcal{E}} \iff \vert \hat{\Theta}_{ij}\vert \in (\lambda C \hat{r},\, 3Cr_\delta^\ast] \enspace .
        \end{equation*}
        \item for all $(i,j)\in\mathcal{V}$ such that $\vert \Theta_{ij}\vert > (3+\lambda )Cr_\delta^\ast$ it is $(i,j)\in\hat{\mathcal{E}}$. 
    \end{enumerate}
\end{cor}
The proof of Corollary~\ref{supportrecov} can be found in Section~A.3 of the supplement. As far as we know, there is no other theoretical result so far that justifies a specific choice for a threshold in a thresholded version of the \glasso{}. Moreover, the corollary offers a theoretical ground for balancing the tradeoff between false positive and false negative rate: 
while maintaining finite-sample guarantees, we can regulate $\lambda$ according to our needs to decrease the false negative rate (part $2$) at the cost of increasing the interval $(\lambda C \hat{r}, 3C r_\delta^\ast]$ in which \thav{} missclassifies negatives (part $1$).

Importantly, the \thav{} also comes with notable computational benefits, since the computations in Definition~\ref{avdefi} only require at most $1$ solution path. Using the \emph{glasso} R package \cite{Friedman.2008, Witten.2011}, we can efficiently compute the \thav{} as described by Algorithm~1 in Section~B of the supplement.

Finally, note that even though there exist theoretical bounds justifying~\eqref{asu} (see Section~A.2 of the supplement), they are usually too loose or bounded to restrictions that are hard to interpret and violated in practice. Thus, what we observe in practice is usually not~\eqref{asu}, but rather a more robust ``quantiled version'' of it:
\begin{equation*}
    \ell_{1-\alpha}(\Theta, \hat{\Theta}) \leq Cr \enspace,
\end{equation*}
where $\ell_{1-\alpha}(\Theta, \hat{\Theta})$ defines the $1-\alpha$ quantile of the set of absolute differences $\{\vert \Theta_{ij}-\hat{\Theta}_{ij} \vert\}_{ij}$. Under this assumption, one could derive the same theory for the quantiled version of the loss, i.e. by replacing all $\ell$ by $\ell_{1-\alpha}$ in this section (compare with the general theory in Section~A.1 of the supplement). However, in practice the results are very similar and our method is computationally less expensive.

\section{Simulation study}\label{simus}
In this section we compare the \thav{} to various other commonly used methods to estimate a Gaussian graphical model, which are the \stars{} \cite{Liu.2010b}, the scaled lasso \cite{Sun}, the TIGER \cite{Liu.20120911}, the regularized score matching estimator (rSME) tuned via eBIC \cite{rsme}, and the SCIO\footnote{The regularized score matching estimator (rSME) and the SCIO estimator solve the same optimization problem in the Gaussian setting.} tuned via CV and via the Bregman-criterion \cite{scio}\footnote{We have also evaluated RIC, which is the default \glasso{} calibration scheme in the \emph{huge} R package \cite{Zhao.2012}. However, we decided to exclude it from our simulation study due to bad results, computational instability, and a lack of theory.}.
We sample synthetic data from a Gaussian distribution $\mathcal{N}_d(\bs{0}_d, \, \Theta^{-1})$, whereby we adopt a similar precision matrix generation procedure from \citet{caballe2015selection} for sampling random and scale-free graphs.
A detailed description of the based generation process can be found in Section~C.1 of the supplement. 
We scale the data such that it is centered and has empirically unit variance.

If not stated differently, we use $t=C\hat{r}$ and $C=0.7$ in the following. We define the set of possible regularization parameters to be $\mathcal{R}:=\{0.05 + i(r_{\operatorname{max}} - 0.05)/40 : \; i\in \{1,...,40\}\}$, where $r_{\operatorname{max}}:= \max_{i\neq j} \vert \hat{\Sigma}_{ij}\vert $ is the largest off-diagonal entry in absolute value of the empirical covariance matrix\footnote{This is the smallest regularization parameter that estimates an empty graph.}. To enhance the robustness of our algorithm, we scale the graphical lasso estimators in Definition~\ref{avdefi} such that they have unit diagonal entries. If not specified, the results of all experiments are averaged over $25$ iterations and standard deviations are shown in parenthesis.
Besides the experiments presented in this paper, we present additional investigations and repeat all experiment in various settings in Section~C.3 of the supplement.
We provide the code for all experiments with the submission, which can further be accessed through the public git repository.
\paragraph{Performance in $F_1$-score}
\label{sec:simuf1}
\begin{table*}[t]
    \caption{Graph recovery performance for varying graphs and sample size. The bold numbers indicate the best score in each setting.}
    \label{f1simus}
    \vskip 0.15in
    \begin{center}
    \begin{small}
    \begin{sc}
    \begin{tabular}{lcccccccc}
        \toprule
         && \multicolumn{3}{c}{random} & \phantom{abc} & \multicolumn{3}{c}{scale-free}\\
          \cmidrule{3-5} \cmidrule{7-9} \\
         && $F_1$ & Precision & Recall & & $F_1$ & Precision & Recall\\
        \midrule
        $n=300,\, d=200$ &&&&&&&&\\
        oracle && 0.70 (0.13) & 0.60 (0.16) & 0.89 (0.03) &&  0.37 (0.12) & 0.37 (0.22) & 0.63 (0.23) \\ 
        StARS && 0.59 (0.14) & 0.44 (0.13) & 0.93 (0.09) && 0.29 (0.13) & 0.20 (0.10) & 0.65 (0.12) \\ 
        scaled lasso && 0.68 (0.02) & 0.52 (0.03) & 0.98 (0.01) && 0.40 (0.07) & 0.26 (0.05) & 0.84 (0.07) \\
        TIGER && 0.47 (0.09) & 0.31 (0.08) & $\bs{0.99}$ (0.01) && 0.34 (0.07) & 0.21 (0.05) & $\bs{0.87}$ (0.07) \\
        rSME (eBIC) && 0.64 (0.17) & 0.49 (0.16) & 0.98 (0.01) && 0.47 (0.23) & 0.42 (0.24) & 0.74 (0.14) \\
        scio (CV) && 0.19 (0.36) & 0.23 (0.41) & 0.17 (0.34) &&  0.15 (0.19) & 0.44 (0.50) & 0.09 (0.12) \\
        scio (Bregman) && 0.13 (0.17) & 0.11 (0.22) & 0.96 (0.16)  && 0.24 (0.18) & $\bs{0.51}$ (0.44) & 0.57 (0.37) \\
        thAV && $\bs{0.91}$ (0.03) & $\bs{0.90}$ (0.04) & 0.93 (0.05) && $\bs{0.54}$ (0.13) & 0.48 (0.19) & 0.70 (0.13) \\
        \midrule
        $n=200, \, d=300$ &&&&&&&&\\
        oracle && 0.70 (0.10) & 0.63 (0.14) & 0.81 (0.03) &&  0.29 (0.07) & 0.25 (0.15) & 0.47 (0.13) \\ 
        StARS && 0.54 (0.11) & 0.39 (0.11) & 0.93 (0.03) && 0.25 (0.07) & 0.17 (0.06) & 0.54 (0.10) \\ 
        scaled lasso && 0.65 (0.03) & 0.49 (0.02) & 0.94 (0.02) && $\bs{0.30}$ (0.04) & 0.20 (0.03) & 0.59 (0.08) \\
        TIGER && 0.45 (0.08) & 0.30 (0.07) & 0.96 (0.02) && 0.25 (0.05) & 0.15 (0.04) & 0.67 (0.09) \\
        rSME (eBIC) && 0.02 (0.00) & 0.01 (0.00) & $\bs{0.99}$ (0.01) && 0.01 (0.00) & 0.01 (0.00) & $\bs{0.95}$ (0.02) \\
        scio (CV) && 0.00 (0.00) & 0.00 (0.00) & 0.00 (0.00) &&  0.06 (0.09) & 0.36 (0.49) & 0.03 (0.05) \\
        scio (Bregman) && 0.26 (0.29) & 0.27 (0.35) & 0.86 (0.21)  && 0.16 (0.12) & $\bs{0.57}$ (0.47) & 0.36 (0.27)  \\
        thAV && $\bs{0.79}$ (0.09) & $\bs{0.73}$ (0.15) & 0.90 (0.04) && 0.28 (0.11) & 0.21 (0.14) & 0.60 (0.11) \\
        \midrule
        $n=400, \, d=200$ &&&&&&&&\\
        oracle && 0.74 (0.13) & 0.65 (0.17) & 0.91 (0.03) &&  0.42 (0.10) & 0.42 (0.22) & 0.67 (0.26) \\ 
        StARS && 0.63 (0.13) & 0.48 (0.14) & 0.96 (0.03) && 0.34 (0.12) & 0.23 (0.09) & 0.70 (0.14)  \\ 
        scaled lasso && 0.70 (0.03) & 0.54 (0.03) & 0.99 (0.01) && 0.44 (0.05) & 0.29 (0.04) & 0.90 (0.07) \\
        TIGER && 0.48 (0.09) & 0.32 (0.08) & 0.99 (0.01) &&  0.33 (0.05) & 0.20 (0.04) & $\bs{0.93}$ (0.06) \\
        rSME (eBIC) && 0.56 (0.20) & 0.41 (0.18) & $\bs{1.00}$ (0.00) && 0.54 (0.13) & 0.43 (0.15) & 0.82 (0.11) \\
        scio (CV) && 0.48 (0.45) & 0.52 (0.47) & 0.48 (0.46)  &&   0.14 (0.25) & 0.28 (0.45) & 0.10 (0.19) \\
        scio (Bregman) && 0.16 (0.16) & 0.12 (0.22) & 0.97 (0.13)  && 0.21 (0.18) & 0.35 (0.40) & 0.71 (0.36)  \\
        thAV && $\bs{0.93}$ (0.04) & $\bs{0.92}$ (0.06) & 0.95 (0.04) && $\bs{0.63}$ (0.13) & $\bs{0.59}$ (0.19) & 0.75 (0.14) \\
        \bottomrule
    \end{tabular}
    \end{sc}
    \end{small}
    \end{center}
    \vskip -0.1in
\end{table*}
Table~\ref{f1simus} shows the performance of the different methods for the graph recovery task. The performance is evaluated based on precision, recall, and the resulting $F_1$-score, which are defined in Section~C.3 of the supplement. 
The proposed \thav{} estimator not only clearly outperforms the baseline methods but it also has a noticeable advantage over the oracle \glasso{} estimator, which is the (non-thresholded) \glasso{} estimator that achieves maximal $F_1$-score among all regularization parameters. This implies that it is mandatory to apply thresholding on top of regularized optimization to obtain good graph recovery results with the \glasso{}. 
Remarkably, in estimating a random graph, we observe that the \thav{} always achieves a recall of above $0.9$ while maintaining good precision. This is in stark contrast to the other methods, which seem to overestimate the graphs resulting in a high recall but comparably low precision.
Moreover, the results indicate that scale-free graphs in general are much harder to estimate than random graphs. As it has already been reported in other works (see \citet{Liu.2011, tang2015learningscale-free} and references therein), the graphical lasso is not able to provide a good estimation of a scale-free graph because its regularization does not impose any preference for identifying hub-like structures.
Nevertheless, \thav{} remains superior to the other methods in terms of reaching the highest $F_1$-Score in most cases. Again, \thav{} can find a good balance between precision and recall, whereas methods such as StARS and TIGER are overestimating the graph, which results in comparably low precision. 

\begin{table*}[t]
    \caption{Graph recovery performance of the \thav{} with fixed sample size $n=500$ for large-scale examples. The results are based on $20$ iterations.}
    \label{table:largescale}
    \vskip 0.15in
    \begin{center}
    \begin{small}
    \begin{sc}
    \begin{tabular}{lcccccccc}
        \toprule
         && \multicolumn{3}{c}{random} & \phantom{abc} & \multicolumn{3}{c}{scale-free}\\
          \cmidrule{3-5} \cmidrule{7-9} \\
         && $F_1$ & Precision & Recall & & $F_1$ & Precision & Recall\\
        \midrule
          $d=600$ && 0.96 (0.03) & 0.97 (0.02) & 0.94 (0.05) && 0.43 (0.17) & 0.45 (0.29) & 0.61 (0.17) \\ 
          $d=700$ && 0.96 (0.02) & 0.98 (0.02) & 0.93 (0.04) && 0.42 (0.16) & 0.43 (0.25) & 0.53 (0.13) \\ 
          $d=800$ && 0.95 (0.02) & 0.97 (0.03) & 0.93 (0.04) && 0.40 (0.18) & 0.45 (0.29) & 0.52 (0.16) \\ 
          $d=900$ && 0.96 (0.02) & 0.97 (0.04) & 0.95 (0.03)  && 0.33 (0.17) & 0.33 (0.26) & 0.54 (0.16) \\  
          $d=1000$ && 0.96 (0.02) & 0.98 (0.01) & 0.93 (0.03) && 0.34 (0.17) & 0.36 (0.27) & 0.46 (0.15) \\ 
        \bottomrule
    \end{tabular}
    \end{sc}
    \end{small}
    \end{center}
    \vskip -0.1in
\end{table*}
In Table~\ref{f1simus}, it appears that the recovery performance drops with an increment of $d$, which makes sense since the number of parameters increases quadratically with $d$. However, in our next experiments (Table~\ref{table:largescale}), in which we investigate the \thav{} in the setting $d \in \{600, \dots , 1000\}$ and set $n=500$, we observe that this is surprisingly not the case when enough data, but still $n<d$, is available. The $F_1$-score for a random graph remains stable across all $d$ at an impressive value of $0.96$. In the case of a scale-free graph, the performance decays slowly, while maintaining a good trade-off between precision and recall. 
Note that the support recovery in the case $d=1000$ involves about $500\,000$ parameters. 

Moreover, the careful reader will actually realize that the proposed calibration scheme can also be employed to tune the rSME estimator. Because of the generality of our results from Section~A.1 we can employ existing results from \citet{rsme} to verify the validity of Assumption~\eqref{asu} for the rSME.
As it is shown in Section~C.4 of the supplement, the rSME calibrated with the thAV approach performs comparably to the thAV graphical lasso. This is an important observation as the rSME can be applied to estimate the conditional dependency structure of a pairwise interaction model, which is a broader model class than the class of Gaussian graphical models. 
Hence, the calibration technique and the underlying theory can naturally be extended to the non-Gaussian setting.

Furthermore, we repeat the empirical study with the modified \glasso{} proposed by \citet{Liu.2011}, which we calibrate and clip via the \thav{} technique. This estimator employs a power law regularization that encourages the appearance of nodes with a high degree and are thus better suited for scale-free graphs.
The experimental study is shown in Section~C.4 of the supplement.
\paragraph{Dependence on $C$}\label{sec:dependenceonc}
If we increase the constant $C$ in the \av{} calibration scheme, we decrease $\hat{r}$ (see Proposition 9 of the supplement) and therefore employ less regularization. Hence, the \av{} estimator is inherently related to the choice of $C$. 
We plot the performance of different \av{} estimators with varying thresholds in Figure~\ref{bestthreshold} and make two crucial observations. First, we see significant dissimilarities in the performance of the unthresholded \av{} estimators: because the calibrated regularization parameter ranges from $\hat{r}=0.23$ in the case $C=0.5$ to $\hat{r}=0.09$ in the case $C=0.8$, the $F_1$-score drops from approximately $0.70$ to $0.35$. Thus, the \av{} estimator's performance heavily depends on the choice of $C$. But after thresholding, and this is the second observation, the thresholded \av{} estimators' performance curves become very similar and reach almost the same peak. We can observe the same behaviour in the other settings, as it is shown in Section~C.3 of the supplement. Importantly, we also show in the supplement that we do not observe a similar performance peak if we threshold the unregularized optimization problem (setting $r=0$ in~\eqref{glasso}).
Hence, neither regularization via regularized optimization is sufficient for graph recovery, see the performance of the oracle estimator in Table~\ref{f1simus}, nor does unregularized thresholding yield to good results. Therefore we claim that it is necessary to apply both types of regularizations, as it is done by \thav{} and which additionally encourages stability in $C$.

To further investigate the stability of the \thav{} in $C$, we consider pairs of \thav{} estimators resulting from different choices of $C$, which we call $\hat{\Theta}_{C^\prime}^{t^\prime}$ and $\hat{\Theta}_{C^{\prime\prime}}^{t^\prime}$. 
Table~\ref{f1stab} reports the differences between these estimators by calculating $F_1(\hat{\Theta}_{C'}^{t'}, \hat{\Theta}_{C''}^{t'})$ for a random graph. 
We do not only achieve a high $F_1(\Theta, \, \hat{\Theta}_C^t)$ for any $C$, but also 
the different estimates are all very similar, i.e.~$F_1(\hat{\Theta}_{C^\prime}^t,\,\hat{\Theta}_{C^{\prime\prime}}^t) $ is always above $0.80$.
Therefore, we can confirm that the recovered graphs are stable in the choice of $C$.

\begin{figure*}
    \centering
    \includegraphics[width=\linewidth]{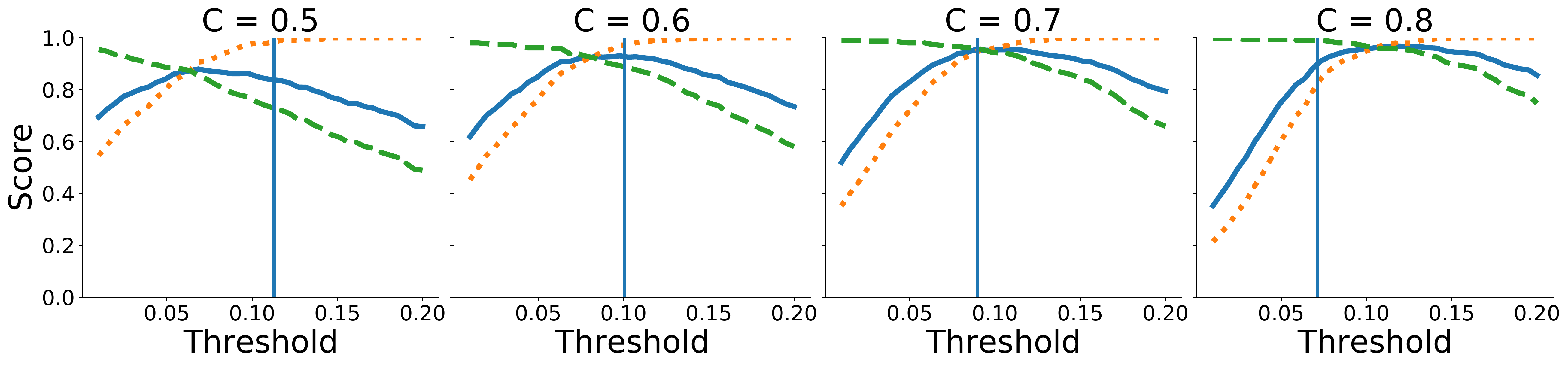}
    \caption{The $F_1$-score (blue, solid), precision (yellow, dotted), and recall (green, dashed) of a thresholded \av{} estimator for a random graph with $d=200$ based on $n=300$ samples for $C \in \{0.5,\, 0.6,\, 0.7,\, 0.8\}$ with varying thresholds. The resulting \av{} regularization parameters are $0.23, \, 0.17, \, 0.13, \, 0.09$, respectively. The vertical line depicts the proposed threshold $t=C\hat{r}$ corresponding to the \thav{} estimate.}
    \label{bestthreshold}
\end{figure*}

\begin{table}[t]
    \caption{Similarity $F_1(\hat{\Theta}_{C^\prime}^t, \hat{\Theta}_{C^{\prime \prime}}^t)$ for different choices of $C$ ($C^\prime$ and $C^{\prime \prime}$) for a random graph with $d=200$ using $n=300$ samples. The performance scores $F_1(\Theta, \hat{\Theta}_C^t)$ are $0.80 \,(0.06),\; 0.88 \, (0.04), \; 0.91\,(0.05), \; 0.85\,(0.10)$ for $C$ in $0.5, \; 0.6, \; 0.7, \; 0.8$, respectively.}
    \label{f1stab}
    \vskip 0.15in
    \begin{center}
    \begin{small}
    \begin{sc}
    \begin{tabular}{lcccccc}
        \toprule
        $C$ &&   0.6 & 0.7 & 0.8 \\
        \midrule
        0.5 & & 0.88 (0.05) & 0.82 (0.08) & 0.81 (0.10) \\ 
0.6 &&   1 & 0.93 (0.04) & 0.84 (0.09) \\ 
0.7 &&  - & 1 & 0.92 (0.05) \\ 
0.8 & & - & - & 1 \\ 

        \bottomrule
    \end{tabular}
    \end{sc}
    \end{small}
    \end{center}
    \vskip -0.1in
\end{table}

\paragraph{Varying graph density}
\begin{table}[t]
    \caption{$F_1$-score if we change the connection probability to $p\in \{2/d, \, 4/d\}$ in a random graph in various settings. The previous simulations employed $p=3/d$. The bold numbers indicate the best score in each setting.}
    \label{table:density}
    \vskip 0.15in
    \begin{center}
    \begin{small}
    \begin{sc}
    \begin{tabular}{lcccc}
        \toprule
         &&  $2/d$ & $4/d$ \\
        \midrule 
        \multicolumn{2}{l}{$n=300, \,d=200$} &&&\\
StARS && 0.55 (0.14) & 0.57 (0.13) \\ 
scaled lasso && 0.60 (0.03) & 0.73 (0.02) \\ 
TIGER && 0.39 (0.09) & 0.53 (0.08) \\ 
rSME (eBIC) && 0.51 (0.27) & 0.55 (0.23) \\ 
SCIO (CV) && 0.33 (0.40) & 0.22 (0.39) \\ 
SCIO (Bregman) && 0.23 (0.29) & 0.22 (0.23) \\ 
thAV && $\bs{0.91}$ (0.05) & $\bs{0.89}$ (0.04) \\ 
        \midrule
        \multicolumn{2}{l}{$n=200, \, d=300$} &&&\\
StARS && 0.49 (0.11) & 0.52 (0.09) \\ 
scaled lasso && 0.58 (0.03) & 0.68 (0.02) \\ 
TIGER && 0.35 (0.09) & 0.48 (0.10) \\ 
rSME (eBIC) && 0.01 (0.00) & 0.03 (0.00) \\ 
SCIO (CV) && 0.14 (0.26) & 0.20 (0.31) \\ 
SCIO (Bregman) && 0.31 (0.36) & 0.33 (0.31) \\ 
thAV && $\bs{0.87}$ (0.03) & $\bs{0.73}$ (0.11) \\ 
        \midrule
        \multicolumn{2}{l}{$n=400, \, d=200$}&&&\\
StARS && 0.61 (0.16) & 0.61 (0.12) \\ 
scaled lasso && 0.60 (0.03) & 0.74 (0.02) \\ 
TIGER && 0.38 (0.07) & 0.52 (0.06) \\ 
rSME (eBIC) && 0.55 (0.22) & 0.54 (0.14) \\ 
SCIO (CV) && 0.21 (0.39) & 0.35 (0.44) \\ 
SCIO (Bregman) && 0.14 (0.20) & 0.15 (0.01) \\ 
thAV && $\bs{0.94}$ (0.04) & $\bs{0.93}$ (0.03) \\ 
        \bottomrule
    \end{tabular}
    \end{sc}
    \end{small}
    \end{center}
    \vskip -0.1in
\end{table}

In this last experiment, we put emphasis on the adaptability of \thav{} on the density of the graph, i.e. the proportion of edges to the number of nodes in the graph. The graph's density of a random graph is controlled via $p$, the probability that a pair of nodes is being connected. Details can be found in Section~C.1 of the supplement. 
We observe from Table~\ref{table:density} that the $F_1$-score of the \thav{} estimator remains stable across all densities, whereas the other estimators tend to perform better for dense graphs. This is no surprise since we have seen in the previous experiments that the other estimators tend to overestimate the presence of edges in a graph as indicated by a high recall but low precision.
In all investigated settings, the \thav{} estimator outperforms the competing estimation procedures considerably.

\section{Applications}\label{sec:application}
Graphical model recovery plays a big role in understanding biological networks. In this section, we apply our procedure on $2$ open-source data sets to obtain sparse and interpretable graph structures. 
\paragraph{Recovering a Microbial Network}\label{sec:gut}
It is believed that the human microbiome plays a fundamental role in human health. Thus, the American Gut Project \cite{McDonald.2018} was launched to pave the way to find associations among the microbiome, but also to confirm associations between the microbiome and other aspects of human health, like psychiatric stability. 
In this example, we estimate the microbial network to enhance the understanding of the roles and the relations between the microbes. 
Since microbial datasets come with some technical problems, it is vital to preprocess the data.
We refer to the work of \citet{Kurtz.2015} and \citet{Yoon.2019} for details about the problems and a suitable preprocessing routine for microbial data.
We use the preprocessed amgut2.filt.phy data, which is included in the \emph{SpiecEasi} package \cite{Kurtz.2015}. The data set measures the abundance of microbial operational taxonomic units (OTUs) and consists of $n=296$ samples and $d=138$ different OTUs. We employ the nonparanormal transformation \cite{Liu.20090303} and estimate the microbial network using the \thav{}.
The \thav{} estimator returns a very sparse graph that identifies various clusters, see Section~D of the supplement. However, the estimator includes no interactions between different classes of microbes. To get insight about interactions across classes of microbes, we reduced the truncation parameter $\lambda$ to $0.5$. Note that the results of Corollary~\ref{supportrecov} are valid for each $\lambda \in (0, 3]$. The resulting graph is depicted in Figure~\ref{fig:gutthav}.

\begin{figure}[h!]
    \vskip 0.2in
    \centerline{
         \subfigure[]{\includegraphics[width=\columnwidth]{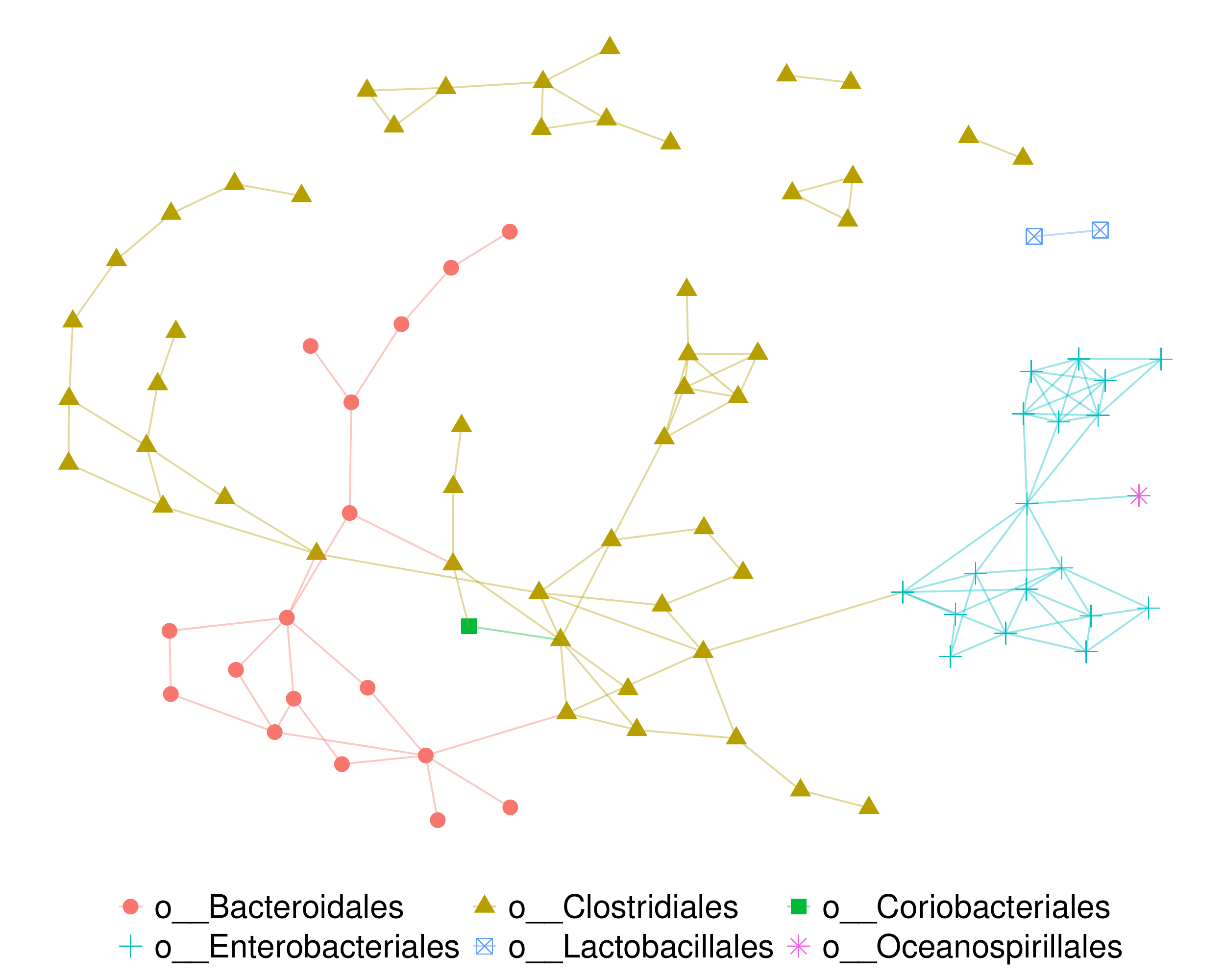}\label{fig:gutthav}}
         }
    \centerline{
        \subfigure[]{\includegraphics[width=1.4\columnwidth, height=0.7\columnwidth]{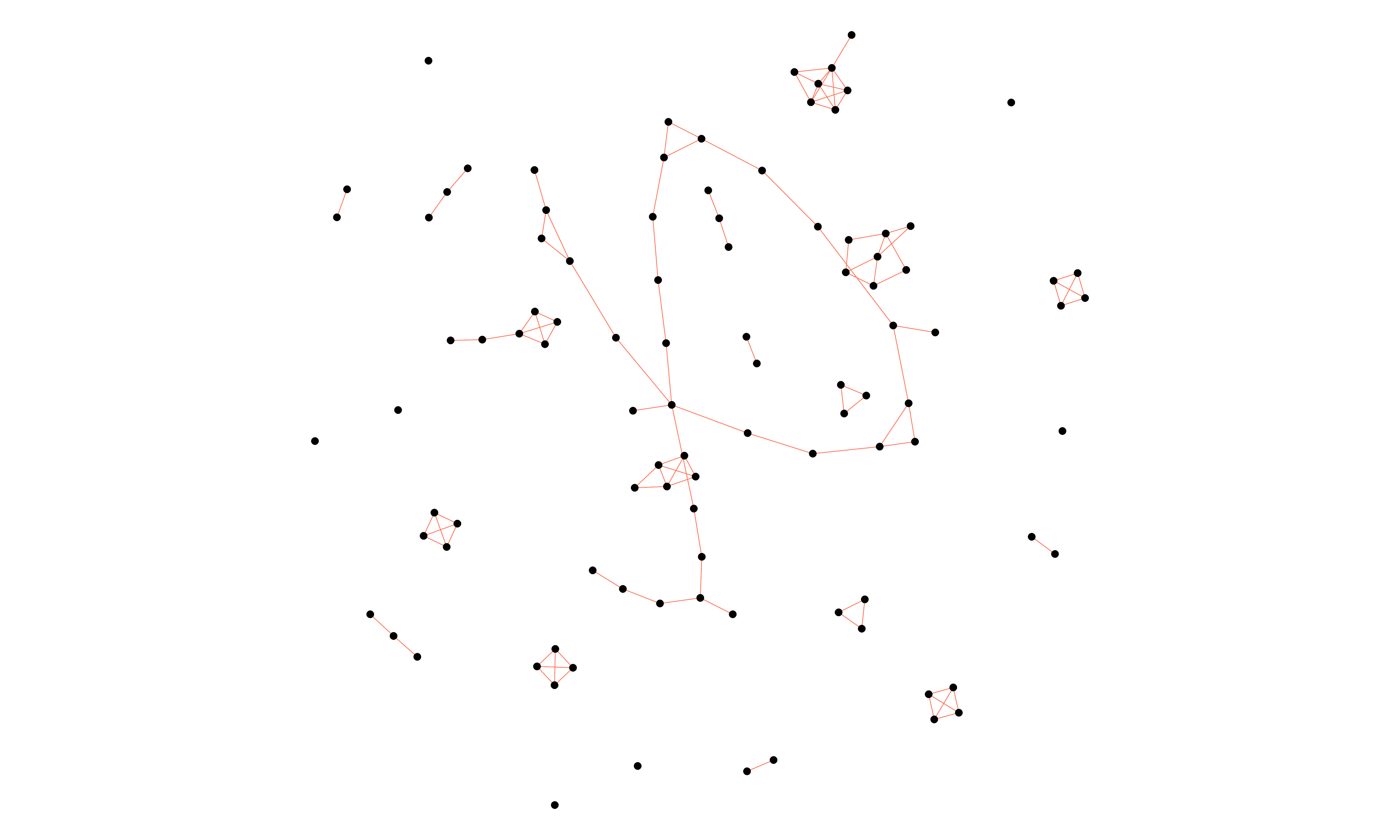}\label{fig:ribothav}}
        }
    \caption{The \thav{} based on the American Gut Data is shown in Figure~\ref{fig:gutthav}. To avoid too large graphics, we exclude isolated vertices. The color and the shape of a node imply the biological cluster 
    of each OTU. Figure~\ref{fig:ribothav} depicts the \thav{} applied on the \textit{ribloflavin} data.}
\end{figure}

\paragraph{Recovering a Gene Network}
In pharmacology, the vitamin \textit{riboflavin} is industrially produced using diverse microorganisms. 
Being able to fully understand the bacterias' genome, biologists promise to further optimize the \textit{riboflavin} production. 
The \textit{riboflavin} data set is provided by the DSM in Switzerland and contains $n=71$ samples and $d=4088$ gen expressions. The R package \emph{hdi} \cite{Dezeure.2015} provides this data set in its implementation.
We compare the results of the \thav{} (see Figure~\ref{fig:ribothav}) with those of \citet{Buhlmann.2014}, who analyzed the same data set using a neighborhood regression approach\footnote{Similar to \citet{Buhlmann.2014}, we shrink the data set by only considering the $100$ genes with the highest empirical variance and scale the data using the nonparanormal transformation.}, and observe that the \thav{}  returns a much sparser graph with more cluster-like structures. 
This does not only increase the interpretability of the graph but also imposes some tight connections between several genes within these clusters.

\section{Discussion and Conclusion} \label{sec:discussion}
Graphical models are a very popular framework for co-occurrence networks, and the graphical lasso is one of the most standard estimators in this framework. 
In this paper, we generalize the theoretical framework of \citet{Chichignoud.20161108}  for deriving a calibration scheme for the lasso and successfully transfer it to calibrate the \glasso{}. 
However, our empirical study reveals that \glasso{} estimation itself is not sufficient for effective support recovery, so an additional thresholding step becomes necessary.  
Our resulting calibration method comes with a finite sample result that allows us to derive a corollary suggesting how to choose a theoretically founded threshold in such a thresholded \glasso{} approach. The resulting estimator, which we call thresholded adaptive validation (\thav{}) estimator, 
provides a simple and fast implementation with finite sample guarantees on the recovery performance. 
To the best of our knowledge, this is the first thresholding methodology for the \glasso{} that comes with a practical, theory-based threshold.
Moreover, the \thav{} clearly outmatches existing graph recovery methods in our empirical analysis, showing both, a high recall but also a high precision in most settings. Other methods, which do not come with a practical threshold, tend to overestimate the graph. Thus, we would recommend the \thav{} as the method of choice for applications requiring an interpretable and sparse graph structure.

One shortcoming of the proposed procedure is that we replace the tuning parameter $r$ by a quantity $C$, and we even introduce $\lambda$, which defines our threshold. However, regarding $\lambda$, we derive a finite sample result for every $\lambda \in (0, 3]$. And secondly, the correspondence between $C$ and the \av{} regularization parameter leads to a threshold that regulates the impact of $C$ on the \thav{}, resulting in an estimator that is stable in $C$. In contrast, the calibration via the regularization parameter $r$ is not equipped with such a stability property.
The introduction of new hyperparameters can also not be seen as a disadvantage in comparison to related methods, which replace the regularization parameter by other hyperparameters as well.
For instance, the StARS calibration scheme for the \glasso{} introduces new parameters $N$ and $b$, which define the number of $N$ subsamples of size $b$, and the parameter $\beta$, which restricts the instability. TIGER introduces a new regularization parameter $\xi$ and the authors \cite{Liu.20120911} argue that the final problem is regularization-parameter-insensitive.

On the other hand, a major advantage of this calibration scheme is its generality. While we focus on the \glasso{} in this paper, the derived theoretical framework can also serve as a foundation for other thresholding approaches, which are not limited to Gaussian graphical modeling. 
For instance, the proposed method has the potential to effectively tune the rSME, which is a graph recovery method for the pairwise interaction model. Using the same primal dual-witness technique, the authors \cite{rsme} prove the assumption on which our theoretical framework is based (see~\eqref{asu}). Hence, we can derive the same theory for these types of estimators using the adaptive validation technique. We note that many empirical results regarding the rSME are rather limited to ROC curves, which conceal the regularization parameter selection. Moreover, we can also employ the framework for methods that aim to recover the support of specific types of graph topologies, such as particularly scale-free graphs \cite{Liu.2011}.


For future work, we hope to apply the proposed general framework to calibrate regularized optimization problems and equip these estimators with finite sample theoretical guarantees. These might include other applications of sparse precision matrix estimation such as high-dimensional discriminant analysis and portfolio allocation (see \citet{fan_overview} and references therein), but also applications beyond the scope of sparse precision matrix estimation.



\section*{Acknowledgement}
This work was supported by the Deutsche Forschungsgemeinschaft (DFG, German Research Foundation) under
Germany’s Excellence Strategy – EXC- 2092 CASA – 390781972. We thank Yannick Düren, Fang Xie, Mahsa Taheri,
Shih-Ting Huang, Ute Krämer, Björn Pietzenuk, and Lara Syllwasschy for their insightful comments. Finally, we also thank the anonymous reviewers for their careful reading and their useful suggestions.

\bibliographystyle{plainnat}
\bibliography{thav}

\renewcommand{\thesection}{\Alph{section}}
\onecolumn
\aistatstitle{Supplementary Files: \\
Thresholded Adaptive Validation:\\
Tuning the Graphical Lasso for Graph Recovery}

\aistatsauthor{Mike Laszkiewicz \And Asja Fischer \And  Johannes Lederer }

\aistatsaddress{ Department of Mathematics, Ruhr University
Bochum, Germany }

\setcounter{equation}{0}
\setcounter{figure}{0}
\setcounter{table}{0}
\renewcommand{\theequation}{S\arabic{equation}}
\renewcommand{\thefigure}{S\arabic{figure}}
\renewcommand{\bibnumfmt}[1]{[S#1]}
\renewcommand{\citenumfont}[1]{S#1}

\setcounter{section}{0}
\begin{samepage}
\section{Supplementary Theory}
This section provides theory and proofs that are used in the main paper.
First, following the theoretical framework proposed by \citet{Chichignoud.20161108}, we introduce a general calibration scheme for a broad class of regularized optimization problems in Section~\ref{sec:general}. Our general framework is based on a central assumption, which we prove for the \glasso{} in Section~\ref{sec:applyglasso}.

\subsection{A General Calibration Scheme for Regularized Optimization} \label{sec:general}
\end{samepage}
Regularizing an optimization problem has proved to serve as an useful concept in various settings to prevent overfitting (i.e. \citet{lasso}), increase numerical stability (i.e. \citet{ridge}), and because there arise attractive properties from a statistical perspective (i.e. \citet{Zhuang18, Wainwright.2019}). Further, various constrained optimization problems can be reframed as regularized unconstrained optimization problems \cite{Boyd.2009}, which are appealing from a computational perspective. However, many of these regularized optimization methods hinge in finding a suitable regularization parameter and hence it became common to fall back on heuristics or asymptotic methods. We generalize the idea in \citet{Chichignoud.20161108} to obtain a calibration scheme for a general optimization problem. The resulting estimator comes with a finite sample bound on its performance and is computationally cheap as it requires only $1$ solution path.

Consider an optimization problem of the form
\begin{equation}
    \label{generalsetting}
    \hat{\Theta}_r \in \argmin_{\Omega \in \mathbb{S}} \bigl\{ f(\bs{Z}, \Omega) + r h(\Omega)\bigr\} \enspace,
\end{equation}
where $\mathbb{S}$ is a collection of possible estimators, $r$ is a real-valued regularization parameter, $\bs{Z}$ is the observed data, $f$ is a function measuring the fit of the estimator $\Omega$ having observed data $\bs{Z}$, and $h$ is some regularization function depending only on the estimator $\Omega$. The optimization problem usually frames all quantities except of the regularization parameter, which is often tuned based on heuristics or personal intuition. Let us fix a symmetric loss function $\ell: \mathbb{S} \times \mathbb{S}\rightarrow \mathbb{R}^+$ that satisfies the triangle inequality. Our goal is then to calibrate $r$ such that we can bound the loss $\ell(\Theta,\hat{\Theta}_r)$ between the estimation target $\Theta$ and the estimator $\hat{\Theta}_r$.
%

%
We propose the following generalized version of the \av{} estimator \cite{Chichignoud.20161108}.
\begin{defi}[\av{}] \label{DefinitionAV}
Let $\mathcal{R}\subset \mathbb{R}$ be a finite and nonempty set of regularization parameters.
Then, the adaptive validation (AV) calibration scheme selects the regularization parameter according to
\begin{equation*}
    \hat{r} := \operatorname{min} \bigl\{ r \in \mathcal{R} : \ell\bigl(\hat{\Theta}_{r^\prime},\hat{\Theta}_{r^{\prime\prime}}\bigr) \leq k(r^\prime) + k(r^{\prime\prime}) \enspace \forall r^\prime,r^{\prime\prime} \in \mathcal{R} \cap [r, \infty)  \bigr\} \enspace ,
\end{equation*}
where $k:\mathcal{R}\rightarrow \mathbb{R}^+$ is a monotonly increasing and positive function (specified in the following) and $\hat{\Theta}_{r^\prime}, \hat{\Theta}_{r^{\prime\prime}}$ are the estimators \eqref{generalsetting} using regularization parameter $r^\prime$ and $r^{\prime\prime}$, respectively. We call $\hat{\Theta}_{\hat{r}}$ (resulting from inserting $\hat{r}$ into \eqref{generalsetting}) the \av{} estimator.
\end{defi}
In contrast to \citet{Chichignoud.20161108}, our general theory is not restricted to the setting of lasso regression and we modify the definition by incorporating $k$. The definition in \citet{Chichignoud.20161108} is limited to linear $k$. 
To allow for a general theory, we defer the restrictions from the specific type of optimization problem (again, \eqref{generalsetting} is very general) to the following assumption, on which our theory is based.
\begin{assumption}\label{assumptiongeneral}
    Let $\bigl( \mathcal{T}_r\bigr)_{r\in\mathcal{R}}$ be a class of events indexed by $r\in\mathcal{R}$.
    Assume that $\mathcal{T}_r$ is monotonly increasing in~$r$, that is, $
        \mathcal{T}_{r^{\prime}} \subset \mathcal{T}_{r^{\prime\prime}}$ for $r^\prime \leq r^{\prime\prime}$.
    Suppose further that, conditioned on $\mathcal{T}_r$, it holds 
    \begin{equation}
        \label{eq:asugeneral}
        \ell \bigl(\Theta,  \hat{\Theta}_r\bigr) \leq k(r)\enspace , 
    \end{equation}
    where $\hat{\Theta}_r$ is the regularized estimator \eqref{generalsetting} and $\Theta$ is the estimation target.
\end{assumption}
Hence, dependent on the setting at hand, Assumption~\ref{assumptiongeneral} defines the function $k$ that is used in the \av{} calibration scheme in Definition~\ref{DefinitionAV}. By deriving theory for arbitrary positive and monotonly increasing functions $k$, we want to emphasize that we can fit very general frameworks into the setting that is described by Assumption~\ref{assumptiongeneral} to obtain finite sample results from Theorem~\ref{avtheo}.

As in the main paper, we can define an optimal regularization parameter by the smallest regularization parameter $r_\delta^\ast$ that allows us to employ \eqref{eq:asugeneral} with probability $1-\delta$ for $\delta \in (0,1]$:
    \begin{equation}
        \label{oracletuning}
        r_\delta^\ast := \argmin_{r \in \mathcal{R}}\left\{ \Prob\bigl( \mathcal{T}_r\bigr) \geq 1 - \delta\right\}.
    \end{equation}
    
Since we do not know $\Prob(\mathcal{T}_r)$ in practice, we cannot access $r_\delta^\ast$ and must rely on approximations. The following theorem, which is a generalization of Theorem~$3$ by \citet{Chichignoud.20161108}, guarantees that the \av{} estimator comes with finite sample bounds on the approximation of the estimation target.
\begin{theorem}[Finite-Sample Bound for the \av]\label{avtheo_supp}
    Suppose that Assumption~\ref{assumptiongeneral} holds and that $\hat{r}$ is the regularization parameter selected by the AV method. Then, for any $\delta \in (0,1)$, it holds that 
    \begin{equation}
        \hat{r} \leq r_{\delta}^\ast \quad \textrm{ and } \quad \ell\bigl( \Theta , \hat{\Theta}_{\hat{r}} \bigr) \leq 3k(r_{\delta}^\ast)\label{avbound_supp}
    \end{equation}
    with probability at least $1-\delta$.
\end{theorem}
\begin{proof}
    Per definition of $r_\delta^\ast$, we know that $\Prob\bigl( \mathcal{T}_{r_\delta^\ast}\bigr) \geq 1-\delta$, hence, it is sufficient to show that \eqref{avbound_supp} holds conditioned on $\mathcal{T}_{r_\delta^\ast}$. \\
    Let 
    \begin{equation}
        \hat{r} := \operatorname{min} \bigl\{ r \in \mathcal{R} : \ell\bigl(\hat{\Theta}_{r^\prime},\hat{\Theta}_{r^{\prime\prime}}\bigr) \leq k(r^\prime) + k(r^{\prime\prime}) \enspace \forall r^\prime,r^{\prime\prime} \in \mathcal{R} \cap [r, \infty)  \bigr\} \label{avdefi_supp} \enspace,
    \end{equation}
    be the \av{}  regularization parameter.
    Our first goal is to prove that $\hat{r}\leq r_{\delta}^\ast$, which we will do by contradiction. For that, assume that $\hat{r} > r_\delta^\ast$. Then, there must exist $r^\prime, r^{\prime \prime} \in \mathcal{R} \cap [r_{\delta}^\ast,\, \infty)$ such that 
    \begin{equation}
    \ell\bigl( \hat{\Theta}_{r^{\prime}},  \hat{\Theta}_{r^{\prime \prime}} \bigr) > k(r^\prime) + k(r^{\prime \prime}) \label{contraproof}
    \end{equation}
    because otherwise $\hat{r}$ could not be a minimizer of the set in \eqref{avdefi_supp}.
    Per Assumption, we know that $\mathcal{T}_{r_\delta^\ast} \subset \mathcal{T}_{r^\prime}, \mathcal{T}_{r^{\prime\prime}}$ since $r^\prime, \, r^{\prime\prime} \geq r_{\delta}^\ast$. As we condition on $\mathcal{T}_{r_\delta^\ast}$, we know due to \eqref{eq:asugeneral} that 
    \begin{equation*}
    \ell \bigl( \Theta, \hat{\Theta}_{r^\prime} \bigr) \leq k(r^\prime) \quad \text{ and }\quad \ell \bigl( \Theta, \hat{\Theta}_{r^{\prime\prime}} \bigr) \leq k(r^{\prime \prime})\enspace. 
    \end{equation*}
    Using the triangle inequality, we conclude that 
    \begin{equation*}
    \ell\bigl( \hat{\Theta}_{r^\prime}, \hat{\Theta}_{r^{\prime \prime}} \bigr) \leq \ell \bigl( \hat{\Theta}_{r^\prime}, \Theta \bigr) + \ell \bigl( \Theta, \hat{\Theta}_{r^{\prime\prime}}\bigr) \leq k( r^\prime )+ k(r^{\prime \prime})\enspace.
    \end{equation*}
    However, this contradicts \eqref{contraproof}.  
    Consequently, it must hold that $\hat{r} \leq r_\delta^\ast$. \\
    To prove the loss bound in \eqref{avbound_supp}, we observe that it is $ \hat{r},r_\delta^\ast \in \mathcal{R}\cap [\hat{r},\infty)$ as previously shown. Therefore, we know per definition of the AV regularization parameter $\hat{r}$ that 
    \begin{align}
     \ell\bigl( \hat{\Theta}_{\hat{r}}, \hat{\Theta}_{r_\delta^\ast}\bigr)  \leq k(\hat{r}) + k(r_\delta^\ast) \enspace.  \label{eq:hatast}
    \end{align}
    Finally, using the triangle inequality once more, we find 
    \begin{align*}
    \ell\bigl( \Theta , \hat{\Theta}_{\hat{r}} \bigr) &\leq \ell \bigl( \Theta , \hat{\Theta}_{r_\delta^\ast} \bigr) + \ell \bigl( \hat{\Theta}_{\hat{r}}, \hat{\Theta}_{r_\delta^\ast}\bigr)  \\
    &\leq k(r_\delta^\ast) + k(\hat{r}) + k(r_\delta^\ast)  \\
    & \leq 3k(r_\delta^\ast) \enspace,
    \end{align*}
    where we used \eqref{eq:hatast} in the second inequality and $\hat{r} \leq r_\delta^\ast$ and the monotonicity of $k$ in the third inequality.
\end{proof}

In summary, the \av{} calibration scheme as defined in Definition~\ref{DefinitionAV} is a calibration scheme for very general optimization problems \eqref{generalsetting}. We impose almost no restrictions on the optimization problem \eqref{generalsetting} itself but defer the restrictions to Assumption~\ref{assumptiongeneral}. But once we can prove \eqref{eq:asugeneral}, which is of a very general form, we are immediately able to obtain finite sample results as stated in Theorem~\ref{avtheo_supp}. 

\subsection{Adaptive Validation for the Graphical Lasso} \label{sec:applyglasso}
In this section, we apply the proposed calibration scheme from the previous section to calibrate the \glasso{} regularization parameter. Importantly, we substantiate the validity of Assumption~\ref{assumptiongeneral} and thus obtain the finite sample results for the \av{} \glasso{} estimator (see Theorem~$2$ from the main paper).

Recall, the \glasso{} optimization problem is defined as
\begin{equation}
    \label{glasso_supp}
    \hat{\Theta}_r = \argmin_{\Omega \in \Sd+} \biggl\{  \operatorname{tr}\biggl[ \frac{1}{n} \sum_{i=1}^n \bigl( \bs{z}^{(i)}\bigr)^\top \bs{z}^{(i)} \Omega \biggr] - \log\left[ \det \left[ \Omega \right] \right] + r\Vert \Omega\Vert_{1,\operatorname{off}} \biggr\} \enspace. \numberthis
\end{equation}
Plainly, the \glasso{} optimization problem fits into the general setting in \eqref{generalsetting}. 

It remains to prove that Assumption~\ref{assumptiongeneral} is valid.
In the following, we consider 
\begin{equation*}
    k(r):= Cr \enspace \text{ and } \enspace \ell(\Theta^\prime, \Theta^{\prime \prime}):= \max_{\substack{i,j \in \mathcal{V} \\i \neq j}} \vert \Theta_{ij}^\prime - \Theta_{ij}^{\prime \prime} \vert \enspace 
\end{equation*}
for some constant $C>0$.
Hence, we aim to prove the existence of a monotone increasing class $(\mathcal{T}_r)_{r\in \mathcal{R}}$ such that, conditioned on $\mathcal{T}_r$, it holds 
\begin{equation}
    \label{eq:assumptionglasso}
    \ell\bigl( \Theta, \hat{\Theta}_r \bigr) \leq Cr \enspace ,
\end{equation}
where $\hat{\Theta}_r$ is the \glasso{}, $\Theta$ is the true precision matrix, and $C$ is a constant.

We make use of the extensive investigations by \citet{Ravikumar.20081121} to derive the validity of this assumption. 
Our result (Theorem~\ref{theoravi}), which states a $\ell$-bound as in \eqref{eq:assumptionglasso}, follows similar steps as those in the proof of Theorem~1 by \citet{Ravikumar.20081121}:
\paragraph{Required Definitions}
In order to obtain a result that substantiates \eqref{eq:assumptionglasso}, it is necessary to impose the estimation target $\Theta$ to be well-behaved. In the following we introduce some quantities that allow us to define such a well behaviour of $\Theta$. As we will see, the Hessian of the mapping
\begin{equation*}
    g:\Sd+ \rightarrow \mathbb{R}, \; \Omega \mapsto \log(\det[ \Omega])
\end{equation*}
at $\Theta$, which we define as
\begin{equation*}
    \Gamma := \frac{\partial^2 g(\Theta)}{\partial \Omega^2}\Bigg\vert_{\Omega=\Theta}\enspace
\end{equation*}
plays a central role in describing a well behaviour of the precision matrix.
We index $\Gamma$ by vertex pairs
\begin{equation} \label{gammaindex}
    \Gamma_{(j,k), (l,m)} = \frac{ \partial^2 g(\Theta)}{\partial \Omega_{jk} \partial \Omega_{lm}}\Bigg\vert_{\Omega=\Theta} \enspace,
\end{equation}
where $(j,k), (l,m)\in \mathcal{V} \times \mathcal{V}$.
Further, let $S:=\mathcal{E} \cup \{ (i,i): \, i\in \mathcal{V}\}$ be the edge set of the graphical model including the self-directed edges. Then, according to the indexation in \eqref{gammaindex}, we can define the quantities 
\begin{equation*}
    \Gamma_{SS}:= \bigl( \Gamma_{(j,k),(l,m)}\bigr)_{(j,k), (l,m)\in S} \enspace \text{ and } \enspace
    \Gamma_{eS}:= \bigl( \Gamma_{e,(l,m)} \bigr)_{(l,m)\in S} \enspace  \forall e\in S \enspace .
\end{equation*}
Finally, we define 
\begin{equation*}
    \kappa_{\Sigma} := \vertiii{\Sigma}_\infty \enspace \text{ and }  \enspace \kappa_{\Gamma} := \vertiii{ (\Gamma_{SS})^{-1}}_\infty \enspace,
\end{equation*}
where $\vertiii{\cdot}_\infty$ defines the $\ell_\infty$-operator norm, that is, 
\begin{equation*}
    \vertiii{ \Sigma}_\infty := \max_{i=1,...d}\sum_{j=1}^d \vert \Sigma_{ij}\vert \enspace.
\end{equation*}
In the following paragraph, we use these quantities to describe conditions under which the so called \textit{Primal-Dual Witness Construction} (PDW) succeeds. The idea of PDW is based on the primal solution $\tilde{\Theta}_r$, which is the solution of the \glasso{} that is constrained on the true edge set.
Hence, this is the solution of the following \glasso \,optimization problem 
\begin{equation*}
    \tilde{\Theta}_r := \argmin_{\Omega \in \Sd+, \Omega_{S^\complement}=0} \bigg\{  \operatorname{tr} \biggl[ \frac{1}{n} \sum_{i=1}^n \bigl( \bs{z}^{(i)}\bigr)^\top \bs{z}^{(i)} \Omega \biggr] - \log\bigl[ \det [ \Omega ] \bigr] + r\Vert \Omega\Vert_{1,\operatorname{off}} \bigg\} \enspace,
\end{equation*}
where $S^\complement:= \{ (i,j)\in \mathcal{V} \times \mathcal{V}: \; (i,j)\not \in \mathcal{E}\}$ and $\Omega_{S^\complement}$ are the components $\Omega_{ij}$ such that $(i,j)\in S^\complement$. As $S$ is unknown in practice, this construction is only relevant from a technical perspective. We say that the PDW succeeds if the unconstrained \glasso{} $\hat{\Theta}_r$ equals the primal solution $\tilde{\Theta}_r$.

\paragraph{Incoherence and bounded norm of the Primal}
Now we are ready to impose a well behaviour on $\Theta$ that ensures the PDW to succeed and which eventually implies the validity of \eqref{eq:assumptionglasso}, see Theorem~\ref{theoravi}. 
\begin{assumption}[$\Theta$ well-behaved]\label{asuravi}
    There exists an $\alpha \in (0,1]$ such that 
    \begin{enumerate}
        \item $\max_{e\in S^\complement} \Vert \Gamma_{eS} \bigl( \Gamma_{SS} \bigr)^{-1} \Vert_1 \leq 1-\alpha \quad $,
        \item $\ell_\infty\bigl( \Theta, \tilde{\Theta}_r\bigr) \leq \bigl( 3 \kappa_\Sigma \operatorname{deg}(\Theta) \bigr)^{-1} \text{ for all } r\in \mathcal{R},$
        where $\ell_\infty(\Theta^{\prime}, \Theta^{\prime\prime}):= \max_{i,j\in \mathcal{V}}\vert \Theta_{ij}^\prime - \Theta_{ij}^{\prime \prime}\vert $ and $\operatorname{deg}(\Theta)$ denotes the maximum degree of the corresponding conditional dependency graph induced by $\Theta$. 
    \end{enumerate}
\end{assumption}
The first assumption is widely used and is often referred to as \textit{incoherence} or \textit{irrepresentability condition}. To shed some light onto this condition \citet{Ravikumar.20081121} consider for $i,j\in \mathcal{V}$ the centered random variable $y_{ij} := z_iz_j - \mathbb{E}[z_iz_j]$, which can be interpreted as the interaction between components $z_i \text{ and }z_j$. One can show that the incoherence condition requires $y_{ij}$ for $(i,j)\not \in \mathcal{E}$ to be not highly correlated with $y_{lm}$ with $(l,m)\in \mathcal{E}$. Hence, the incoherence assumption bounds the correlation between irrelevant interactions, $y_{ij}$ with $(i,j)\not \in \mathcal{E}$, and relevant interactions, $y_{ij}$ with $(i,j)\in \mathcal{E}$. This interpretation goes in line with the incoherence condition for the lasso, see \citet{Ravikumar.20081121} and references therein for further details.\\
The second assumption states that the constrained \glasso{} $\tilde{\Theta}_r$, i.e. the primal solution, is accurate enough in the sense that it satisfies the $\ell_\infty$-bound. Note, that the bound scales with the maximal degree of the graph. Hence, the higher the degree of the true graph, the more accurate the primal solution must be to satisfy Assumption~\ref{asuravi}. This indicates a reason why scale-free graphs, which have in general a larger maximal degree than random graphs, tend to be harder to be well-estimated by the \glasso{}.

In the following we consider $n$ samples $\bs{z}^{(1)}, \dots ,\bs{z}^{(n)}$ drawn independently from $\bs{z} \sim \mathcal{N}_d(\bs{0}_d, \Sigma)$ for a positive definite covariance matrix $\Sigma$.
Along the line of thoughts of \citet{Ravikumar.20081121}, we can make use of Assumption~\ref{asuravi} and obtain the following theoretical properties. We start by stating conditions for a $\ell_\infty$-bound of the primal solution $\tilde{\Theta}_r$ (Lemma~\ref{lemma:infbound}). Next, we state conditions, under which it holds that $\hat{\Theta}_r=\tilde{\Theta}_r$ (Lemma~\ref{lemma:primalsuccess}), so that Lemma~\ref{lemma:infbound} gives us a $\ell_\infty$-bound for the \glasso{} $\hat{\Theta}_r$.
Finally, we use another Lemma (Lemma~\ref{lemma:rbound}) to show that the conditions in Lemma~\ref{lemma:primalsuccess} are satisfied. 
The proofs of Lemma~\ref{lemma:infbound},~\ref{lemma:primalsuccess}, and~\ref{lemma:rbound} can be found in \citet{Ravikumar.20081121}. 
\begin{lem}[Lemma $6$ in \citet{Ravikumar.20081121}] \label{lemma:infbound}
    Assume that 
    \begin{equation}
        2 \kappa_{\Sigma} \bigl( \ell_\infty \bigl( \Sigma , \, \hat{\Sigma}_{\operatorname{emp}}\bigr) +r \bigr) \leq C^\prime \enspace, \label{assumption:lemmabound}
    \end{equation}
    where $C^\prime:= \min \bigl\{ (3\kappa_{\Gamma}\operatorname{deg}(\Theta))^{-1},\, (3\kappa^3_{\Sigma} \kappa_{\Gamma}\operatorname{deg}(\Theta))^{-1} \bigr\}$, and $\hat{\Sigma}_{\operatorname{emp}}:=1/n \sum_{i=1}^n ( \bs{z}^{(i)})^\top \bs{z}^{(i)}$ is the empirical covariance matrix.
    Then it holds that 
    \begin{equation}
        \ell_\infty \bigl(\Theta, \,  \tilde{\Theta}_r , \bigr) \leq 2 \kappa_{\Sigma} \bigl( \ell_\infty \bigl( \Sigma , \, \hat{\Sigma}_{\operatorname{emp}} \bigr) +r \bigr) \enspace. \label{primalbound}
    \end{equation}
\end{lem}
\begin{lem}[Lemma $3$ $+$ Lemma $4$ in \citet{Ravikumar.20081121}] \label{lemma:primalsuccess}
    Suppose that 
    \begin{equation*}
        \operatorname{max}\bigl\{ \ell_\infty \bigl( \Sigma, \, \hat{\Sigma}_{\operatorname{emp}}  \bigr), \; \Vert R \Vert_\infty  \bigr\} \leq \frac{\alpha r}{8} \enspace,
    \end{equation*}
    where \begin{equation*} R:= \tilde{\Theta}_r^{-1} - \Theta^{-1} + \Theta^{-1} \bigl( \tilde{\Theta}_r - \Theta \bigr) \Theta^{-1} \enspace
        \end{equation*}
        is the remainder of the first-order Taylor approximation of $\tilde{\Theta}_r^{-1}$ at $\Theta$.
    Then it holds that $\tilde{\Theta}_r=\hat{\Theta}_r$.
\end{lem}
\begin{lem}[Lemma 5 in \citet{Ravikumar.20081121}] \label{lemma:rbound}
    Suppose that part 2 in Assumption~\ref{asuravi} is satisfied. Then, it holds that 
        \begin{equation}
            \Vert R \Vert_\infty \leq \frac{\kappa_\Sigma}{6\operatorname{deg}(\Theta)} \enspace.
        \end{equation}
\end{lem}
\paragraph{Bound for the \glasso{}}
Having introduced this bunch of lemmas, we are ready to put the pieces together and obtain the following theorem that summarizes the aforementioned lemmas and assumptions, and that proves the validity of \eqref{eq:assumptionglasso}.
\begin{theorem}[$\ell$-Bound for the Graphical Lasso] \label{theoravi}
    Suppose that Assumption~\ref{asuravi} holds. Then, conditioned on 
    \begin{equation}
        \mathcal{T}_r := \Biggl\{ \max \Bigl\{  \ell_\infty \bigl( \Sigma, \, \hat{\Sigma}_{\operatorname{emp}}  \bigr), \frac{\kappa_\Sigma}{6\operatorname{deg}(\Theta)} \Bigr\} \leq \frac{ \alpha r}{8} \Biggr\}\label{condition}\ ,
    \end{equation}
    it holds for $r \leq \frac{4}{\kappa_{\Sigma}(\alpha + 8)} C^\prime$ that 
    \begin{equation*}
        \ell(\Theta, \, \hat{\Theta}_r) \leq \kappa_{\Gamma}\frac{\alpha + 8}{4} r \enspace.
    \end{equation*}
\end{theorem}
\begin{proof}
    First, note that conditioned on $\mathcal{T}_r$, we almost satisfy the conditions in Lemma~\ref{lemma:primalsuccess}. It remains to show that $\Vert R \Vert_\infty \leq \alpha r / 8$. However, this is true since
    \begin{equation*}
        \Vert R \Vert_\infty  \leq \frac{\kappa_\Sigma}{6\operatorname{deg}(\Theta)} \leq \frac{\alpha r}{8},
    \end{equation*}
    where the first inequality follows from Lemma~\ref{lemma:rbound}, and the second inequality follows from the condition on $\mathcal{T}_r$. Therefore, the primal-dual construction succeeds (Lemma~\ref{lemma:primalsuccess}), that is, $\tilde{\Theta}_r=\hat{\Theta}_r$. Thus, we obtain a $\ell_\infty$-bound for the \glasso{} \eqref{primalbound}, if \eqref{assumption:lemmabound} in Lemma~\ref{lemma:infbound} is satisfied. Plugging in our assumptions, we see that this is indeed the case, since
    \begin{align*}
        2 \kappa_{\Sigma} \bigl( \ell_\infty \bigl( \Sigma, \, \hat{\Sigma}_{\operatorname{emp}} \bigr) +r \bigr) &\leq 2\kappa_\Sigma \bigl( \frac{\alpha r}{8} + r \bigr) \\
        & = 2\kappa_\Sigma \frac{\alpha + 8}{8}r \\
        & \leq 2\kappa_\Sigma \frac{\alpha + 8}{8} \frac{4}{\kappa_\Sigma (\alpha + 8)} C^\prime \\ &=C^\prime \enspace. 
    \end{align*}
    Hence, we obtain 
    \begin{equation*}
        \ell( \Theta, \hat{\Theta}_r) \leq \ell_\infty \bigl( \Theta, \hat{\Theta}_r \bigr) = \ell_\infty \bigl( \Theta, \tilde{\Theta}_r \bigr) \leq \kappa_\Sigma \frac{\alpha + 8}{4} r \enspace.
    \end{equation*}
\end{proof}
Theorem~\ref{theoravi} is very similar to Theorem~1 by \citet{Ravikumar.20081121}, but differs in $2$ important aspects:
\begin{enumerate}
    \item \citet{Ravikumar.20081121} fix a regularization parameter (dependent on $\alpha$, the amount of samples $n$, and some quantity $q$). On the other hand, Theorem~\ref{theoravi} is a bound for a range of regularization parameters;
    \item \citet{Ravikumar.20081121} derive a probabilistic bound (where the probability that this bound holds is determined by $q$). Theorem~\ref{theoravi} is, conditioned on a class of events $\mathcal{T}_r$, a deterministic bound.
\end{enumerate}
The purpose of these changes is to fit the result to our setting that is generalized by Assumption~\ref{assumptiongeneral}. 

Note that $\mathcal{T}_r$ in Theorem~\ref{theoravi} bounds two quantities, whereas the first one describes an effective noise in the empirical covariance matrix. Therewith, it shares the same intuition as the set considered for the linear regression problem in \citet{Chichignoud.20161108}: the larger the noise, the larger we need to choose the regularization parameter to be able to control these fluctuations. Consequently, the set of ``controllable'' scenarios $\mathcal{T}_r$ grows with $r$. On the other hand, large noise shrinks the set of these scenarios. Even though we obtain a class of events $\bigl( \mathcal{T}_r\bigr)_{r\in \mathcal{R}}$ sharing a similar interpretation to the class considered in \citet{Chichignoud.20161108}, we have to resort to a more involved primal-dual-witness construction to get these results. 
Further, we need to keep in mind that Theorem~\ref{theoravi} supposes $r\leq 4/(\kappa_\Sigma(\alpha + 8)) C^\prime$, which means that the set of possible regularization parameters $\mathcal{R}$ must satisfy $\mathcal{R} \subset (0,  4/(\kappa_\Sigma(\alpha + 8)) C^\prime]$.

Lastly, although the primal-dual witness construction underpinning Theorem~\ref{theoravi} already implies that there are no false positives, we show robustness against false positives independently for our calibration scheme in Corollary~$4$ in the main paper. This way, we make the theory we are presenting less dependent on the specific assumptions of Theorem~\ref{theoravi}.
\subsection{Remaining Proofs of the Results in the Main Paper}
\label{sec:otherproofs} 
\paragraph{Proof of Corollary $4$}
Based on Theorem~\ref{avtheo_supp}, we derive Corollary~$4$ from the main text, which provides strong graph recovery results for the \thav{}. In the following we write $\hat{\Theta}:= \hat{\Theta}_r$ to enhance readability.
\begin{proof}
    The first statement is a direct consequence of Theorem~\ref{avtheo_supp} applied to this setting: 
    Consider any zero-entry $\Theta_{ij}=0$ of the precision matrix. Therefore, \eqref{avbound_supp} yields with probability $1-\delta$ that
    \begin{equation*}
        \vert \hat{\Theta}_{ij} \vert=  \vert \Theta_{ij} - \hat{\Theta}_{ij}\vert  \leq 3Cr_{\delta}^\ast \quad \text{and} \quad \hat{r} \leq r_\delta^\ast \enspace .
    \end{equation*}
    Hence, it holds with probability $1-\delta$ that $\vert \hat{\Theta}_{ij} \vert \in [0, \, 3Cr_\delta^\ast]$ and that $\lambda C\hat{r} \leq 3C\hat{r} \leq 3Cr_\delta^\ast$, which we need to prove the well-definedness of the interval $(\lambda C\hat{r}, \, 3 C r_\delta^\ast$]. This proves part $1$. \\ 
    To prove 2., let us consider a significant edge $(i,j)$ such that $\vert \Theta_{ij}\vert > (3+\lambda)Cr_\delta^\ast$. We prove the result via contradiction. Hence, let us assume that $\hat{\Theta}_{ij}^t=0$. It follows that $\vert \hat{\Theta}_{ij}\vert < \lambda C\hat{r}\leq \lambda Cr_\delta^\ast$ with probability $1-\delta$, where the first inequality follows via the definition and the second inequality follows via Theorem~\ref{avtheo_supp}.
    But then, using the reversed triangle inequality, it holds that
    \begin{equation*}
        \vert \Theta_{ij} - \hat{\Theta}_{ij} \vert \geq \vert \Theta_{ij} \vert - \vert \hat{\Theta}_{ij}\vert > (3+\lambda)Cr_\delta^\ast - \lambda Cr_\delta^\ast = 3Cr_\delta^\ast \enspace,
    \end{equation*}
    which contradicts \eqref{avbound_supp}. Hence, $\hat{\Theta}^t_{ij}$ cannot be zero-valued and therefore it is $(i,j)\in \hat{\mathcal{E}}$.
\end{proof} 
\paragraph{Behaviour of $\hat{r}$ in $C$}
The proposed \av{} calibration scheme employs a hyperparameter $C$ that is related to our central Assumption~\ref{assumptiongeneral}. For completeness, we prove that the \av{} regularization parameter $\hat{r}$ is decreasing in $C$.
\begin{prop}\label{behaviourC}
    Consider the \av{} regularization parameter $\hat{r}_C$ as a quantity of $C$. Then, $\hat{r}_C$ is monotonely decreasing in $C$.
\end{prop}
\begin{proof}
    Consider any $C > 0$. Then, according to the definition of $\hat{r}_C$, it holds for any positive $r^\prime, r^{\prime \prime} \in \mathcal{R} \text{ with } \hat{r}_C \leq r^\prime, r^{\prime \prime}$ that 
    \begin{equation*}
        \frac{\ell\bigl(\hat{\Theta}_{r^\prime},\, \hat{\Theta}_{r^{\prime \prime}} \bigr)}{r^\prime + r^{\prime \prime}} \leq C \enspace.
    \end{equation*}
    Therefore, it holds for any $\varepsilon > 0$ that 
    \begin{equation*}
    \ell\bigl( \hat{\Theta}_{r^\prime},\, \hat{\Theta}_{r^{\prime \prime}} \bigr) \leq C (r^\prime + r^{\prime\prime}) 
     \leq (C+\varepsilon) (r^\prime + r^{\prime \prime}) \enspace .
    \end{equation*}
    This holds for any $r^\prime, r^{\prime \prime} \geq \hat{r}_C$, hence, 
    \begin{equation*}
    \hat{r}_C \in \Bigl\{ r \in \mathcal{R} : \ell\bigl(\hat{\Theta}_{r^\prime},\hat{\Theta}_{r^{\prime\prime}}\bigr) \leq (C+\varepsilon)(r^\prime + r^{\prime\prime}) \enspace \forall r^\prime,r^{\prime\prime} \in \mathcal{R} \cap [r, \infty)  \Bigr\} \enspace.
    \end{equation*}
    Thus, $\hat{r}_C$ must be greater than or equal to the minimizer of the above set, which is per definition minimized by $\hat{r}_{C+\varepsilon}$. We conclude that $\hat{r}_C \geq \hat{r}_{C+\varepsilon}$ for any $\varepsilon >0$. Since we chose $C>0$ arbitrary, this proves the claim.
\end{proof}

\section{Computational Efficiency of the \thav{}}
\label{compeff}
In addition to its theoretical finite sample guarantees, the \thav{} estimator also comes with notable computational benefits. In the following, we consider the \thav{} technique applied to the \glasso{} optimization problem. In comparison to calibration schemes that are based on data splitting and resampling methods like \stars{} or other CV-like approaches, \thav{} only requires at most one solution path $( \hat{\Theta}_r )_{r\in\mathcal{R}}$. We can readily implement the \thav{} using Algorithm~\ref{alg:glasso}. We observe that using Algorithm~\ref{alg:glasso}, we
\begin{enumerate}
    \item calculate at most $1$ \glasso{} path since we can store and reuse every \glasso{} solution;
    \item start to calculate \glasso{} solutions for large regularization parameters and decrease the regularization until we break. Therefore, we most likely avoid to calculate and store \glasso{} solutions for very small regularization parameters, which are computationally most demanding;
    \item can use warm starts, which are implemented in the \emph{glasso} R package \cite{Friedman.2008, Witten.2011}, to calculate $\hat{\Theta}_r$. This decreases the computation time even more.
\end{enumerate}
An implementation of the \thav{} is provided along with the submission\footnote{The public git repository can be accessed through \url{https://github.com/MikeLasz/thav.glasso}}. 

Empirically, we can observe that the \thav{} estimator scales very well in runtime, computational stability, and memory usage even in very large settings. This is an important advantage of \thav{} in practical applications.

\begin{algorithm}[H]
\DontPrintSemicolon
 \KwData{ data $Z$, set of increasing regularization parameters $\mathcal{R}$, constant $C$, threshold parameter $\lambda$}
 \KwResult{\thav{} estimator }
 $j \gets \text{length}(\mathcal{R}) - 1$, \hspace{0.5cm} $r \gets \mathcal{R}[j]$, \hspace{0.5cm}$\hat{r} \gets \mathcal{R}[1]$ \;
 \While{$r> \mathcal{R}[1]$}{
  $\hat{\Theta}_r \gets $ \glasso{}  using regularization parameter $r$ \;
  $j^\prime\gets \text{length}(\mathcal{R})$, \hspace{0.5cm} $r^\prime \gets \mathcal{R}[j^\prime]$ \;
  \While{$r^\prime > r$}{
    \eIf{$\ell( \hat{\Theta}_r,\; \hat{\Theta}_{r^\prime}) > C(r + r^\prime)$}
    {
        $\hat{r} \gets \mathcal{R}[j+1]$ \tcp*[f]{proposed regularization parameter} \;
        BREAK \;
    }
    {
        $j^\prime \gets j^\prime - 1$\;
        $r^\prime \gets \mathcal{R}[j^\prime]$ \tcp*[f]{decrease until break or $r^\prime=r$}\; 
    }
    }
    $j \gets j - 1$ \;
    $r \gets \mathcal{R}[j]$ \tcp*[f]{decrease until break or $r=\mathcal{R}[1]$} \;
   }
   AV $\gets \hat{\Theta}_{\hat{r}}$ \;
   $t \gets \lambda C\hat{r}$  \;
   \KwRet{AV thresholded by $t=\lambda C \hat{r}$}
 \caption{Thresholded adaptive validation \glasso{}} \label{alg:glasso}
\end{algorithm}

\section{Experimental Analysis: Details and Additional Experiments}
\label{sec:othersimu}
In this section, we present details and additional results of our empirical analysis of the proposed calibration scheme.
We describe the methods that are used to generate the precision matrices in Section~\ref{sec:graphgeneration} and give details about the computation of the estimators that are used for comparison in Section~\ref{sec:compdetails}. We extend the empirical analysis from the main paper in Section~\ref{supplement:additionalsimus}. In Section~\ref{sec:otheravmethods} we apply the \av{} calibration scheme to tune the rSME and a modified version of the graphical lasso that employs power law regularization.

\subsection{Precision Generation Methods} \label{sec:graphgeneration}
This section serves as a detailed description of the generation process of the precision matrices that are used in the experimental analysis.
The precision matrices are generated according to the following steps\footnote{The generation procedure is adopted from \citet{caballe2015selection}}:
\begin{itemize}
    \item \textit{random graph:} First, we generate a \textit{Gilbert graph} \cite{Gilbert.1959}, that is, we independently connect two nodes $i\neq j$  with some fixed probability $p$, where we choose $p=3/d$ as the default value. Then, we set the non-zero pattern of $\Theta$ according to the adjacency matrix of the Gilbert graph. The next step is to assign all non-zero entries of the precision matrix a value. For each edge $(i,j)$, we sample a value for $\Theta_{ij}=\Theta_{ji}$ from a uniform distribution over the set $I:= [-0.9, \,-0.5] \cup [0.5,\, 0.9]$. To guarantee positive definiteness, we set the diagonal entries to $\mu:= - \lfloor \lambda_{\operatorname{min}}( \Theta) \rfloor_1$, which is the smallest eigenvalue of $\Theta$ rounded downwards to $1$ decimal place. Finally, we scale $\Theta$ such that it has unit diagonal entries. 
    
    \item \textit{scale-free graph:} In this case, the adjacency matrix is generated according to a \textit{Barabasi-Albert Algorithm} \cite{Barabasi.2016} that generates the corresponding graph gradually. The algorithm starts with $2$ connected nodes and adds sequentially a node that will be connected with one of the existent nodes in the network. The probability to be connected to a node $i$ is proportional to its degree, that is,
    \begin{equation*}
        \mathbb{P}(\{\text{new node connected with node }i\})= \frac{\operatorname{deg}(i)}{\sum_{j\in\mathcal{V}} \operatorname{deg}(j)},
    \end{equation*}
    where $\operatorname{deg}(i),\, \operatorname{deg}(j)$ denote the degree of the nodes $i$ and $j$ respectively. We repeat this step until we end up with a graph with $d$ nodes. We set the non-zero pattern of $\Theta$ according to the generated graph and proceed as in the previous case.
\end{itemize}
Both generation methods are summarized in Algorithm~\ref{alg:graphgeneration}.

\begin{algorithm}[H]
\DontPrintSemicolon
 \KwData{ dimension $d$, type of graph $\operatorname{type}$}
 \KwResult{positive definite precision matrix}
 \eIf{$\operatorname{type} == \text{random}$}{
    adj $\gets$ adjacency matrix of a Gilbert graph \tcp*[f]{using the \emph{huge} package}\; 
    }{
    adj $\gets$ adjacency matrix of a scale-free graph \tcp*[f]{using the \emph{huge} package}\;}
 \For{$i<j \And  i,j \in \{1,\dots ,d\}$ }{
    $\Theta[i, j] \gets \Theta[j, i] \gets$ Sample from a uniform distribution over $[-0.9,\, -0.5] \cup [0.5, \, 0.9]$ \; 
    }
 $\operatorname{min.ev} \gets \lfloor \lambda_{\operatorname{min}}(\Theta)\rfloor$ \tcp*[f]{smallest Eigenvalue of $\Theta$ downwards rounded to $1$ decimal place} \; 
 $\Theta \gets \Theta + \operatorname{min.ev} \ast I_{d\times d}$ \; 
 $\Theta \gets 1/\operatorname{min.ev}  \ast \Theta$\; 
 \KwRet{Precision matrix $\Theta$ of type $\operatorname{type}$}
 \caption{Precision matrix generation} \label{alg:graphgeneration}
\end{algorithm}

We use pre-implemented functions provided by the R package \emph{huge} \cite{Zhao.2012} to generate the adjacency matrices.
The generation of both types of precision matrices is implemented in our attached code. Figure~\ref{examplegraph} shows an example of both graph structures.
\begin{figure*}[t]
\vskip 0.2in
\centerline{
     \subfigure[]{\includegraphics[width=0.5\columnwidth]{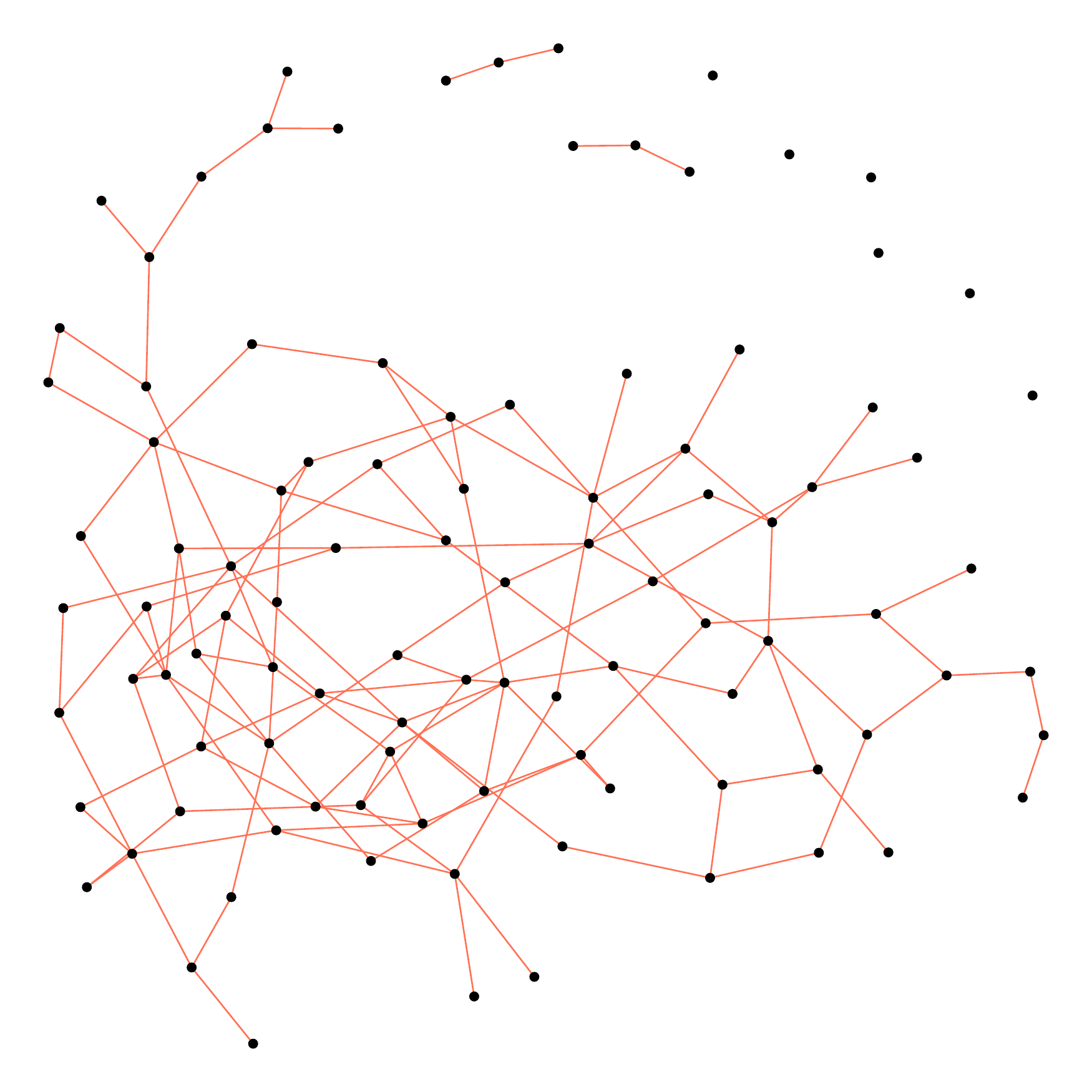}}
     \subfigure[]{\includegraphics[width=0.5\columnwidth]{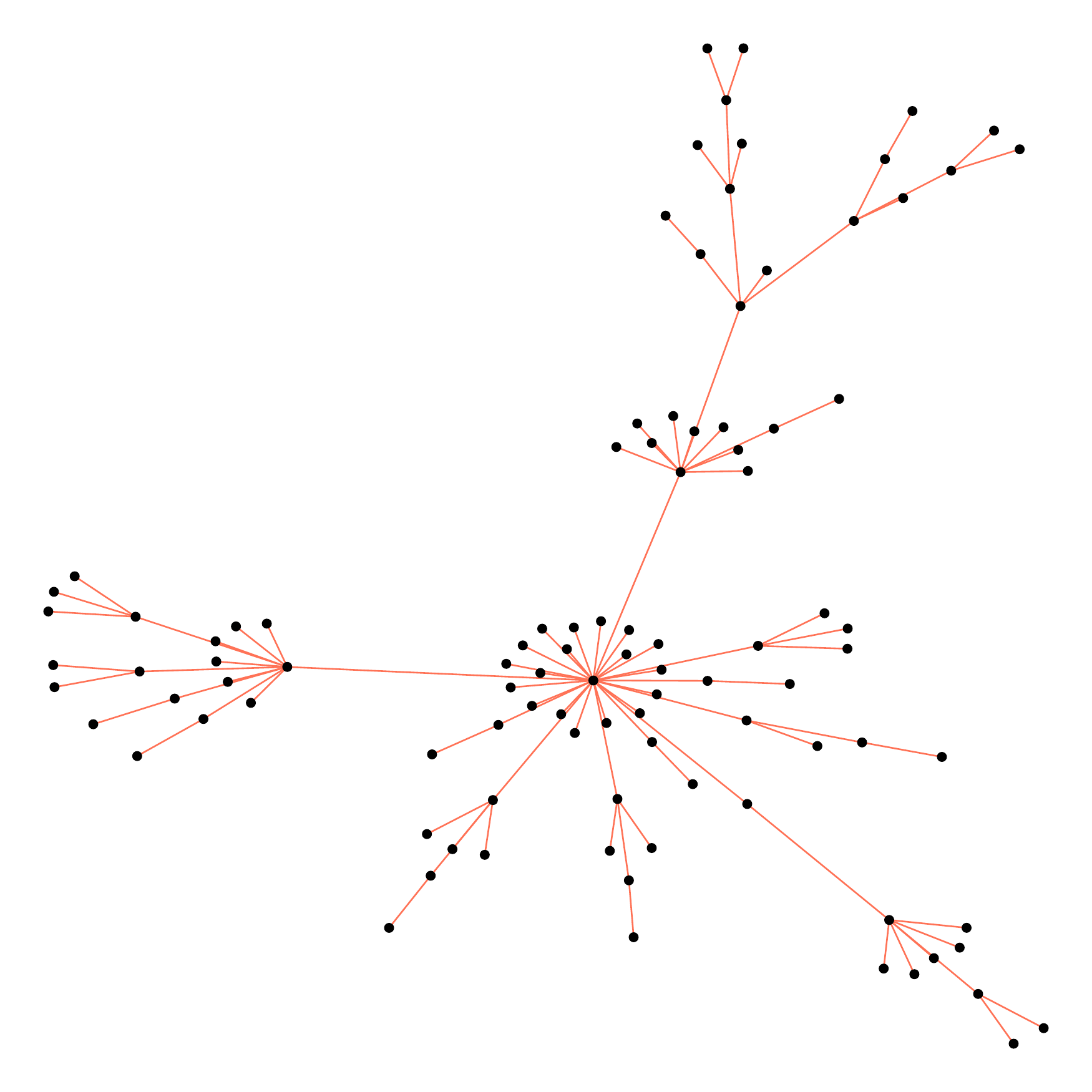}}}
     
\caption{An example of the $2$ graph types that are used in our simulations. Both graphs have $d=100$ nodes. Graph~(a) represents a random graph with $p=3/100$, and graph~(b) is a scale-free graph according to the described \textit{Barabasi-Albert Algorithm}. \label{examplegraph}} \vskip -0.2in
\end{figure*}

Precision matrices generated by these methods include values $\Theta_{ij}$ that can be either positive or negative, and vary in absolute value. The absolute values of the generated scale-free graphs tend to be smaller than those of the generated random graphs, which is demonstrated in Figure~\ref{fig:weights}. 

\subsection{Computational Details} \label{sec:compdetails}
For computing the \glasso{} calibrated via \stars{} and for computing the TIGER, we employ the R package \emph{huge} \cite{Zhao.2012}. As suggested by \citet{Liu.20120911}, we set the regularization parameter for the TIGER to $\xi=\sqrt{2}/\pi$. The scaled lasso estimator is calculated using the R package \emph{scalreg} \cite{scalreg}. To obtain the rSME and the SCIO, we employ the \emph{genscore} package \cite{genscore} and the \emph{scio} package \cite{scio}, respectively. An implementation of the sf-glasso is provided with our submission, which is computed using a reweighted $\ell_1$-method proposed by \citet{Liu.2011}. We use $C=0.5$ and $\lambda=2$ to calibrate the rSME, and $C=1.5$ and $\lambda=1.5$ to calibrate the sf-glasso using \thav{}.
If not stated differently, we kept all the default settings.
\subsection{Additional Experiments}
\label{supplement:additionalsimus}
We evaluate the graph recovery performance based on the $F_1$-score, the precision, and the recall of the estimator. These quantities are defined as follows:
\begin{align*}
    \textrm{precision}\bigl(\mathcal{E},\,  \hat{\mathcal{E}}\bigr) &:= \frac{ \vert \hat{\mathcal{E}} \cap \mathcal{E}\vert}{\vert \hat{ \mathcal{E}}\vert }\in [0,1],\\
    \textrm{recall}\bigl(\mathcal{E},\,  \hat{\mathcal{E}}\bigr) &:= \frac{ \vert \hat{\mathcal{E}} \cap \mathcal{E}\vert}{\vert \mathcal{E}\vert }\in[0,1],
\end{align*}
where $\mathcal{E},\; \hat{\mathcal{E}}$ are the edge sets obtained from the non-zero patterns of $\Theta$ and some estimate $\hat{\Theta}$, respectively.
While precision and recall behave very similar to the number of false positives and the number of false negatives, respectively, they also put the size of the edge sets into relation. The $F_1$-score puts both measures, precision and recall, into relation 
\begin{align*}
    F_1\bigl(\mathcal{E},\,  \hat{\mathcal{E}}\bigr)&:= \Biggl( \frac{ \operatorname{precision}\bigl(\mathcal{E},\,  \hat{\mathcal{E}}\bigr)^{-1} + \operatorname{recall}\bigl(\mathcal{E},\,  \hat{\mathcal{E}}\bigr)^{-1}}{2}\Biggr)^{-1} \in [0,1].
\end{align*}
In the best case, where we correctly assign all edges, that is, if $\hat{\mathcal{E}}=\mathcal{E}$, we get a $F_1$-score equal to $1$. The worst $F_1$-score is $0$, if either precision or recall is equal to $0$.

In the following, we complement the simulation study from the main text. Similar to the main paper, we average the experiments over $25$ iterations (if not stated differently) and standard deviations are presented in parenthesis.
\begin{figure*}
\vskip 0.15in
    \includegraphics[width=\columnwidth]{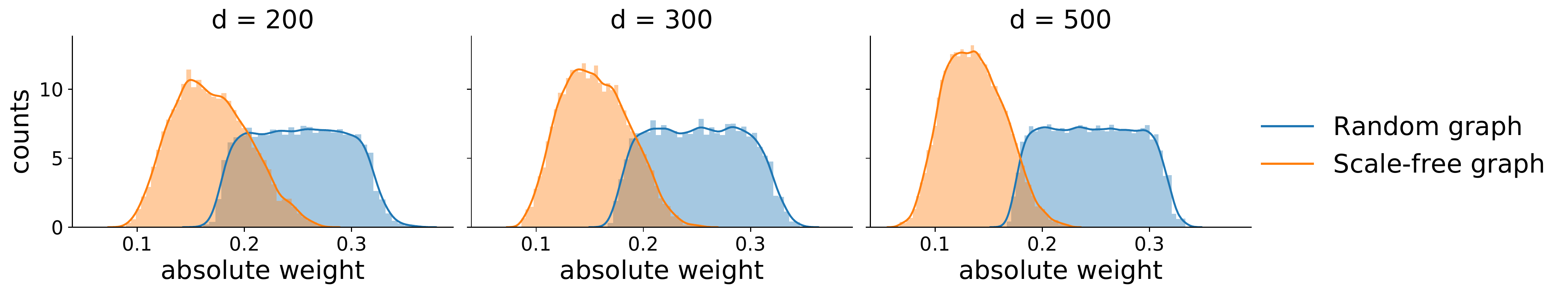}
    \caption{Histograms of the absolute valued non-zero weights $\vert \Theta_{ij}\vert$ of $50$ scale-free and $50$ random graphs generated by the procedure described in Section~\ref{sec:graphgeneration}. We considered graphs with
    $d\in \{200, 300, 500\}$ nodes.}\label{fig:weights}
\end{figure*}
\begin{table*}
    \caption{Graph recovery performance for varying graphs and sample size. The bold numbers indicate the best score in each setting.}
    \label{f1simus_supp}
    \vskip 0.15in
    \begin{center}
    \begin{small}
    \begin{sc}
    \begin{tabular}{lcccccccc}
        \toprule
         && \multicolumn{3}{c}{random} & \phantom{abc} & \multicolumn{3}{c}{scale-free}\\
          \cmidrule{3-5} \cmidrule{7-9} \\
         && $F_1$ & Precision & Recall & & $F_1$ & Precision & Recall\\
        \midrule
        $n=300,\, d=100$ &&&&&&&&\\
        oracle && 0.67 (0.11) & 0.55 (0.14) & 0.88 (0.06) &&  0.41 (0.11) & 0.46 (0.28) & 0.67 (0.30) \\ 
        StARS && 0.57 (0.12) & 0.42 (0.10) & 0.91 (0.07) && 0.32 (0.11) & 0.21 (0.09) & 0.75 (0.10) \\ 
        scaled lasso && 0.70 (0.03) & 0.54 (0.03) & 0.99 (0.01) && 0.48 (0.04) & 0.32 (0.04) & $\bs{0.97}$ (0.03) \\ 
        TIGER && 0.54 (0.08) & 0.37 (0.08) & 0.98 (0.01) && 0.40 (0.07) & 0.25 (0.06) & $\bs{0.97}$ (0.03) \\ 
        rSME (eBIC) && 0.41 (0.17) & 0.27 (0.13) & $\bs{1.00}$ (0.00) &&  0.37 (0.18) & 0.24 (0.13) & 0.96 (0.04)  \\ 
        scio (CV) && 0.46 (0.45) & 0.46 (0.45) & 0.46 (0.46)  &&   0.25 (0.29) & 0.45 (0.48) & 0.19 (0.24) \\ 
        scio (Bregman) && 0.19 (0.08) & 0.13 (0.15) & 0.98 (0.11) && 0.31 (0.23) & 0.34 (0.32) & 0.77 (0.28) \\ 
        thAV && $\bs{0.90}$ (0.06) & $\bs{0.87}$ (0.08) & 0.93 (0.05) &&  $\bs{0.66}$ (0.10) & $\bs{0.53}$ (0.13) & 0.91 (0.05)  \\ 
        \midrule
        $n=200, \, d=200$ &&&&&&&&\\
        oracle && 0.69 (0.11) & 0.61 (0.16) & 0.83 (0.03) &&  0.35 (0.08) & 0.34 (0.19) & 0.54 (0.19) \\ 
        StARS && 0.56 (0.11) & 0.41 (0.11) & 0.93 (0.03) &&  0.30 (0.08) & 0.20 (0.07) & 0.62 (0.11) \\ 
        scaled lasso && 0.66 (0.03) & 0.50 (0.03) & 0.96 (0.02) &&  0.37 (0.05) & 0.25 (0.04) & 0.71 (0.09) \\ 
        TIGER && 0.48 (0.08) & 0.32 (0.07) & $\bs{0.97}$ (0.01) &&  0.32 (0.07) & 0.20 (0.06) & $\bs{0.76}$ (0.09) \\ 
        rSME (eBIC) && 0.60 (0.30) & 0.51 (0.27) & 0.90 (0.19) &&  $\bs{0.43}$ (0.17) & 0.48 (0.23) & 0.55 (0.19) \\ 
        scio (CV) && 0.12 (0.25) & 0.22 (0.41) & 0.09 (0.21) &&  0.04 (0.10) & 0.16 (0.37) & 0.03 (0.06) \\ 
        scio (Bregman) && 0.45 (0.28) & 0.43 (0.36) & 0.87 (0.18) && 0.37 (0.12) & $\bs{0.87}$ (0.08) & 0.26 (0.14) \\ 
        thAV && $\bs{0.84}$ (0.07) & $\bs{0.79}$ (0.12) & 0.91 (0.06)  &&  0.39 (0.15) & 0.31 (0.19) & 0.69 (0.10) \\ 
        \midrule
        $n=300, \, d=300$ &&&&&&&&\\
        oracle && 0.76 (0.10) & 0.69 (0.14) & 0.87 (0.03) && 0.30 (0.08) & 0.29 (0.24) & 0.60 (0.20) \\ 
        StARS && 0.61 (0.09) & 0.45 (0.10) & 0.96 (0.03) &&  0.23 (0.09) & 0.15 (0.07) & 0.56 (0.10) \\ 
        scaled lasso && 0.68 (0.02) & 0.52 (0.02) & 0.98 (0.01) &&  0.34 (0.05) & 0.22 (0.03) & 0.73 (0.09) \\ 
        TIGER && 0.44 (0.10) & 0.29 (0.08) & $\bs{0.99}$ (0.01) &&  0.27 (0.07) & 0.17 (0.06) & $\bs{0.78}$ (0.11) \\ 
        rSME (eBIC) && 0.72 (0.11) & 0.58 (0.13) & 0.97 (0.02) && 0.32 (0.21) & 0.32 (0.23) & 0.63 (0.23) \\ 
        scio (CV) && 0.25 (0.39) & 0.35 (0.47) & 0.23 (0.36) &&  0.09 (0.13) & 0.36 (0.49) & 0.05 (0.07) \\ 
        scio (Bregman) && 0.38 (0.33) & 0.31 (0.33) & 0.98 (0.05) && 0.30 (0.06) & $\bs{0.92}$ (0.09) & 0.18 (0.04) \\ 
        thAV && $\bs{0.91}$ (0.04) & $\bs{0.90}$ (0.08) & 0.94 (0.03) &&  $\bs{0.42}$ (0.14) & 0.36 (0.18) & 0.62 (0.13) \\ 
        \bottomrule
    \end{tabular}
    \end{sc}
    \end{small}
    \end{center}
    \vskip -0.1in
\end{table*}
\paragraph{Performance in $F_1$-score}
We extend the simulations from the main paper to further graphs and number of samples, and compare the $F_1$-score, precision, and recall of the oracle, \thav{}, \stars{}, scaled lasso, TIGER, rSME using eBIC, SCIO using CV, and SCIO using a Bregman criterion. The results of the experiments in Table~\ref{f1simus_supp} are in accordance with the results from the main paper: the \thav{} estimator is superior to the other methods in $F_1$-score (in every but one setting) and achieves often a high precision. Again, the results for recovering a scale-free graph are worse than the results for recovering a random graph.
Figures~\ref{fig:exarecovery1}~--~\ref{fig:exarecovery6} show examples\footnote{We depict the true graph, the oracle \glasso, the \thav{}, and the $3$ best estimators among the remaining methods (ranked by $F_1$-score).} of the graph recovery performance of various methods in different settings. We can clearly observe that \thav{} returns a much sparser and more interpretable graph than the other methods. In the case $d=300 \text{ and } n=200$, no method is able to recover the graph structure of a scale-free graph, see Figure~\ref{fig:exarecovery5}.

\begin{table}[t]
    \caption{Similarity $F_1(\hat{\Theta}_{C^\prime}^t, \hat{\Theta}_{C^{\prime \prime}}^t)$ between \thav{} solutions using different $C$ ($C^\prime$ and $C^{\prime \prime}$) for a random graph.}
    \label{table:stabc}
    \begin{center}
    \begin{small}
    \begin{sc}
    \subfigure[$n=200, \, d=300$. The performance scores $F_1(\Theta, \hat{\Theta}^t_C)$ are $0.75\, (0.06), \, 0.83 \, (0.05) , \,  0.80 \, (0.08), \, 0.75 \, (0.15)$ for \mbox{$C\in\{0.5, \, 0.6, \, 0.7, \, 0.8\}$}, respectively.]{
    \begin{tabular}{lcccc}
        \toprule
        $C$ & & $0.6$ & $0.7$ & $0.8$ \\
        \midrule
        0.5 && 0.84 (0.05) & 0.68 (0.12) & 0.51 (0.21) \\ 
0.6 && 1 & 0.85 (0.09) & 0.67 (0.18) \\ 
0.7 &&  - & 1 & 0.79 (0.16) \\ 
0.8 && - & - & 1 \\ 
        \bottomrule
    \end{tabular}
    }
    \hspace{0.2in}
    \subfigure[$n=400,\, d=200$. The performance scores $F_1(\Theta, \hat{\Theta}^t_C)$ are $0.81 \, (0.04),\, 0.91 \, (0.04),\, 0.94 \, (0.02), \, 0.94\, (0.03)$ for \mbox{$C\in\{0.5, \, 0.6, \, 0.7, \, 0.8\}$}, respectively.]{
    \begin{tabular}{lcccc}
        \toprule
        $C$ & & $0.6$ & $0.7$ & $0.8$ \\
        \midrule
        0.5 && 0.87 (0.06) & 0.81 (0.07) & 0.80 (0.09) \\ 
0.6 && 1 & 0.95 (0.03) & 0.91 (0.05) \\ 
0.7 && - & 1 & 0.94 (0.06) \\ 
0.8 && - & - & 1 \\ 
        \bottomrule
    \end{tabular}
    }
    \end{sc}
    \end{small}
    \end{center}
    \vskip -0.1in
\end{table}

\paragraph{Dependence on $C$}
In this section we show further results demonstrating that the performance of the \thav{} estimator does not depend decisively on the specific value of $C$. Supplementing the results from Table~3 in the main paper, we investigate the difference in $F_1$-score between \thav{} estimators using different choices for $C$ (see Table~\ref{table:stabc}). We obtain a decent similarity among all estimates and, importantly, the overall $F_1(\Theta, \hat{\Theta})$ is always very high so that \thav{} outperforms competing methods for any choice of $C$\footnote{(as can be seen by comparing the results with the experiments depicted in Table~1 in the main paper)}. Figures~\ref{fig:exadifferentC}~--~\ref{fig:exadifferentClast} show some examples of recovered graphs employing the \thav{} with different $C$. We observe that they look all very similar and we obtain a good $F_1$-score, independently of the selected $C$.
 
Figure~\ref{bestthreshold_supp} supplements the investigations in Figure~1 from the main text. Again, we can observe that the (unthresholded) \av{} estimator is heavily dependent on the chosen $C$. However after thresholding, all estimators, except of the scale-free graph with $d=300$ using $n=200$ samples, reach a very similar performance plateau. An evident question would be, if it is then even necessary to apply regularization via $r$ instead of employing solely thresholding on the unregularized solution. But note that the unregularized \glasso{}, which is the maximum likelihood estimator, only exists in the setting $d<n$. This is a large restriction in the high-dimensional case. Secondly, we cannot obtain similarly good results for the unregularized thresholded estimator for any threshold, as can be seen in Figure~\ref{fig:mlethresholding}. Hence, we claim that it is necessary to combine both regularizations, regularization via regularized optimization and also regularization via thresholding, to obtain good graph recovery results. 

\subsection{Calibrating other methods via thAV} \label{sec:otheravmethods}
\begin{table*}[t]
    \caption{Graph recovery performance for varying graphs and sample size. The bold numbers indicate the best score in each setting.}
    \label{f1_otherav}
    \vskip 0.15in
    \begin{center}
    \begin{small}
    \begin{sc}
    \begin{tabular}{lcccccccc}
        \toprule
         && \multicolumn{3}{c}{random} & \phantom{abc} & \multicolumn{3}{c}{scale-free}\\
          \cmidrule{3-5} \cmidrule{7-9} \\
         && $F_1$ & Precision & Recall & & $F_1$ & Precision & Recall\\
        \midrule
        $n=300,\, d=200$ &&&&&&&&\\
        rSME  && $\bs{0.90}$ (0.03) & $\bs{0.93}$ (0.04) & $\bs{0.87}$ (0.07) &&  $\bs{0.54}$ (0.13) & 0.49 (0.19) & $\bs{0.66}$ (0.14) \\ 
        sf-glasso  && 0.70 (0.24) & 0.84 (0.18) & 0.73 (0.31) && 0.52 (0.18) & $\bs{0.82}$ (0.17) & 0.45 (0.22) \\
        \midrule
        $n=200, \, d=300$ &&&&&&&&\\
        rSME && $\bs{0.81}$ (0.06) & $\bs{0.82}$ (0.14) & $\bs{0.83}$ (0.10) &&  0.27 (0.09) & 0.19 (0.07) & $\bs{0.51}$ (0.15) \\ 
        sf-glasso  && 0.61 (0.19) & 0.73 (0.16) & 0.63 (0.25) && $\bs{0.31}$ (0.12) & $\bs{0.77}$ (0.24) & 0.24 (0.17) \\
        \midrule
        $n=400, \, d=200$ &&&&&&&&\\
       rSME && $\bs{0.90}$ (0.08) & $\bs{0.96}$ (0.04) & $\bs{0.86}$ (0.12) && $\bs{0.65}$ (0.14) & 0.63 (0.17) & $\bs{0.70}$ (0.15) \\ 
        sf-glasso  && 0.69 (0.27) & 0.79 (0.27) & 0.79 (0.33) && 0.62 (0.17) & $\bs{0.92}$ (0.07) & 0.49 (0.18) \\ 
        \bottomrule
    \end{tabular}
    \end{sc}
    \end{small}
    \end{center}
    \vskip -0.1in
\end{table*}
In Section~\ref{sec:general} we have derived a very general foundation for the calibration of regularized optimization problems. In the following empirical study, we tune $2$ other optimization problems for graphical modeling employing the proposed methodology. First, we use the \thav{} technique to calibrate the regularized score matching estimator (rSME) \cite{rsme}, which can be employed to recover the conditional dependency structure of a pairwise interaction model. That is, we assume $\bs{z}$ to have a log-probability density according to
\begin{equation} \label{pairwise_interaction_model}
    \log p(\bs{z}) = \sum_{1 \leq i \leq j \leq n} \Theta_{ij} t_{ij}(z_i, z_j) - \Psi(\Theta) + b(\bs{z})
\end{equation}
for sufficient statistics $t_{ij}$, $\Psi$ is a function depending on $\Theta$, and $b$ is the base measure.
The model class defined by \eqref{pairwise_interaction_model} includes the Gaussian model, but is much broader.
Applying the Hammersley-Clifford theorem, we can derive that similarly to the Gaussian setting it is
\begin{equation*}
    z_i \perp z_j \vert \bs{z}_{\backslash \{i,j\}} \quad \Leftrightarrow \quad \Theta_{ij} = 0.
\end{equation*}
The rSME is motivated by the score matching loss \cite{hyvarinen2005estimation} and can be reformulated as a the solution of the following optimization problem: 
\begin{equation} \label{rsme_optproblem}
    \hat{\Theta}_r = \argmin_{\Theta \in \operatorname{Sym}} \frac{1}{2} \operatorname{vec}\bigl(\Theta\bigr)^\top \Gamma(\bs{z}) \operatorname{vec} \bigl( \Theta\bigr) + g(\bs{z})^\top \operatorname{vec}\bigl(\Theta\bigr) + c(\bs{z}) + r \Vert \Theta \Vert_1,
\end{equation}
where $\operatorname{Sym}$ is the set of symmetric matrices in $\mathbb{R}^{D \times D}$, $\operatorname{vec}(\Theta) \in \mathbb{R}^{D^2}$ defines the columnwise stacked version of $\Theta$, $\Gamma$ is a symmetric $D^2 \times D^2$ block-diagonal matrix with blocks of size $D \times D$, and $g, c$ are functions mapping $\bs{z}$ to $\mathbb{R}^{D^2}$. For details, we refer to Lemma 2 by \citet{rsme}. Clearly, \eqref{rsme_optproblem} is of the same form as \eqref{generalsetting} and Theorem~1 in \citet{rsme} justifies\footnote{Using $\mathcal{T}_r:= \bigl\{ \frac{3(2-\alpha)}{\alpha} \max \{ c_\Theta, \varepsilon_1, \varepsilon_2 \} < r \bigr\}$ and $k(r) := \frac{c_\Gamma}{2-\alpha} r$.} Assumption~\ref{assumptiongeneral}.
Hence, we can derive similar results for the rSME tuned via the AV technique and threshold similarly to the \thav{} \glasso{}. 

The second type of estimators that we tune via AV is the \glasso{} equipped with a power law regularization\footnote{We set $\beta=0$ in the definition of the regularization \cite{Liu.2011}.} \cite{Liu.2011}
\begin{equation}
    \label{powerlawregularization_glasso}
    \hat{\Theta}_r \argmin_{\Omega \in \mathcal{S}_d^+} \biggl\{
        \tr \Bigl[ \frac{1}{n} \sum_{i=1} \bigl( \bs{z}^{(i)}\bigr)^\top \bs{z}^{(i)}\Omega \Bigl] - \log[ \det[\Omega]] + r \sum_{j=1}^D \log( \Vert \Omega_{i \cdot} \Vert_1)
    \biggr\},
\end{equation}
where $\Omega_{i \cdot}$ is the $i$th row of $\Omega$. This estimator is designed to learn scale-free graphs in a Gaussian graphical model. In comparison to the rSME, there is no theory about \eqref{powerlawregularization_glasso} that justifies Assumption~\ref{assumptiongeneral}. Nonetheless, we applied our calibration scheme for the \glasso{} with power law regularization (sf-glasso).

The results of the experiments for both estimators are shown in Table~\ref{f1_otherav} (see Section~\ref{sec:compdetails} for computational details). Due to the thresholding, both methods are able to find a good balance between precision and recall.
Surprisingly, even though the rSME can be applied to a much broader setting, i.e. in a pairwise interaction model~\eqref{pairwise_interaction_model}, it performs comparably to the \thav{} \glasso{} if tuned via the \thav{} technique. Hence, \thav{} rSME might be a promising candidate for an application of the general \thav{} scheme to the non-Gaussian setting.
The \thav{} sf-glasso performs weaker and is less stable. But remarkably, sf-glasso achieves a good precision in recovering a scale-free graph, while maintaining a decent recall.
Figure~\ref{fig:otherav_rand} and Figure~\ref{fig:otherav_sf} show some recovered graphs of \thav{} rSME and \thav{} sf-glasso.

\section{Supplement of the Applications}
We present the \thav{} estimator with threshold $t=C\hat{r}$, applied to the amgut2.filt.phy data from Section~4 in the main paper, in Figure~\ref{fig:gutthav2}. We decided to lower the threshold ($t=0.5C\hat{r}$) in the main paper, because the graph does not include information about interactions across different classes of microbes. But, as already mentioned in the main paper, Corollary~$4$ from the main paper is valid for any $\lambda \in (0,3]$. 

\begin{figure*}
 \centerline{
     \subfigure[]{  \includegraphics[width=0.45\columnwidth, height=0.37\columnwidth ]{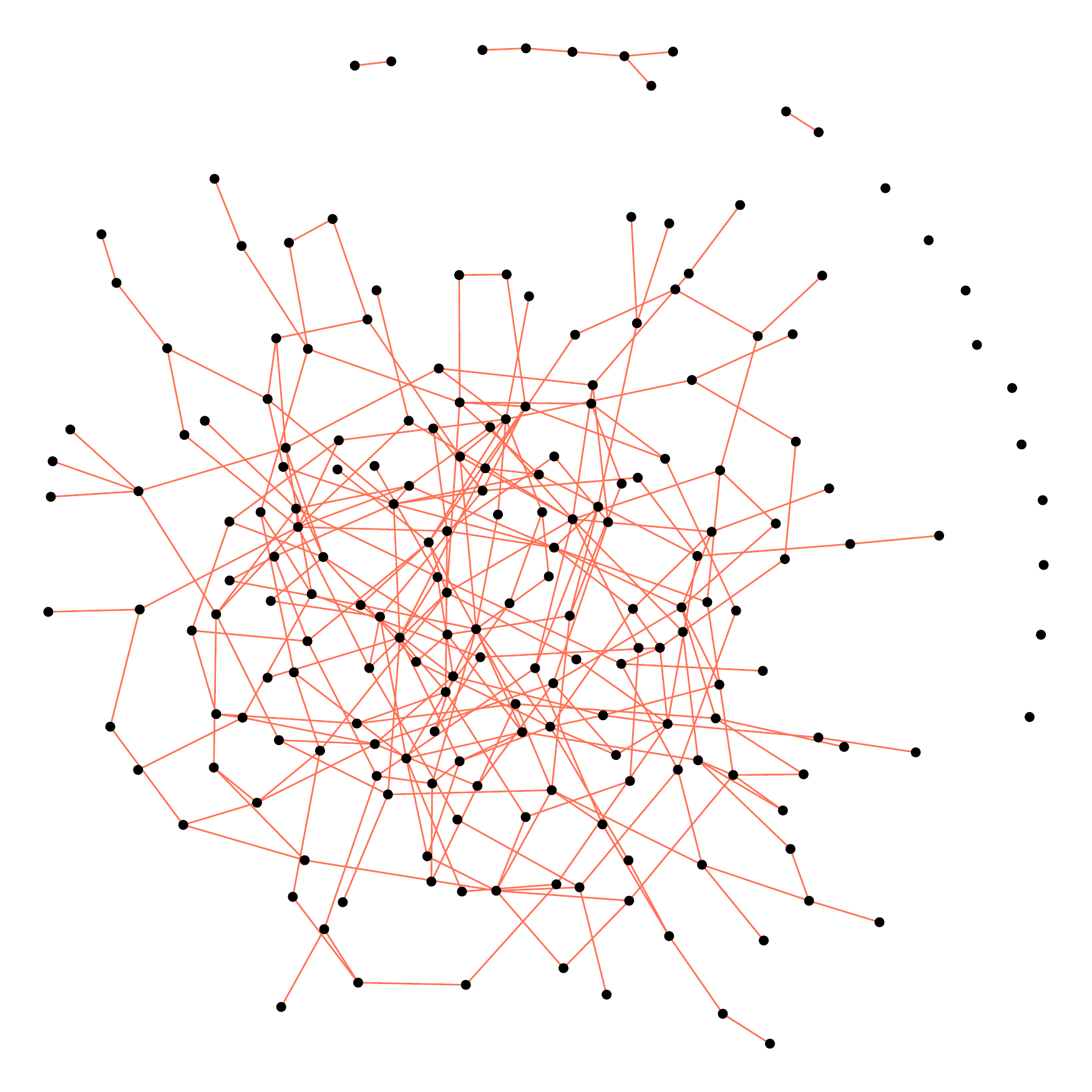}}
     \subfigure[]{  \includegraphics[width=0.45\columnwidth, height=0.37\columnwidth]{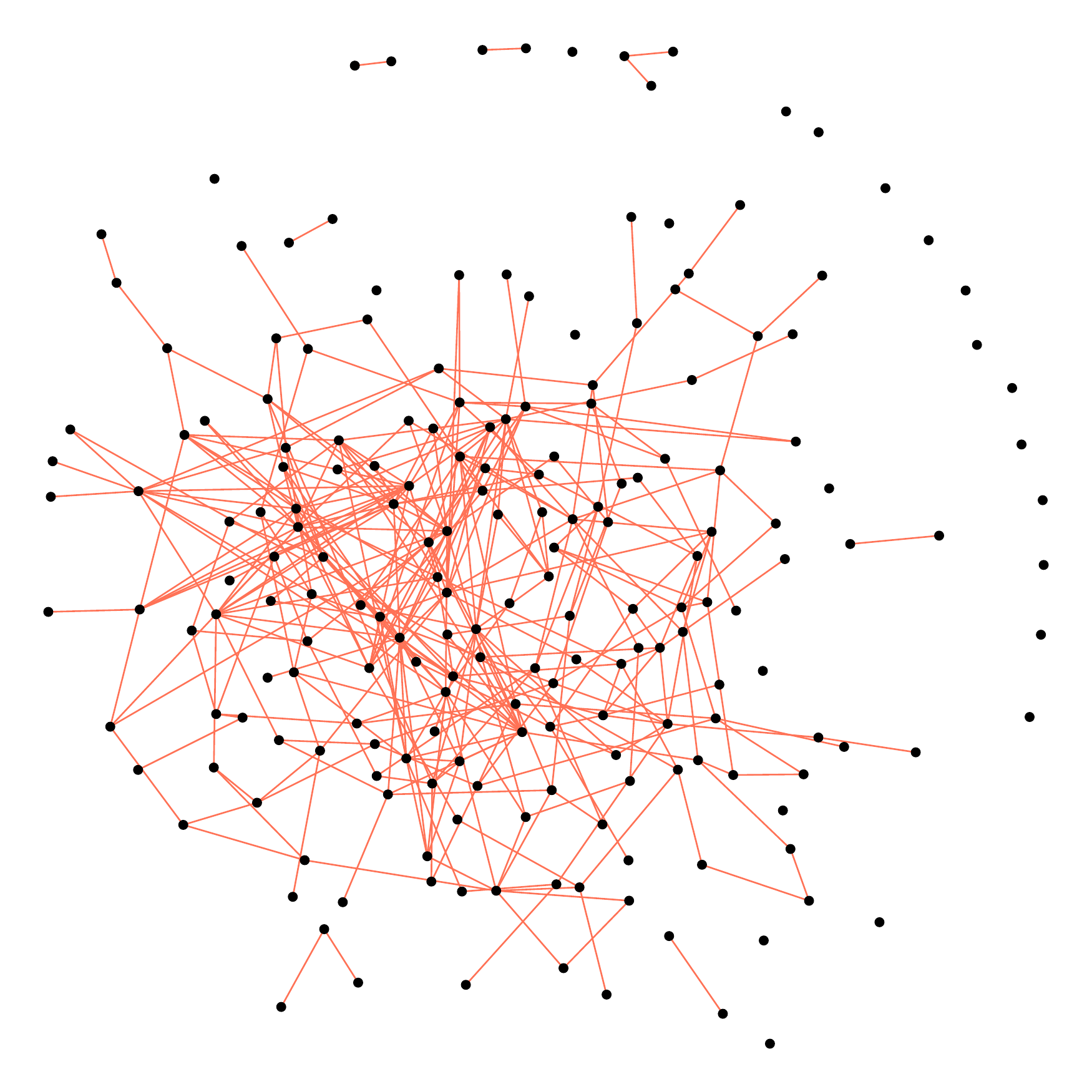} }}
 \centerline{    
     \subfigure[]{  \includegraphics[width=0.45\columnwidth, height=0.37\columnwidth]{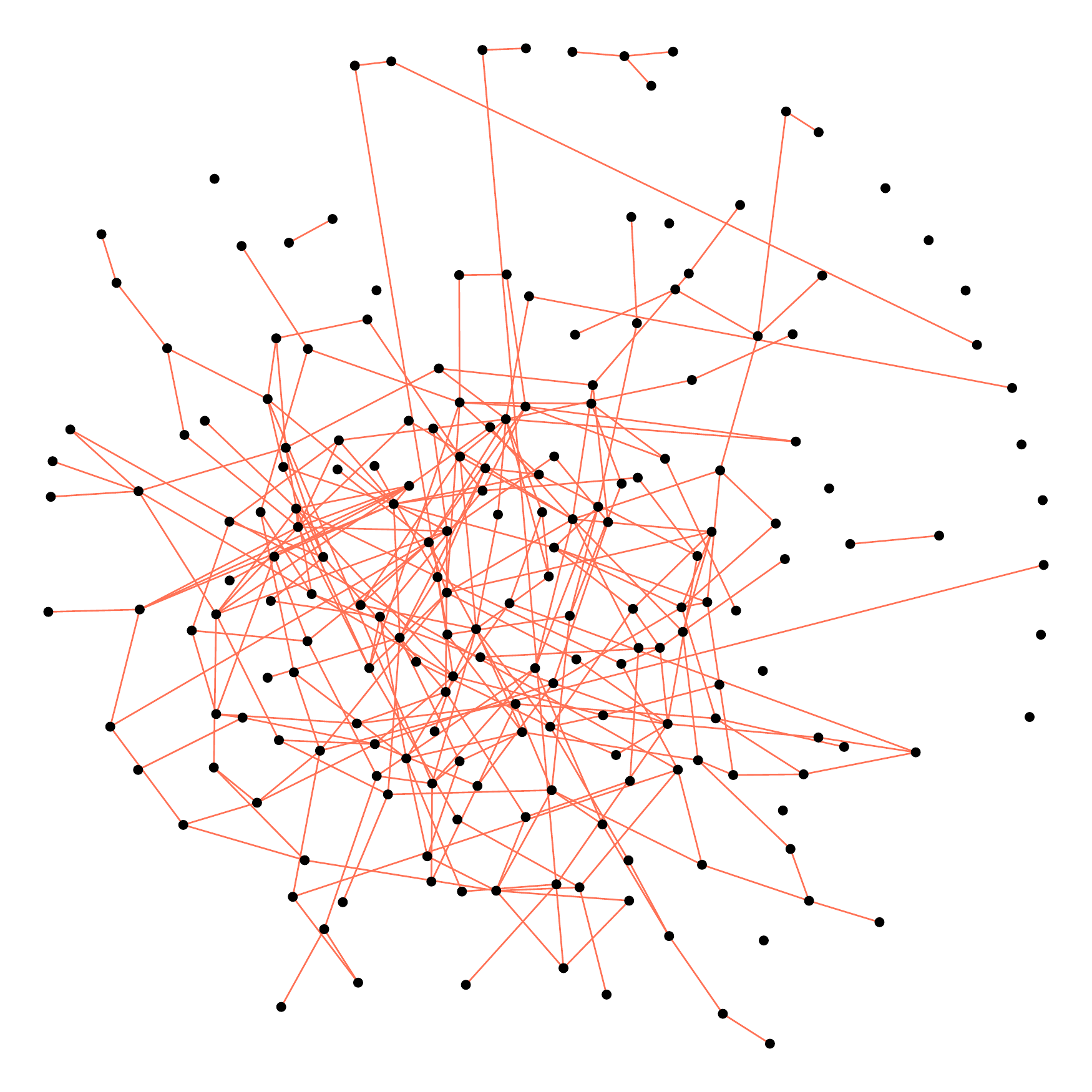}}
     \subfigure[]{  \includegraphics[width=0.45\columnwidth, height=0.37\columnwidth]{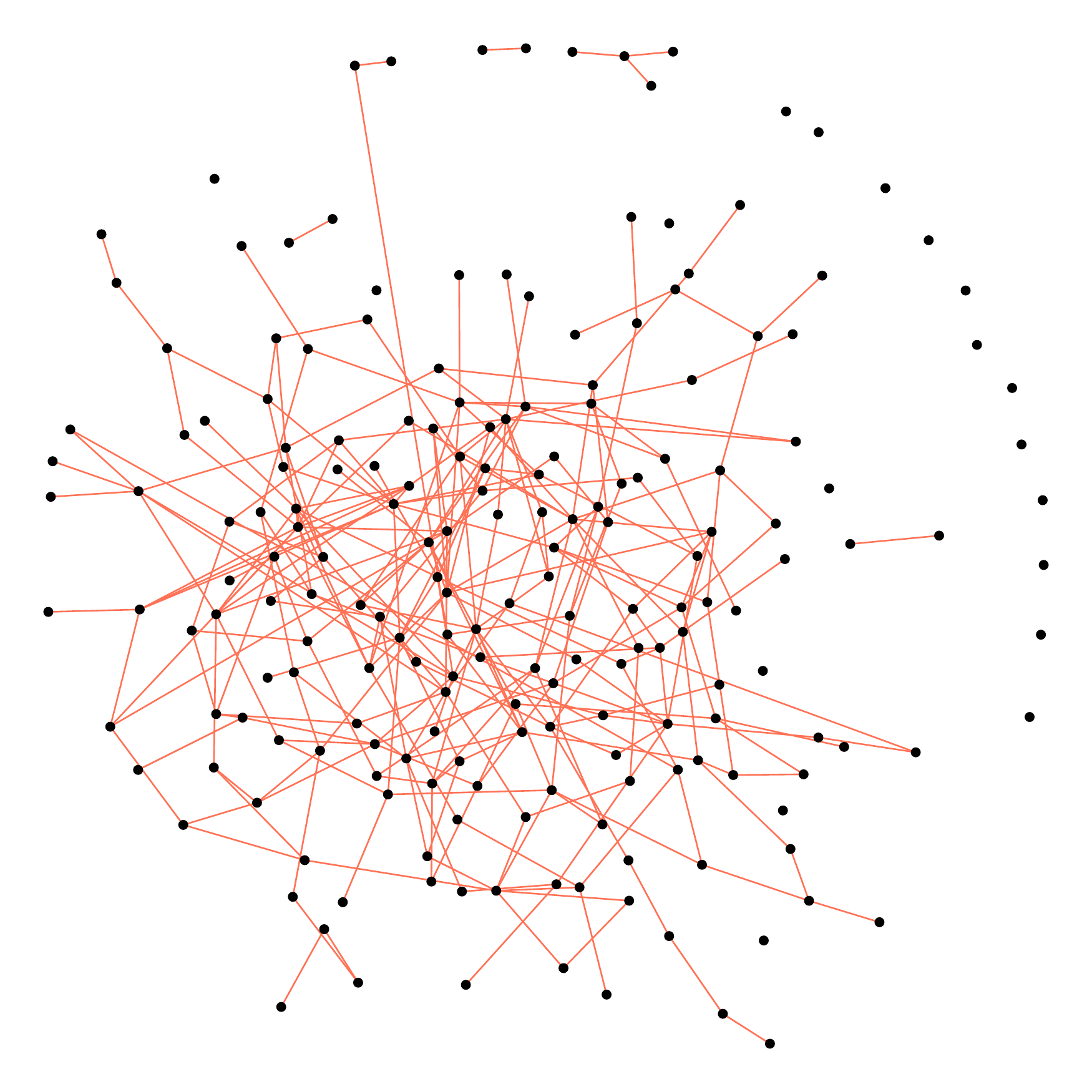}}}
 \centerline{
     \subfigure[]{  \includegraphics[width=0.45\columnwidth, height=0.37\columnwidth]{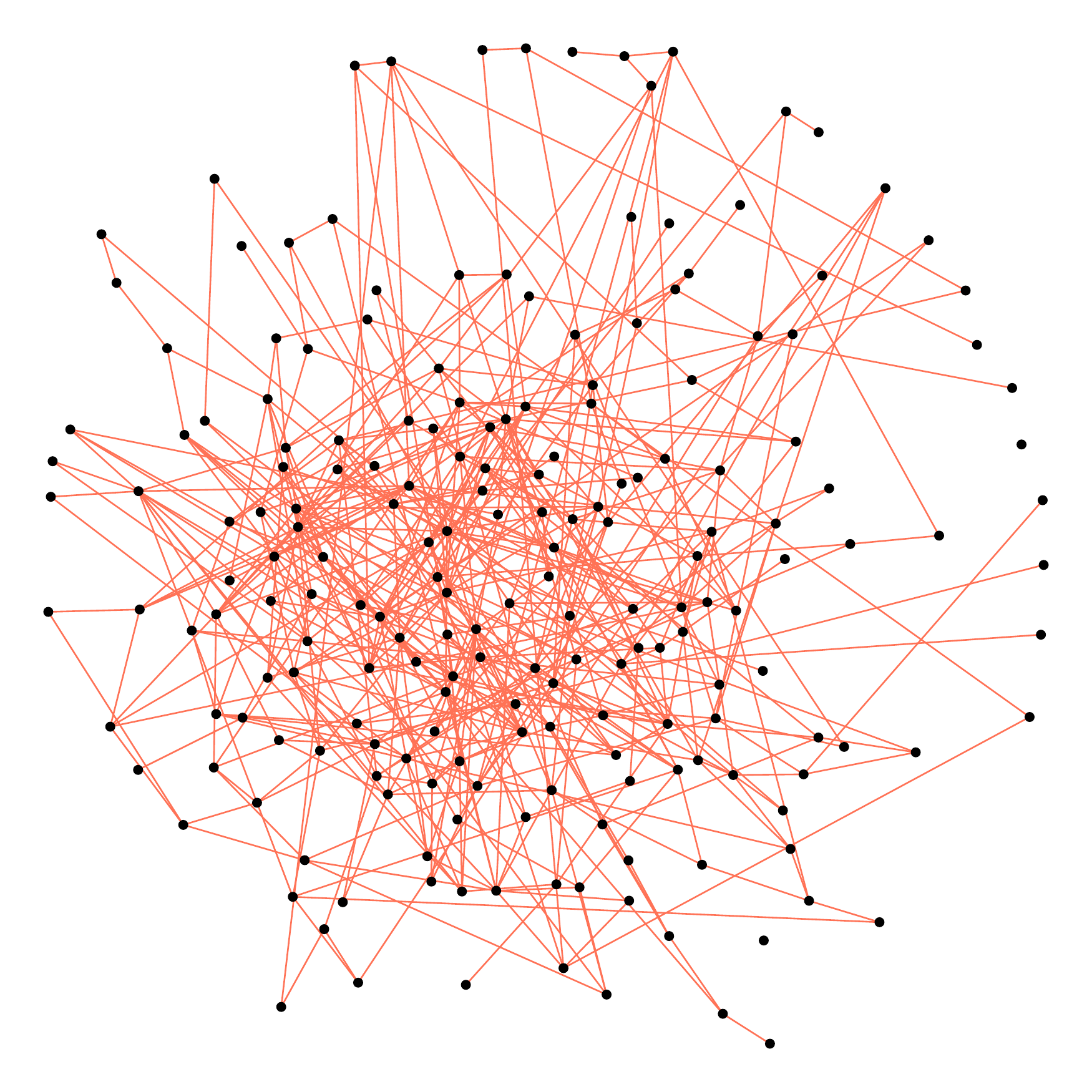}}
     \subfigure[]{  \includegraphics[width=0.45\columnwidth, height=0.37\columnwidth]{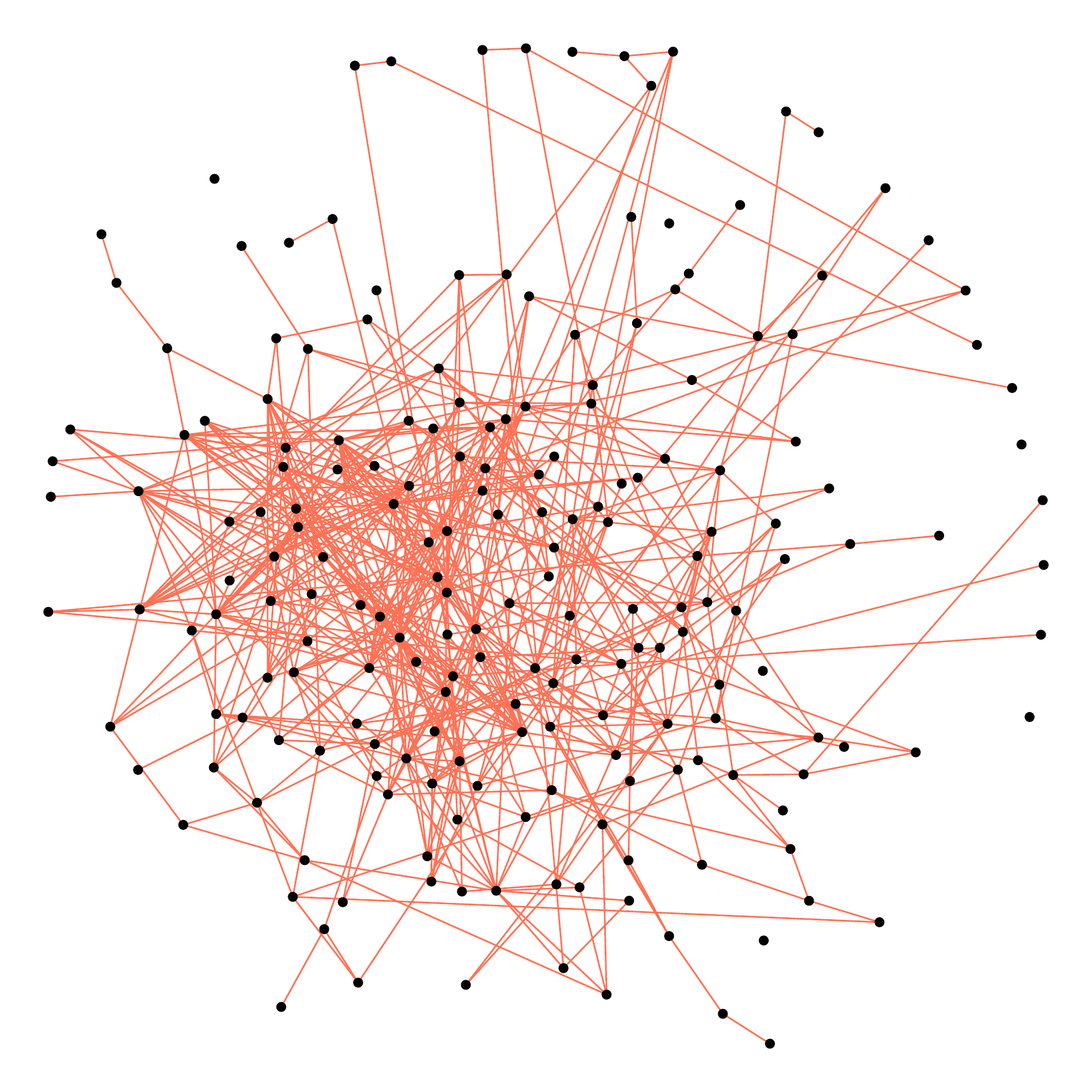}}}
     \caption{Examples of graph recovery for a random graph of size $d=200$ and $n=300$ samples,
     where (a) depicts the true graph, 
     (b) is the \oracle{}  ($F_1=0.79$), 
     (c) is the \thav{} ($F_1=0.89$), 
     (d) is the SCIO (CV) ($F_1=0.89$), 
     (e) is the rSME ($F_1=0.72$), and 
     (f) is the StARS ($F_1=0.65$).}
     \label{fig:exarecovery1}
 \end{figure*}

 \begin{figure*}
 \centerline{
     \subfigure[]{  \includegraphics[width=0.45\columnwidth, height=0.37\columnwidth]{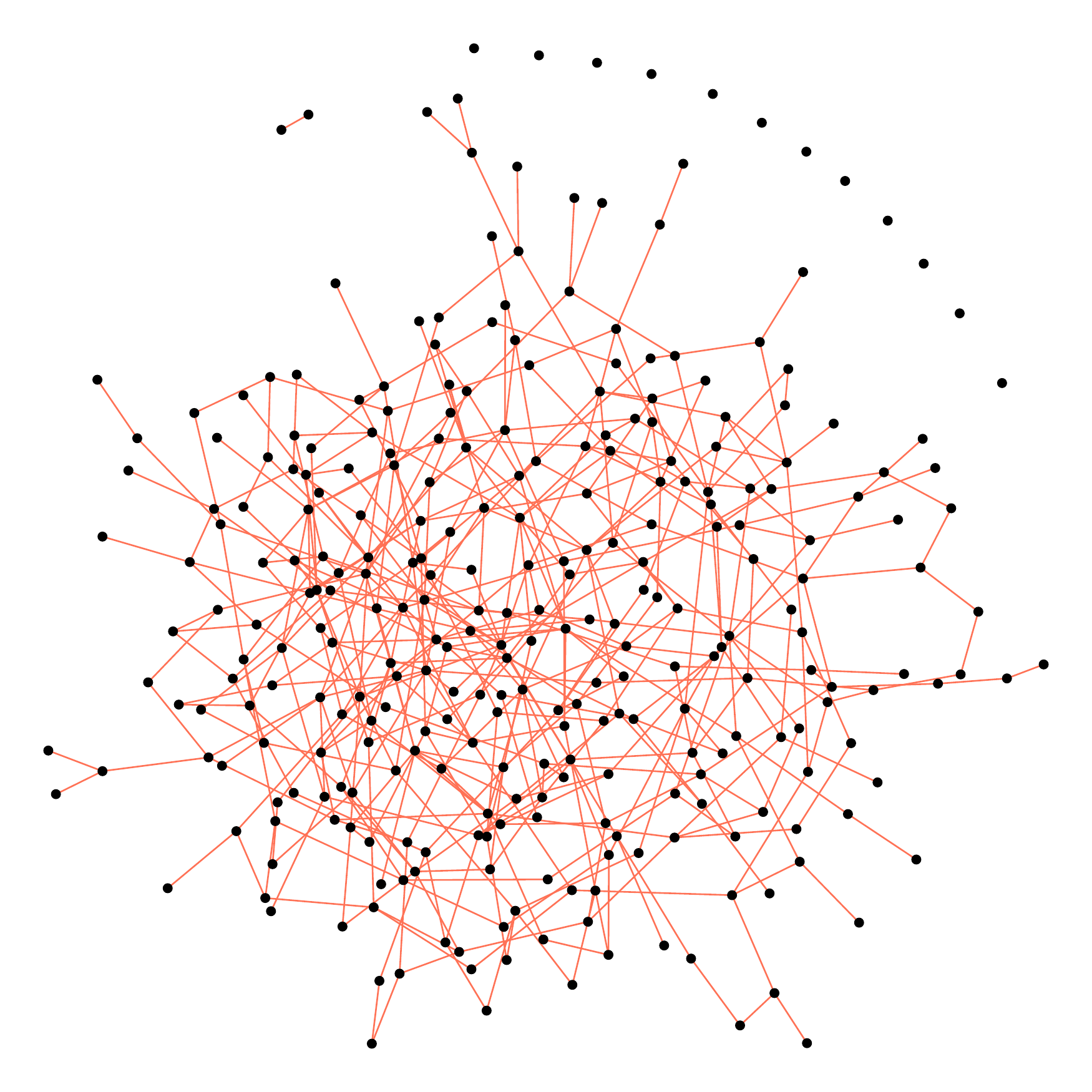}}
     \subfigure[]{  \includegraphics[width=0.45\columnwidth, height=0.37\columnwidth]{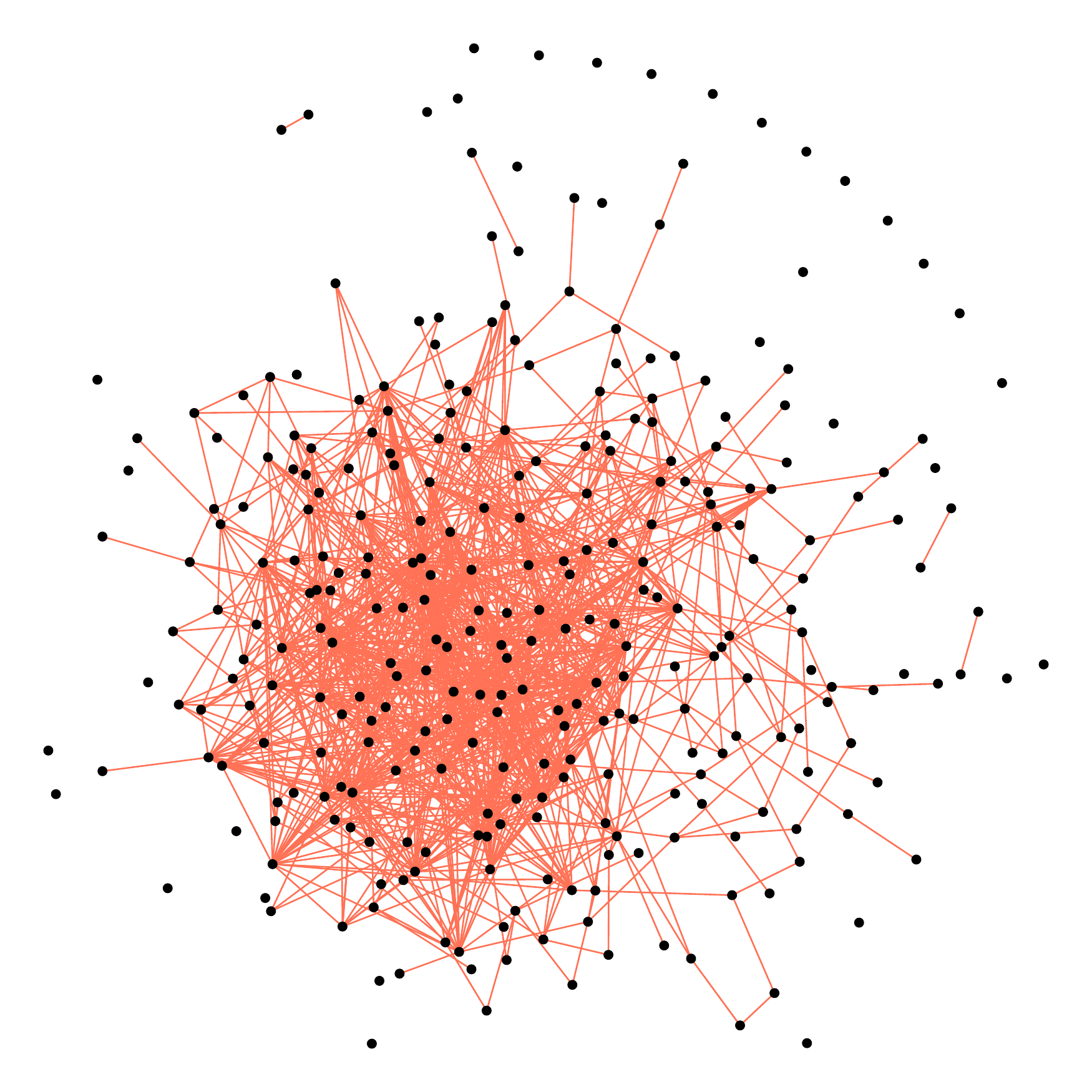}}}
 \centerline{    
     \subfigure[]{  \includegraphics[width=0.45\columnwidth, height=0.37\columnwidth]{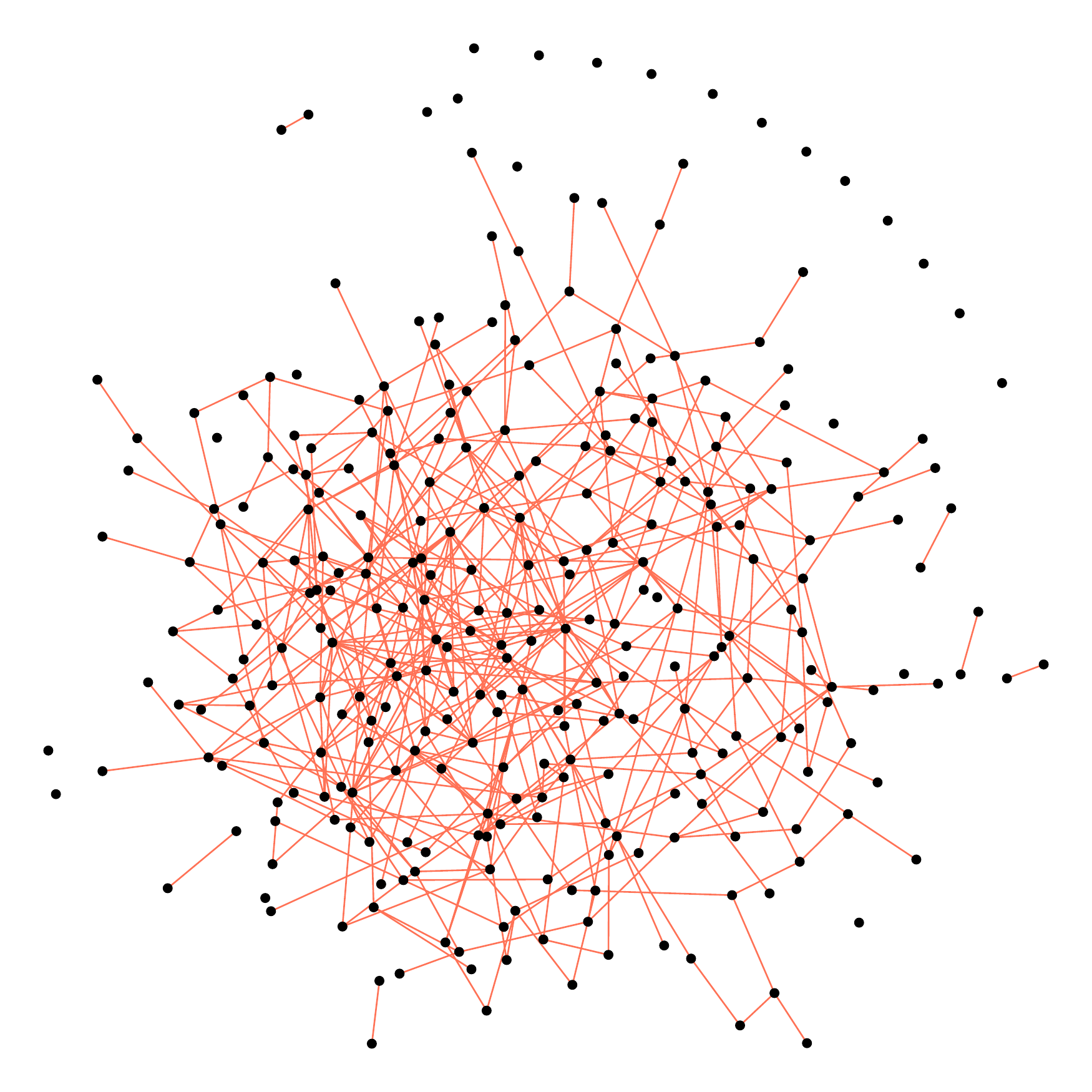}}
     \subfigure[]{  \includegraphics[width=0.45\columnwidth, height=0.37\columnwidth]{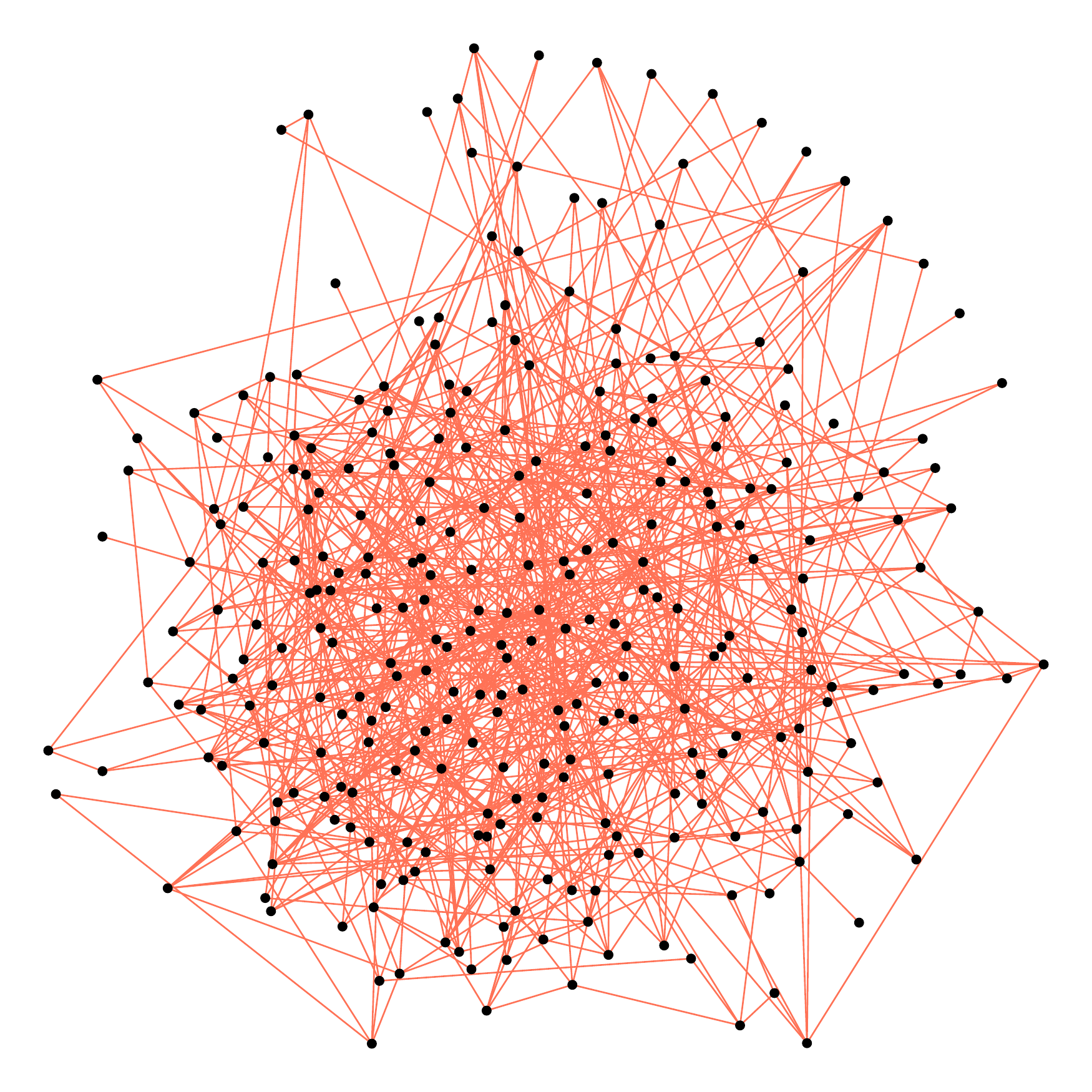}}}
 \centerline{
     \subfigure[]{  \includegraphics[width=0.45\columnwidth, height=0.37\columnwidth]{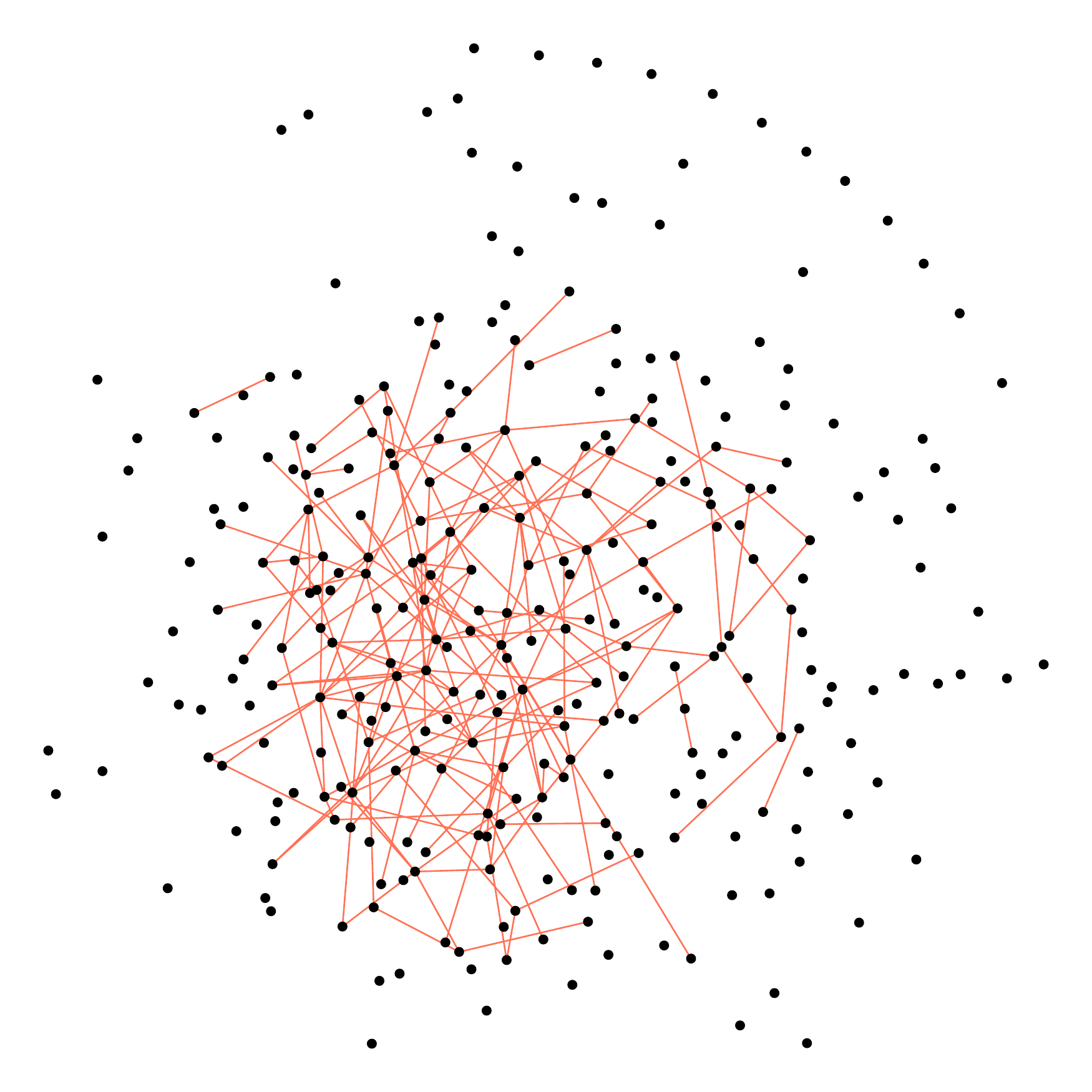}}
     \subfigure[]{  \includegraphics[width=0.45\columnwidth, height=0.37\columnwidth]{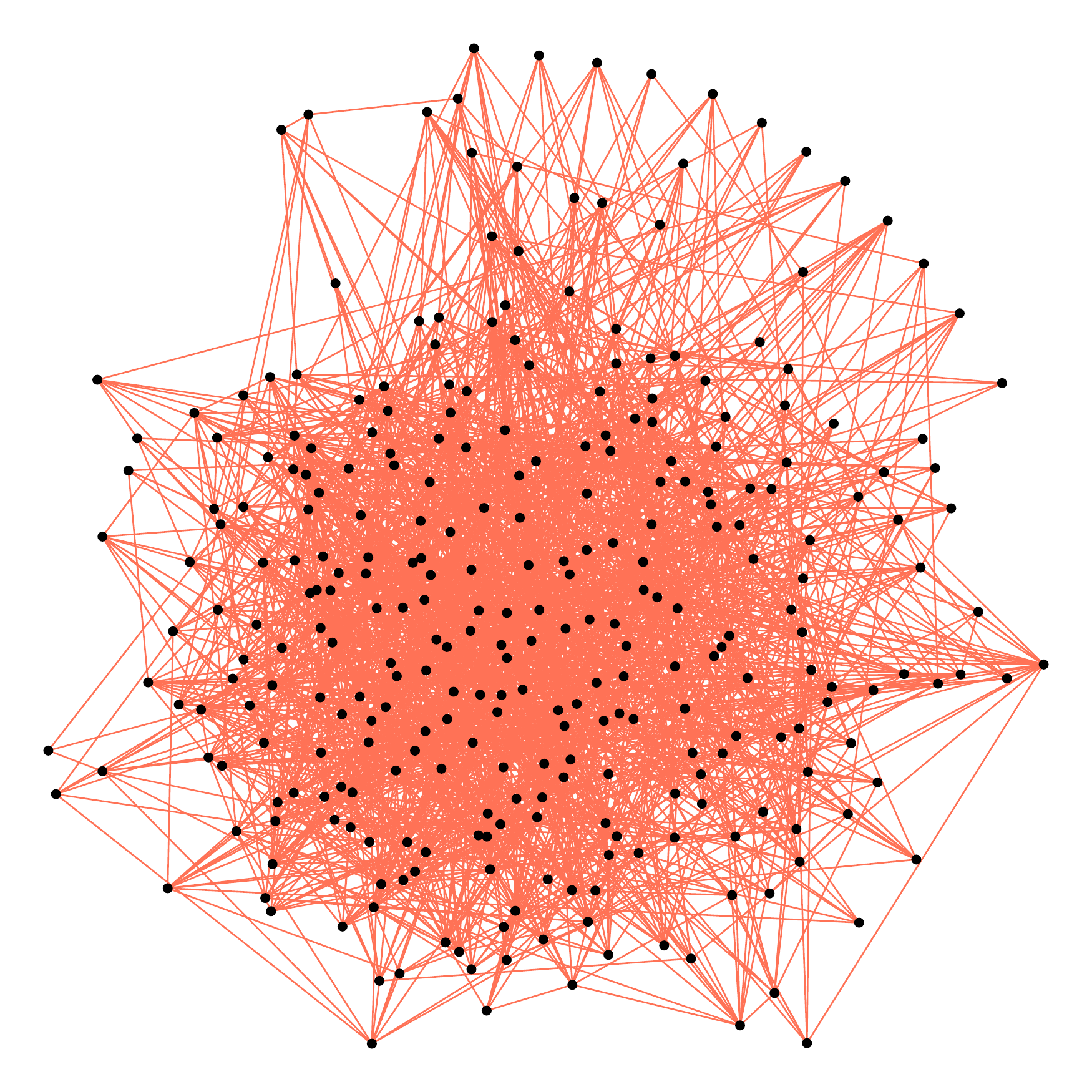}}}
     \caption{Examples of graph recovery for a random graph of size $d=300$ and $n=200$ samples,
     where (a) depicts the true graph, 
     (b) is the \oracle{}  ($F_1=0.49$), 
     (c) is the \thav{} ($F_1=0.86$), 
     (d) is the scaled lasso ($F_1=0.66$), 
     (e) is the SCIO (Bregman) ($F_1=0.54$), and 
     (f) is the TIGER ($F_1=0.39$).}
     \label{fig:exarecovery2}
 \end{figure*}

 \begin{figure*}
 \centerline{
     \subfigure[]{  \includegraphics[width=0.45\columnwidth, height=0.37\columnwidth]{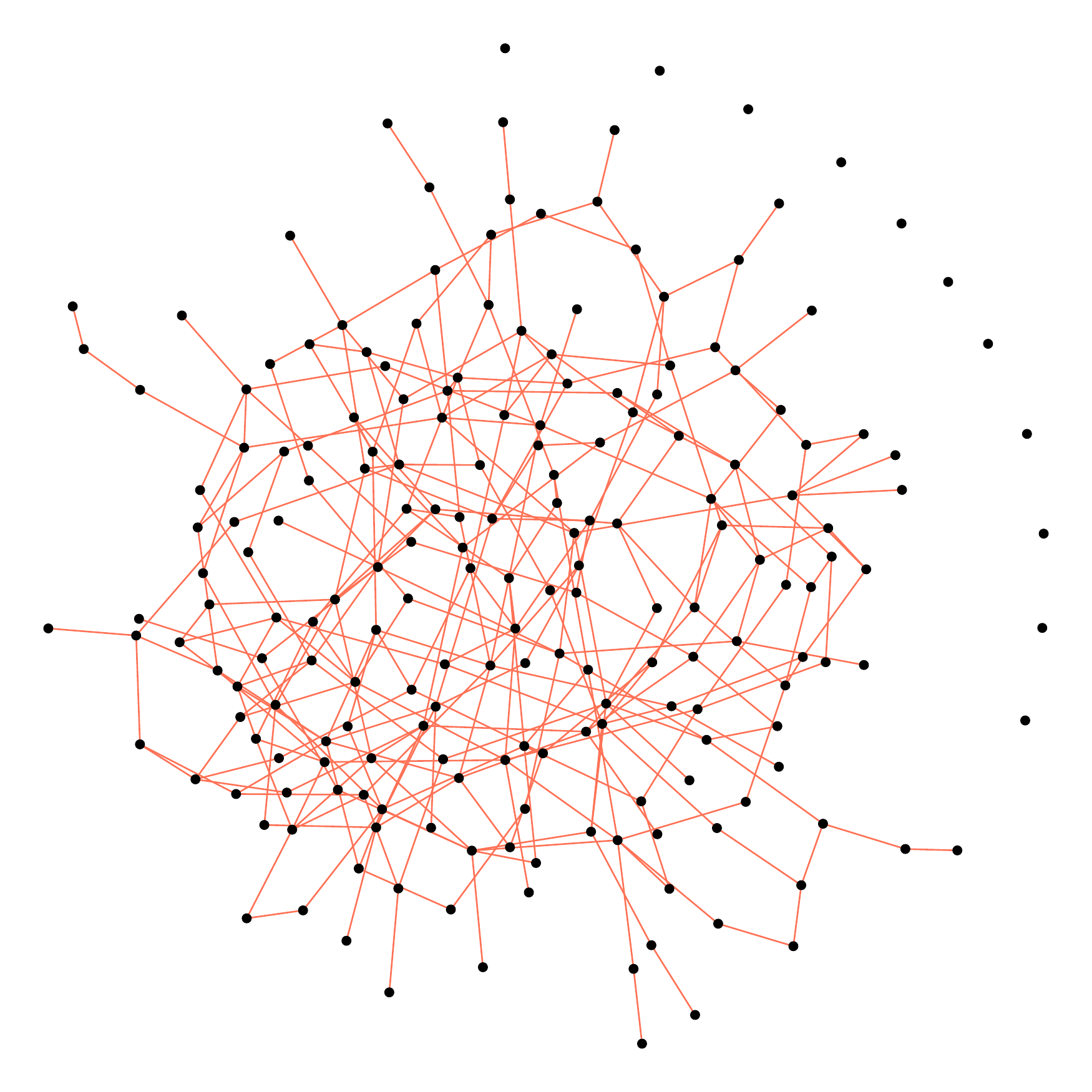}}
     \subfigure[]{  \includegraphics[width=0.45\columnwidth, height=0.37\columnwidth]{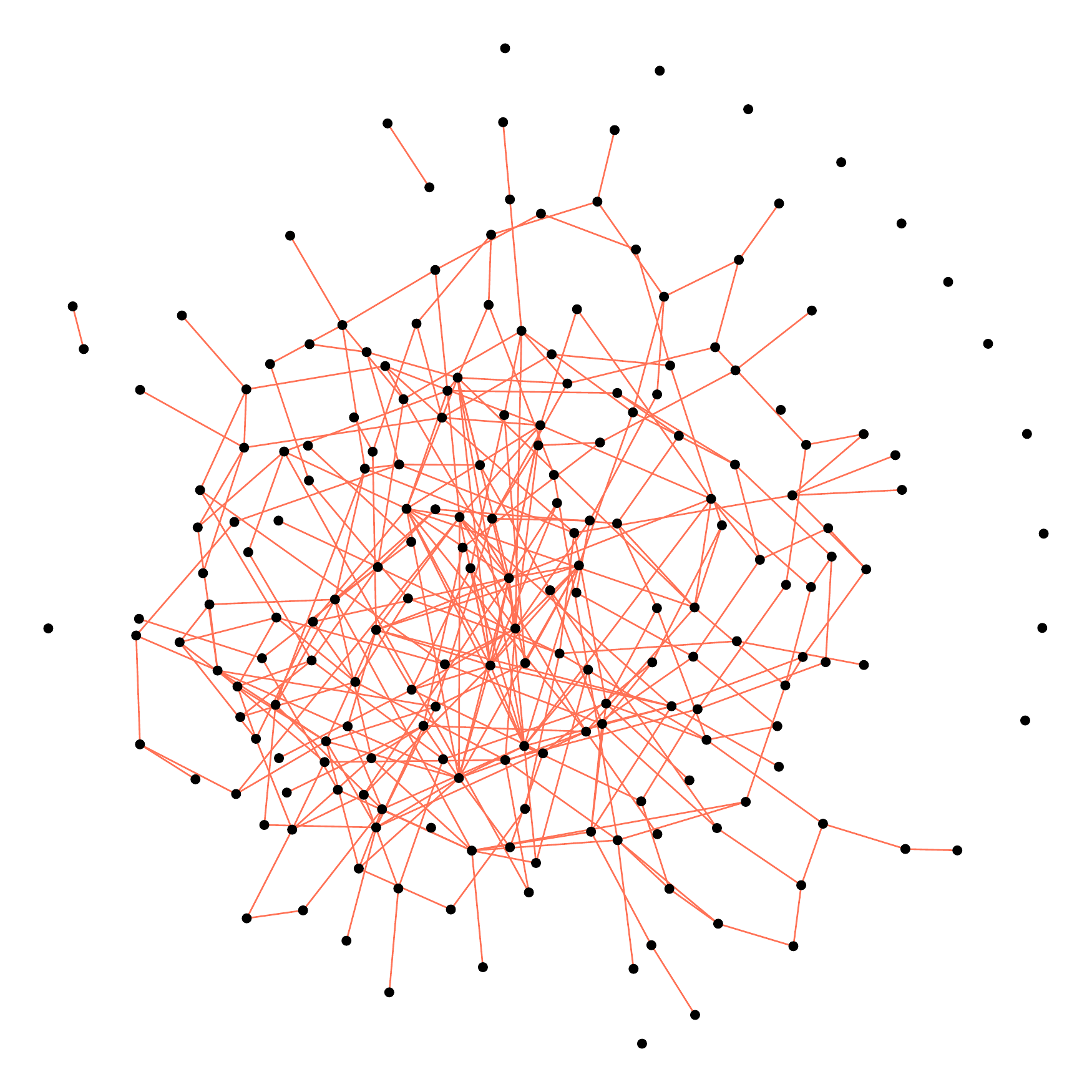} }}
 \centerline{    
     \subfigure[]{  \includegraphics[width=0.45\columnwidth, height=0.37\columnwidth]{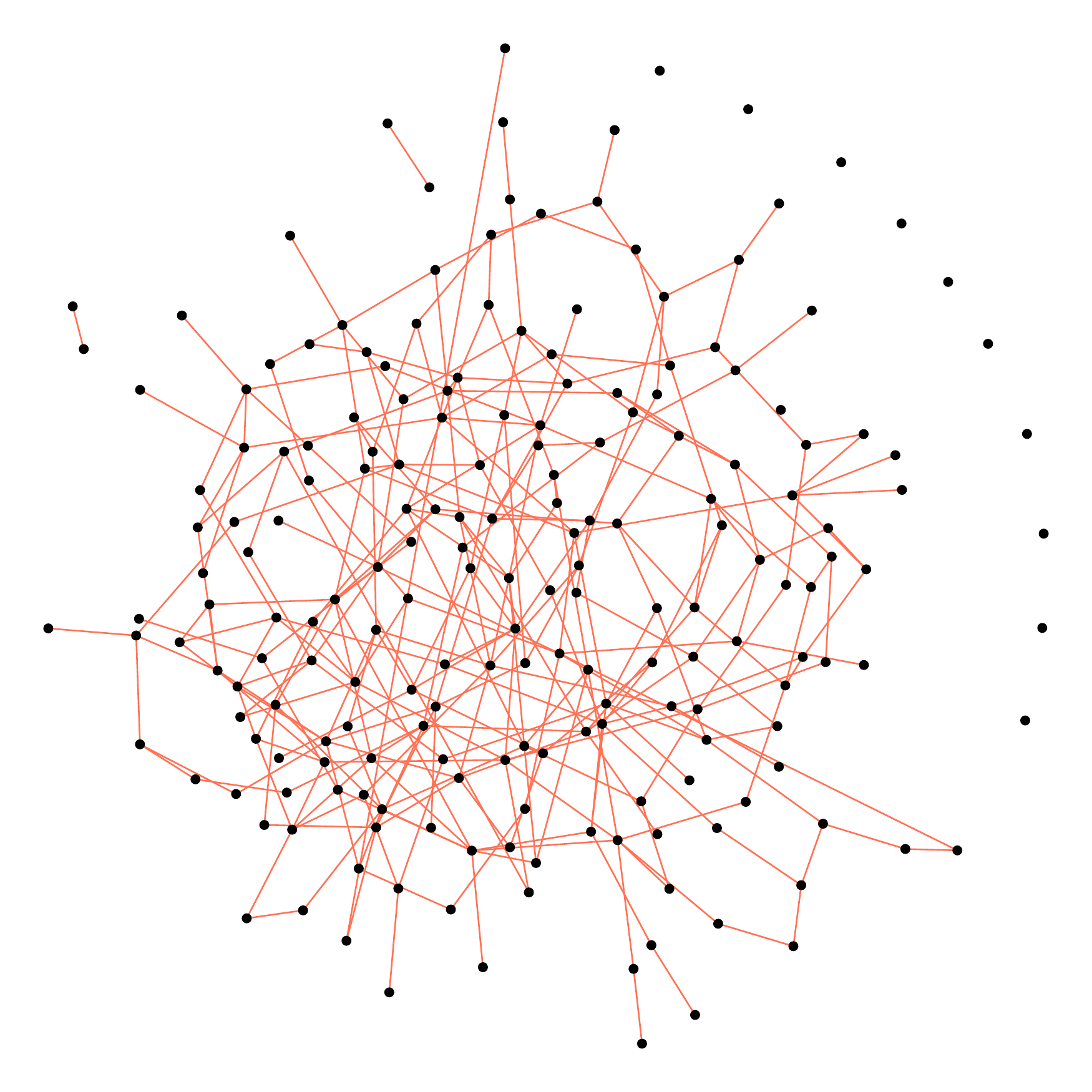}}
     \subfigure[]{  \includegraphics[width=0.45\columnwidth, height=0.37\columnwidth]{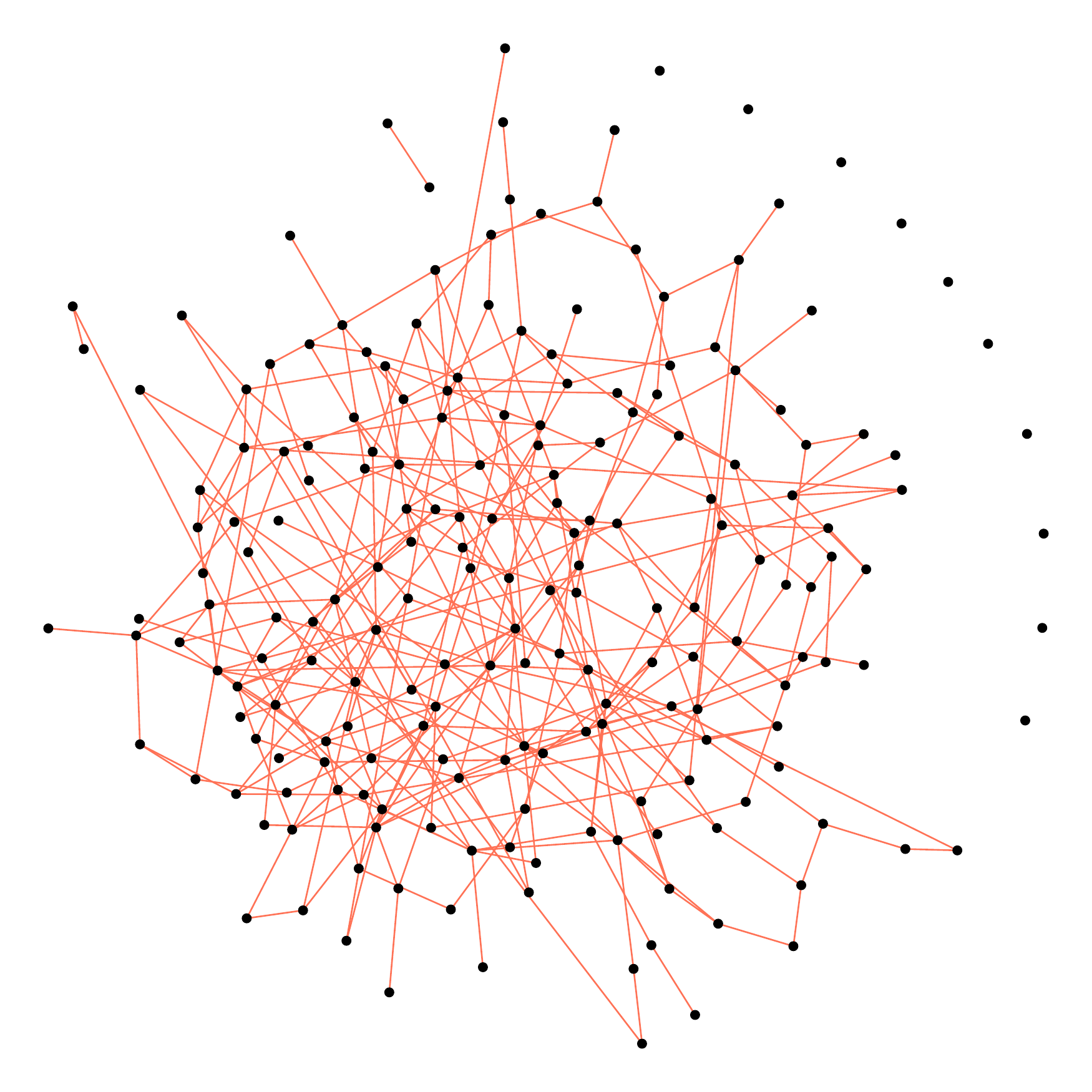}}}
 \centerline{
     \subfigure[]{  \includegraphics[width=0.45\columnwidth, height=0.37\columnwidth]{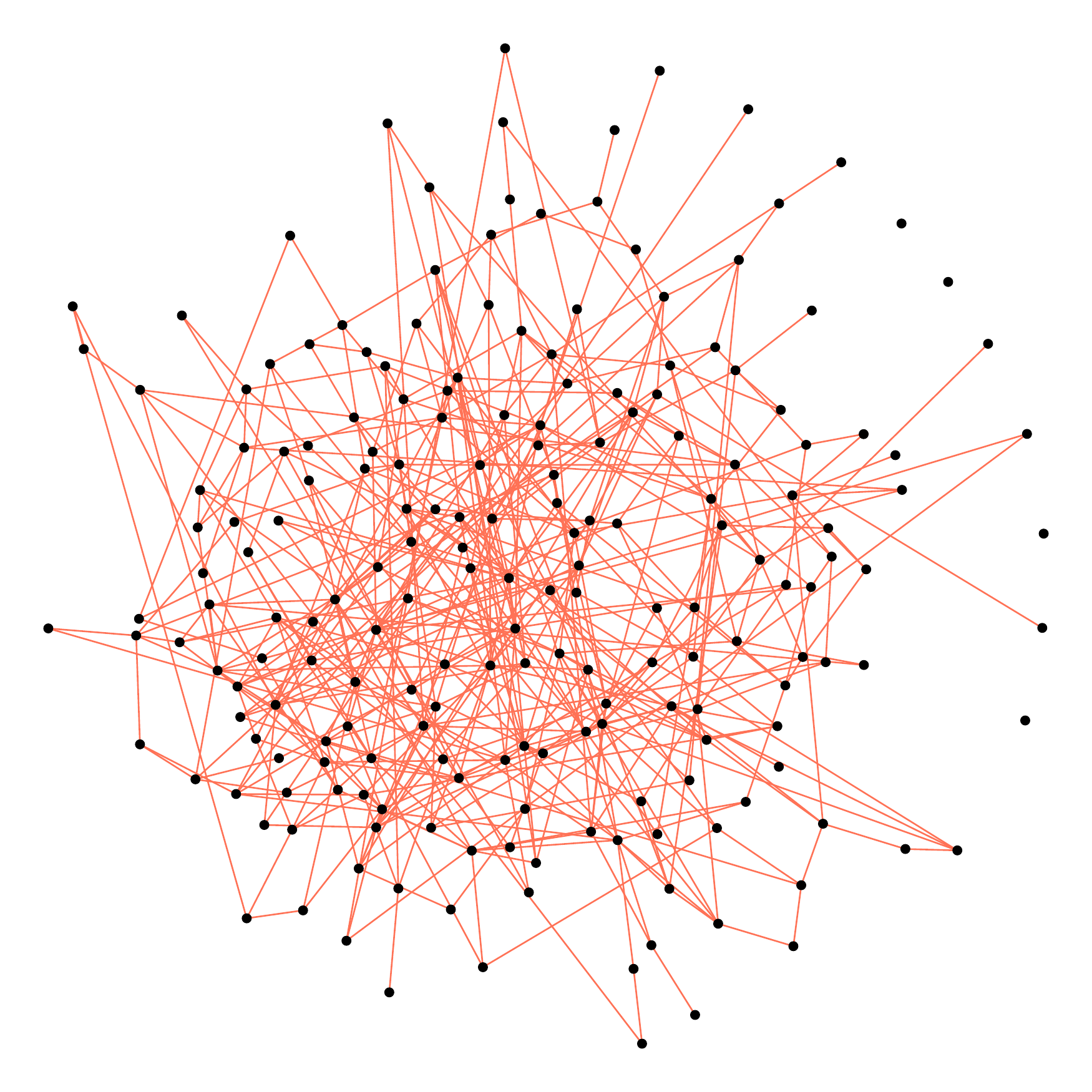}}
     \subfigure[]{  \includegraphics[width=0.45\columnwidth, height=0.37\columnwidth]{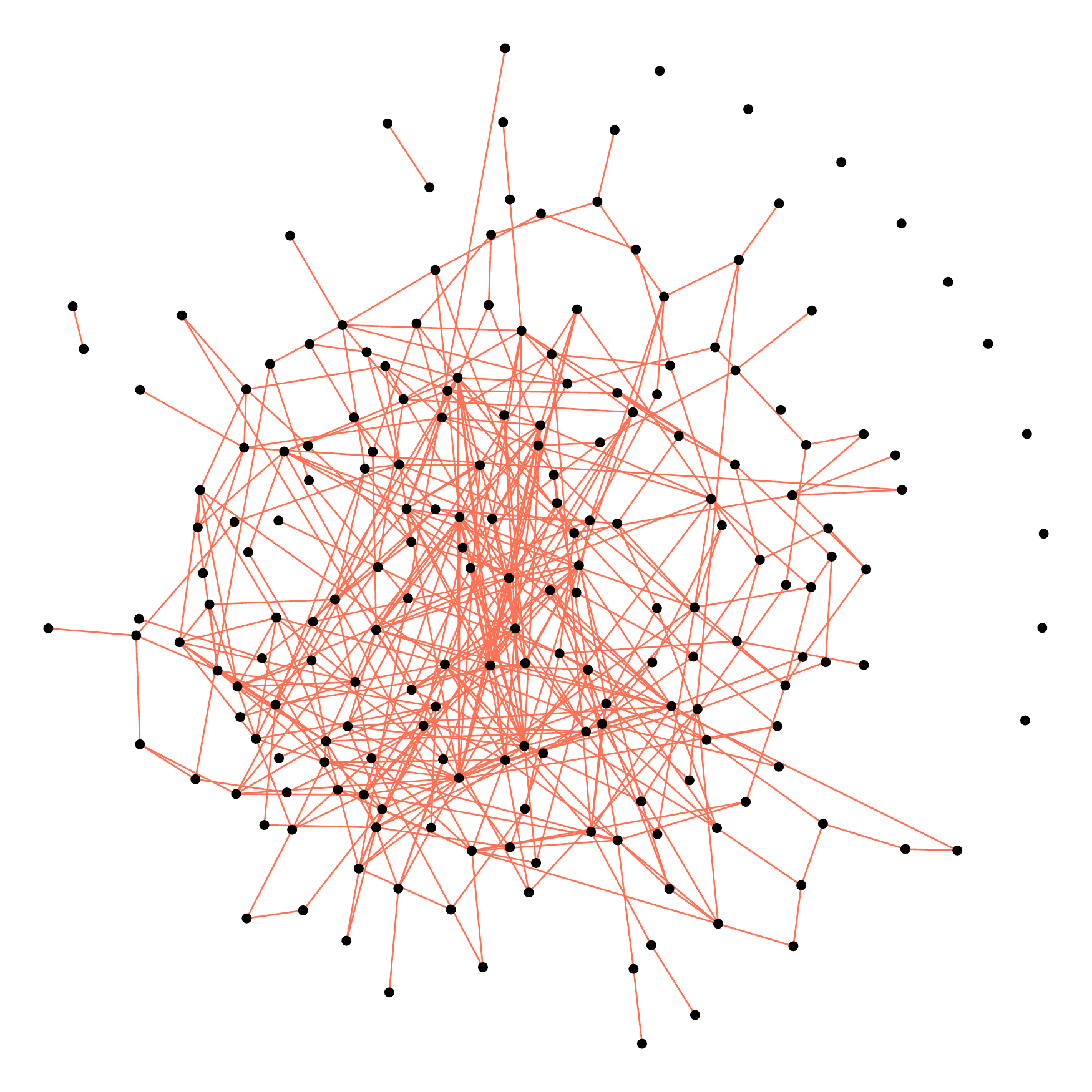}}}
     \caption{Examples of graph recovery for a random graph of size $d=200$ and $n=400$ samples,
     where (a) depicts the true graph, 
     (b) is the \oracle{}  ($F_1=0.86$), 
     (c) is the \thav{} ($F_1=0.96$), 
     (d) is the SCIO (CV) ($F_1=0.92$), 
     (e) is the rSME ($F_1=0.78$), and 
     (f) is the \stars{} ($F_1=0.76$).}
     \label{fig:exarecovery3}
 \end{figure*}

 \begin{figure*}
 \centerline{
     \subfigure[]{  \includegraphics[width=0.45\columnwidth, height=0.37\columnwidth]{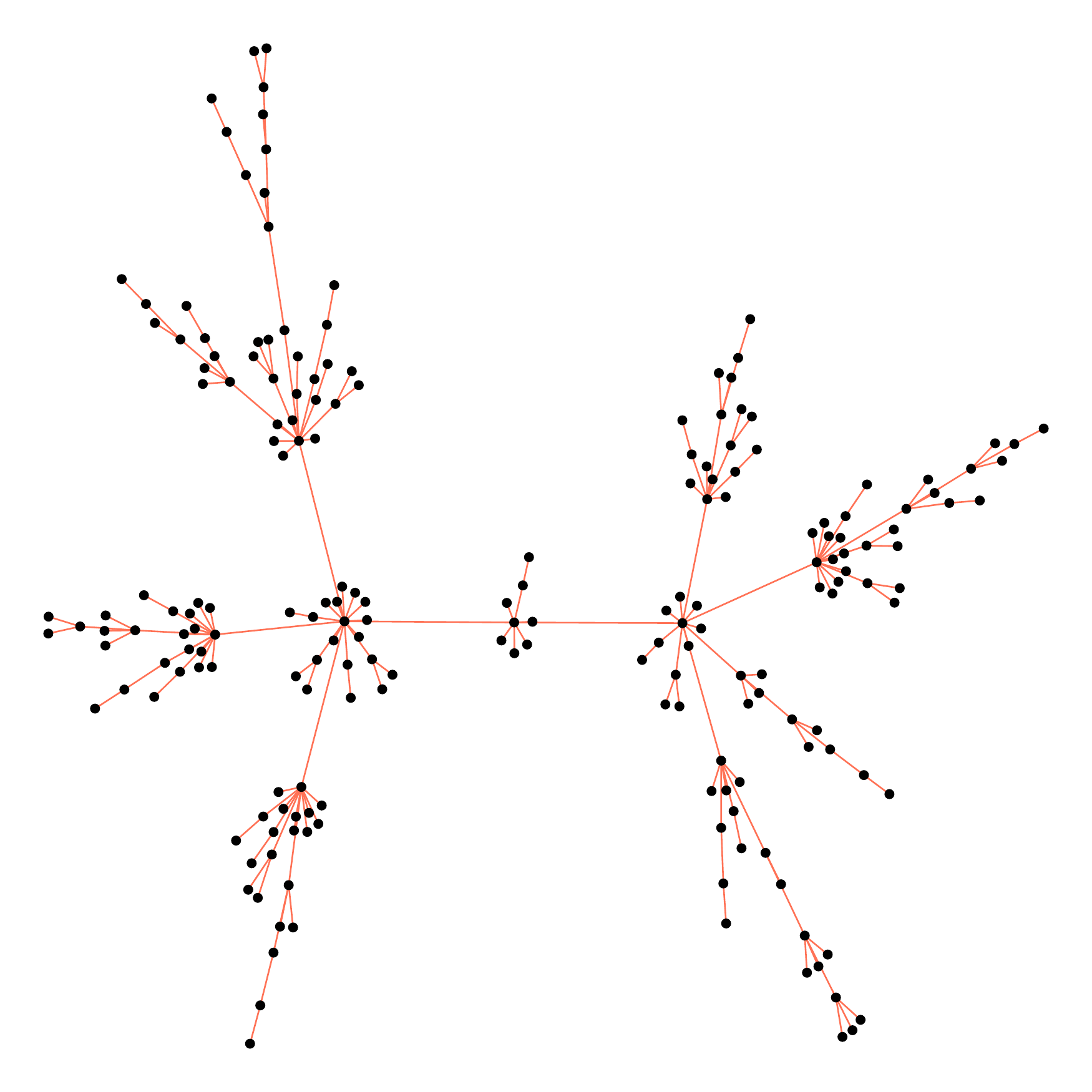}}
     \subfigure[]{  \includegraphics[width=0.45\columnwidth, height=0.37\columnwidth]{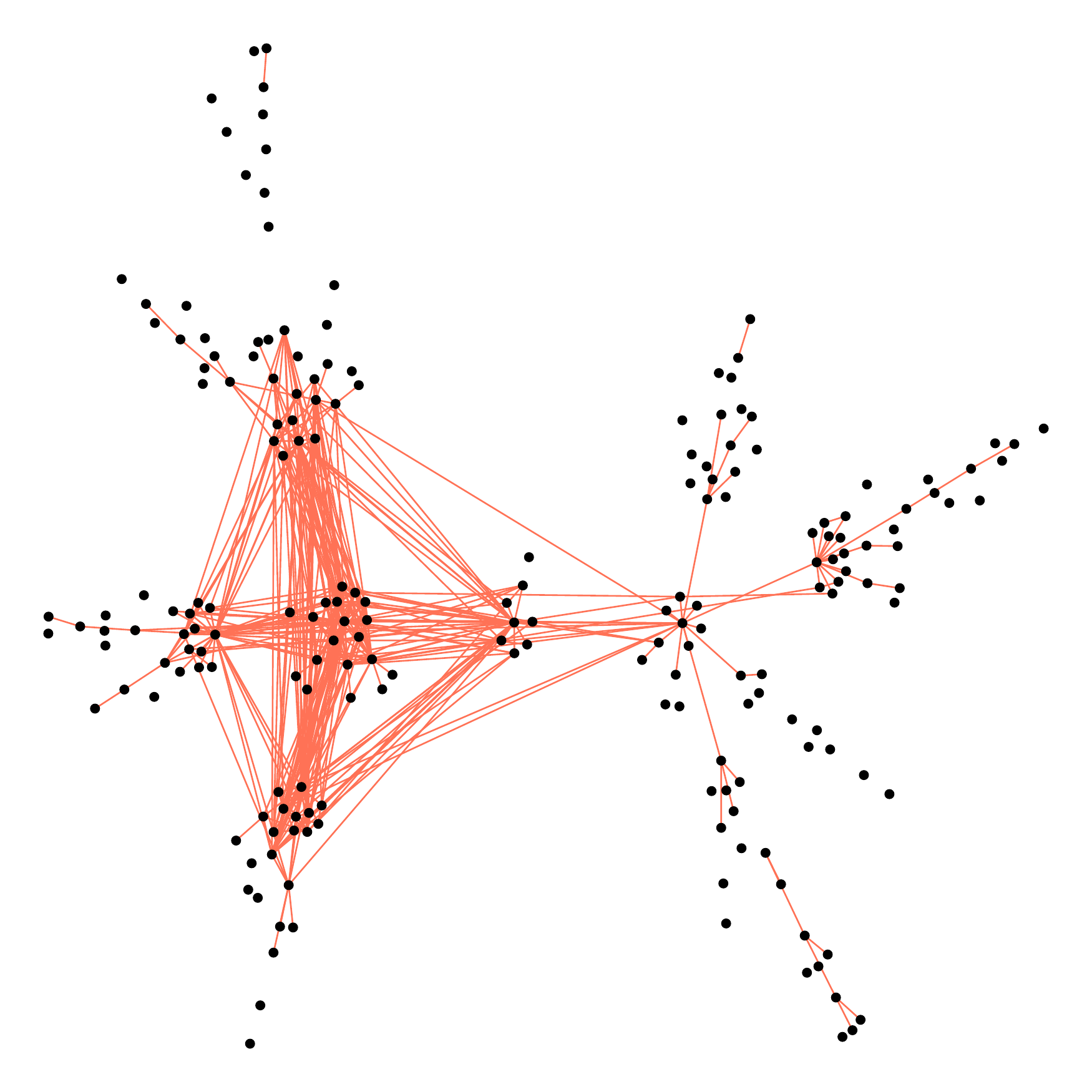} }}
 \centerline{    
     \subfigure[]{  \includegraphics[width=0.45\columnwidth, height=0.37\columnwidth]{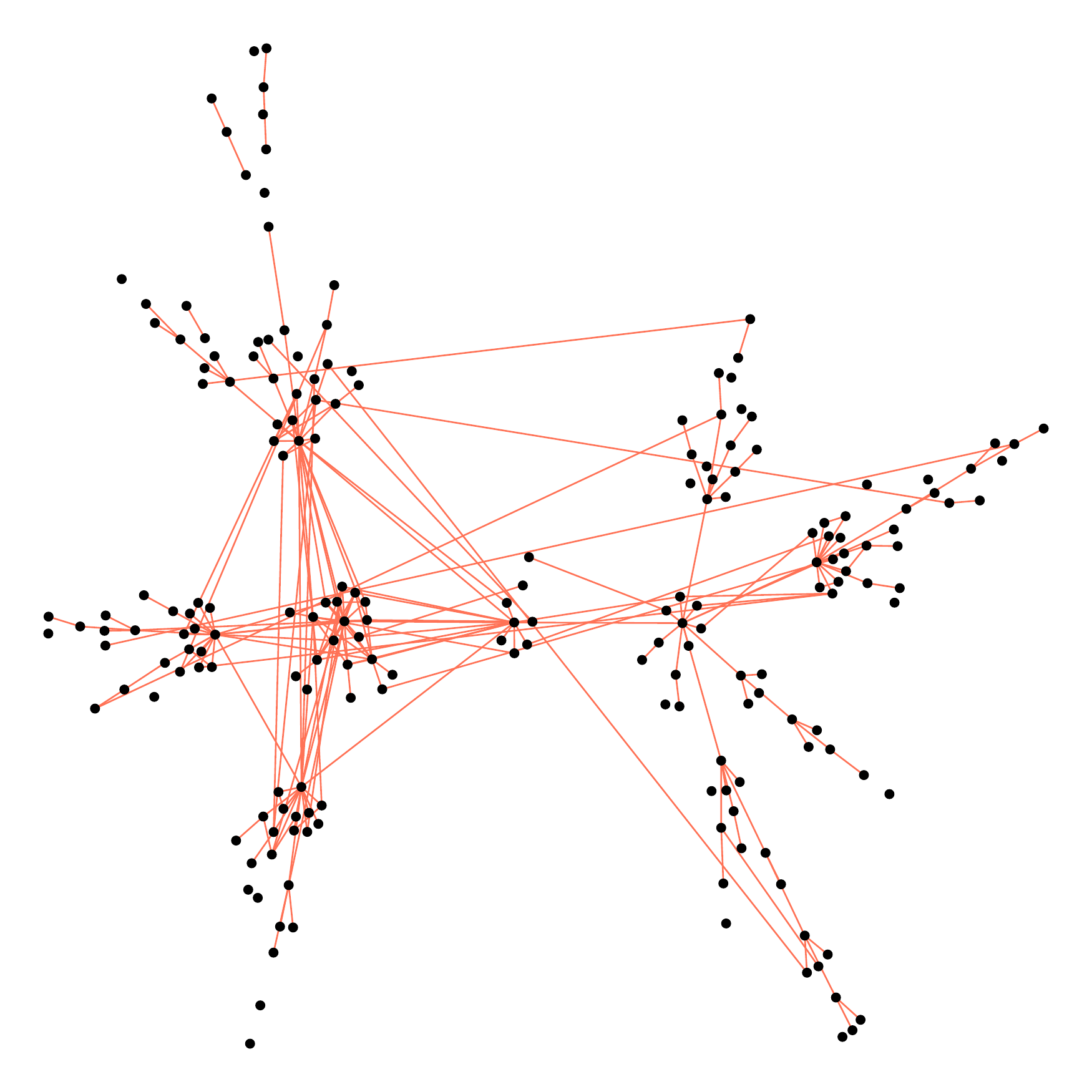}}
     \subfigure[]{  \includegraphics[width=0.45\columnwidth, height=0.37\columnwidth]{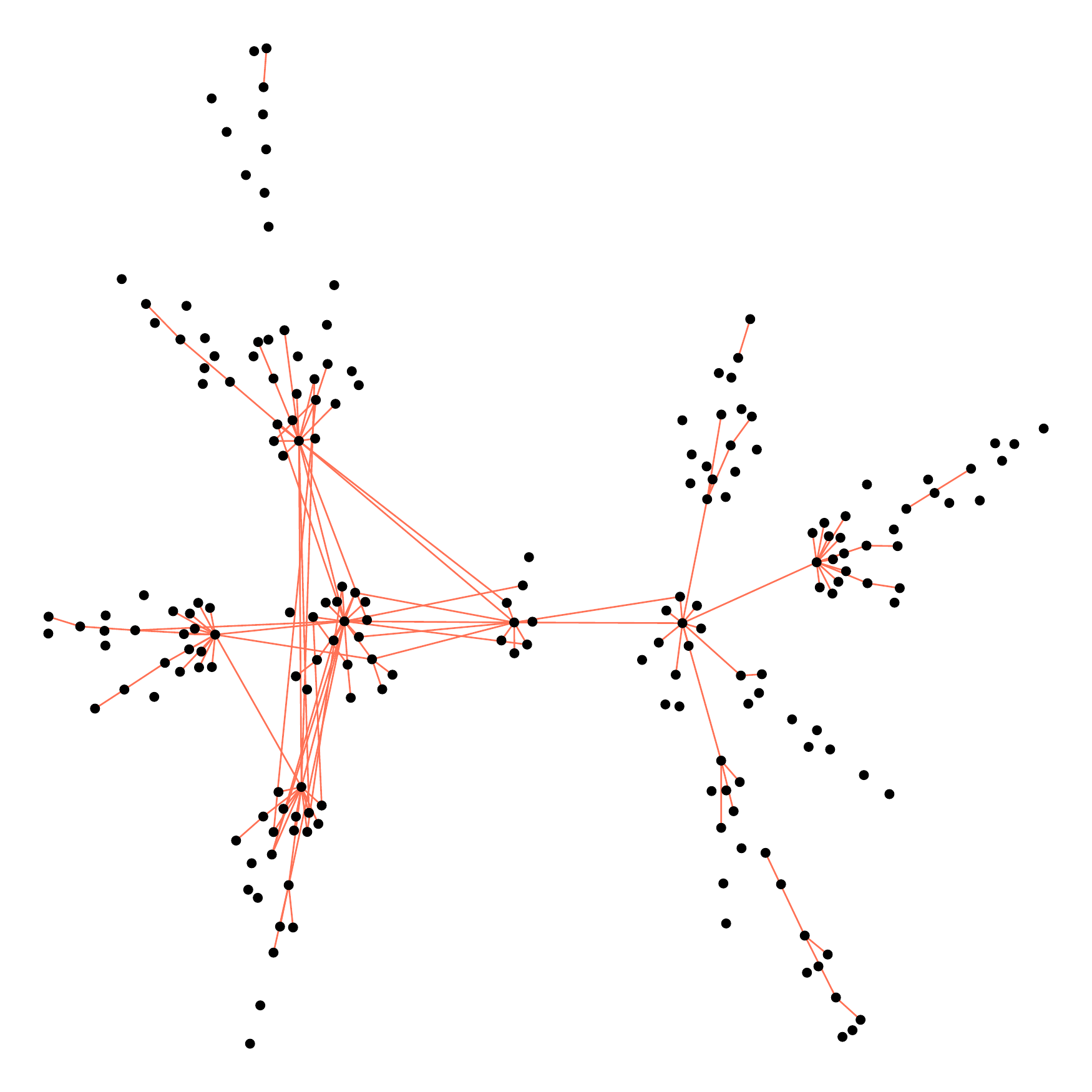}}}
 \centerline{
     \subfigure[]{  \includegraphics[width=0.45\columnwidth, height=0.37\columnwidth]{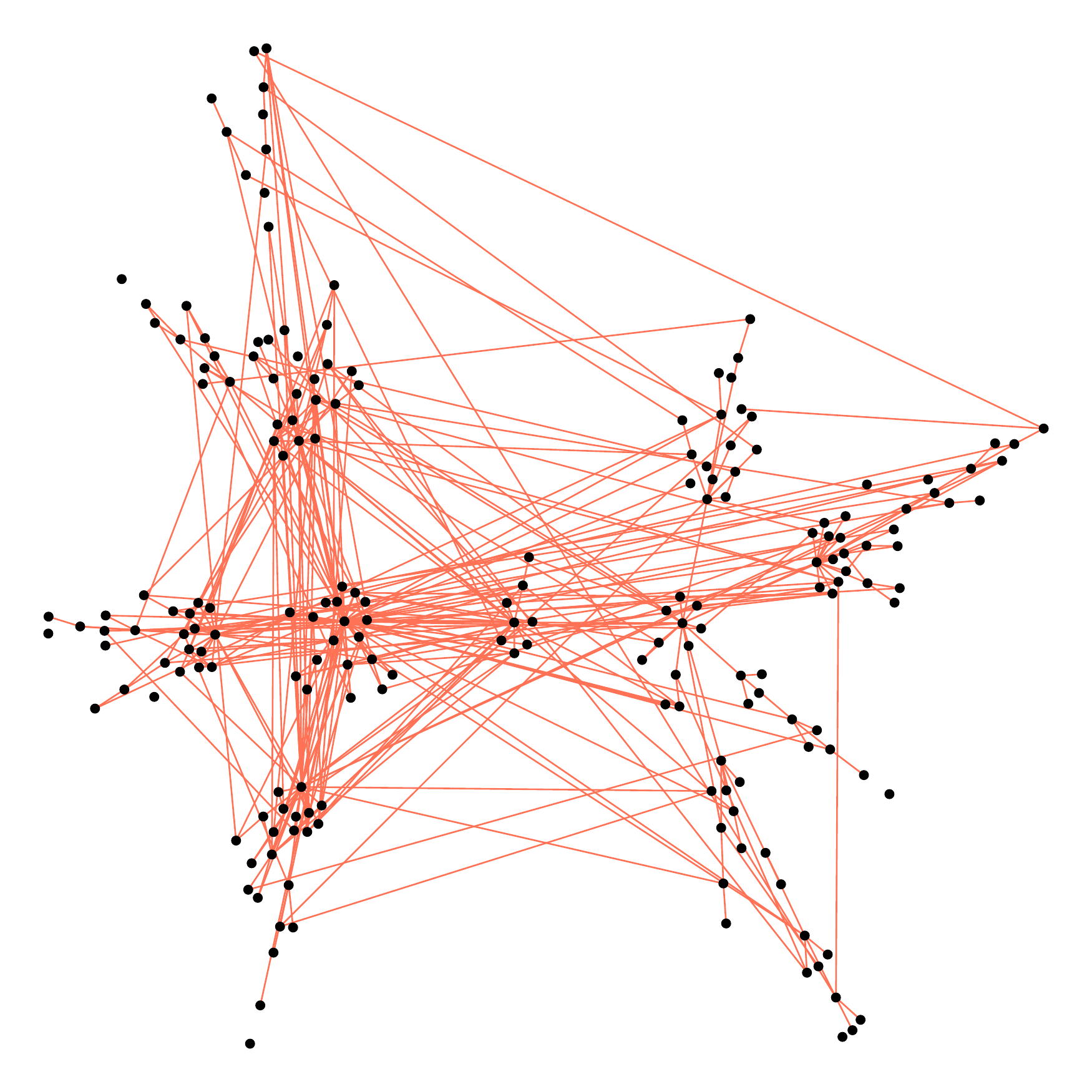}}
     \subfigure[]{  \includegraphics[width=0.45\columnwidth, height=0.37\columnwidth]{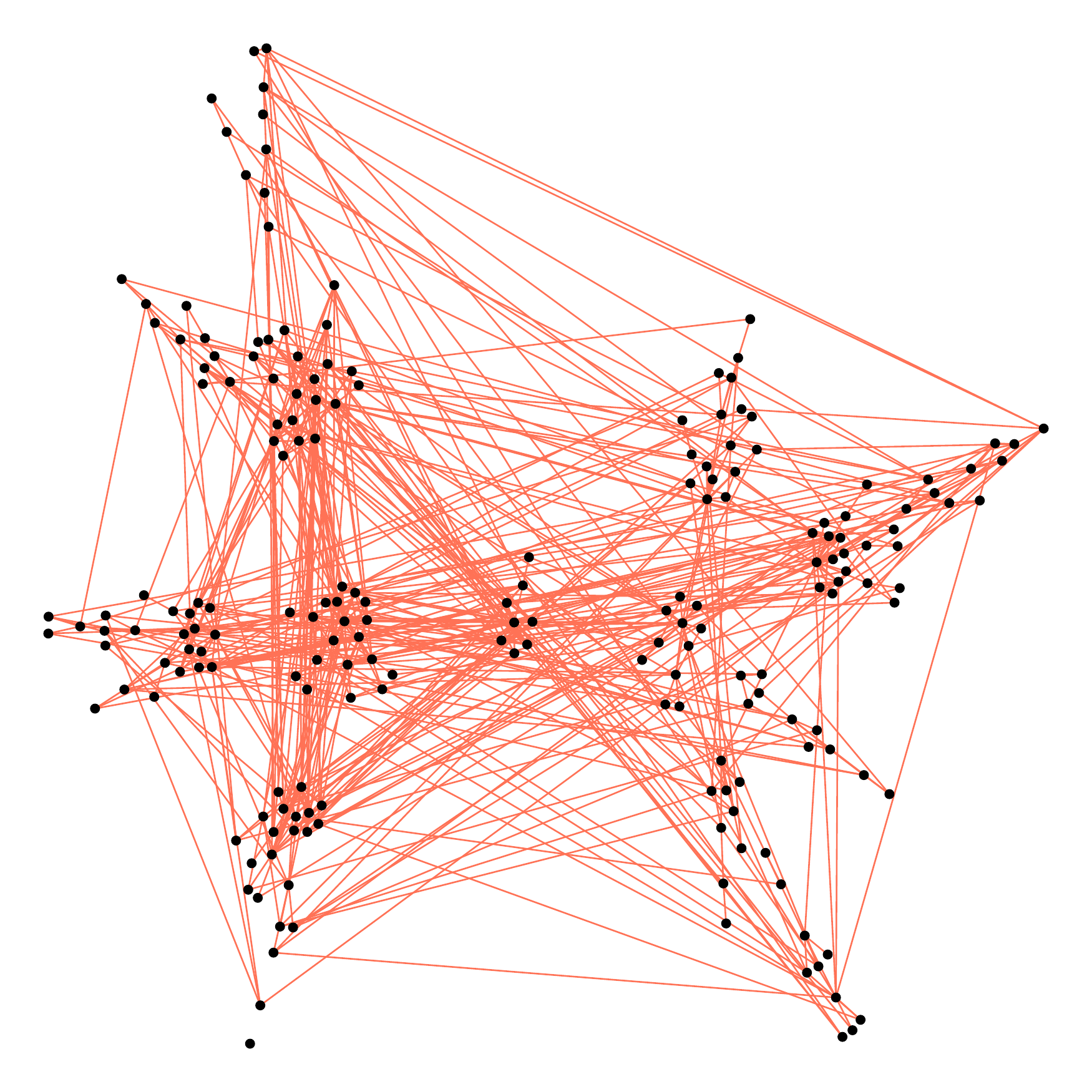}}}
     \caption{
     Examples of graph recovery for a scale-free graph of size $d=200$ and $n=300$ samples,
     where 
     (a) depicts the true graph, 
     (b) is the \oracle{}  ($F_1=0.42$), 
     (c) is the \thav{} ($F_1=0.73$), 
     (d) is the SCIO (Bregman) ($F_1=0.68$), 
     (e) is the rSME ($F_1=0.59$), and 
     (f) is the TIGER ($F_1=0.50$).}
     \label{fig:exarecovery4}
 \end{figure*}
\begin{figure*}
 \centerline{
     \subfigure[]{  \includegraphics[width=0.45\columnwidth, height=0.37\columnwidth]{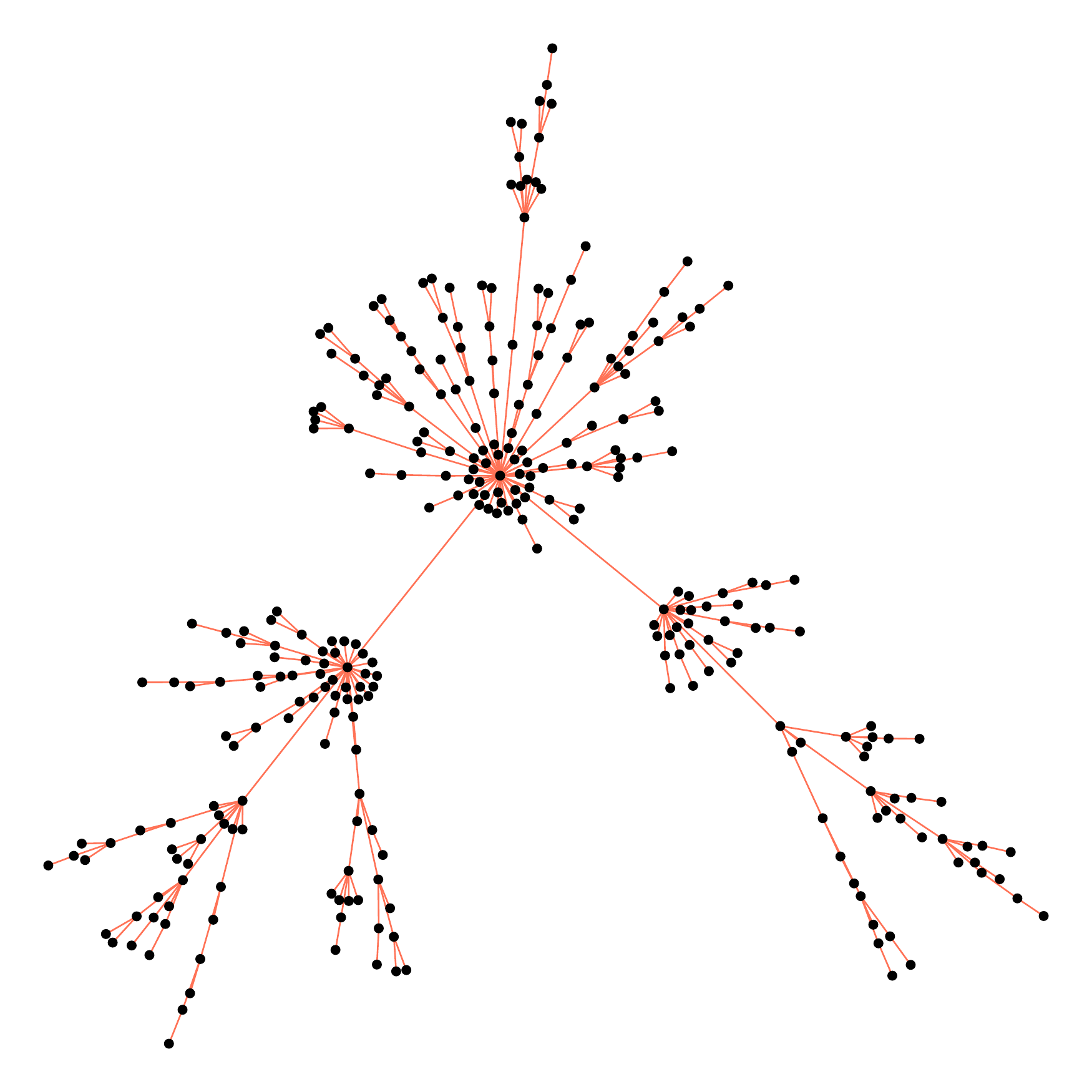}}
     \subfigure[]{  \includegraphics[width=0.45\columnwidth, height=0.37\columnwidth]{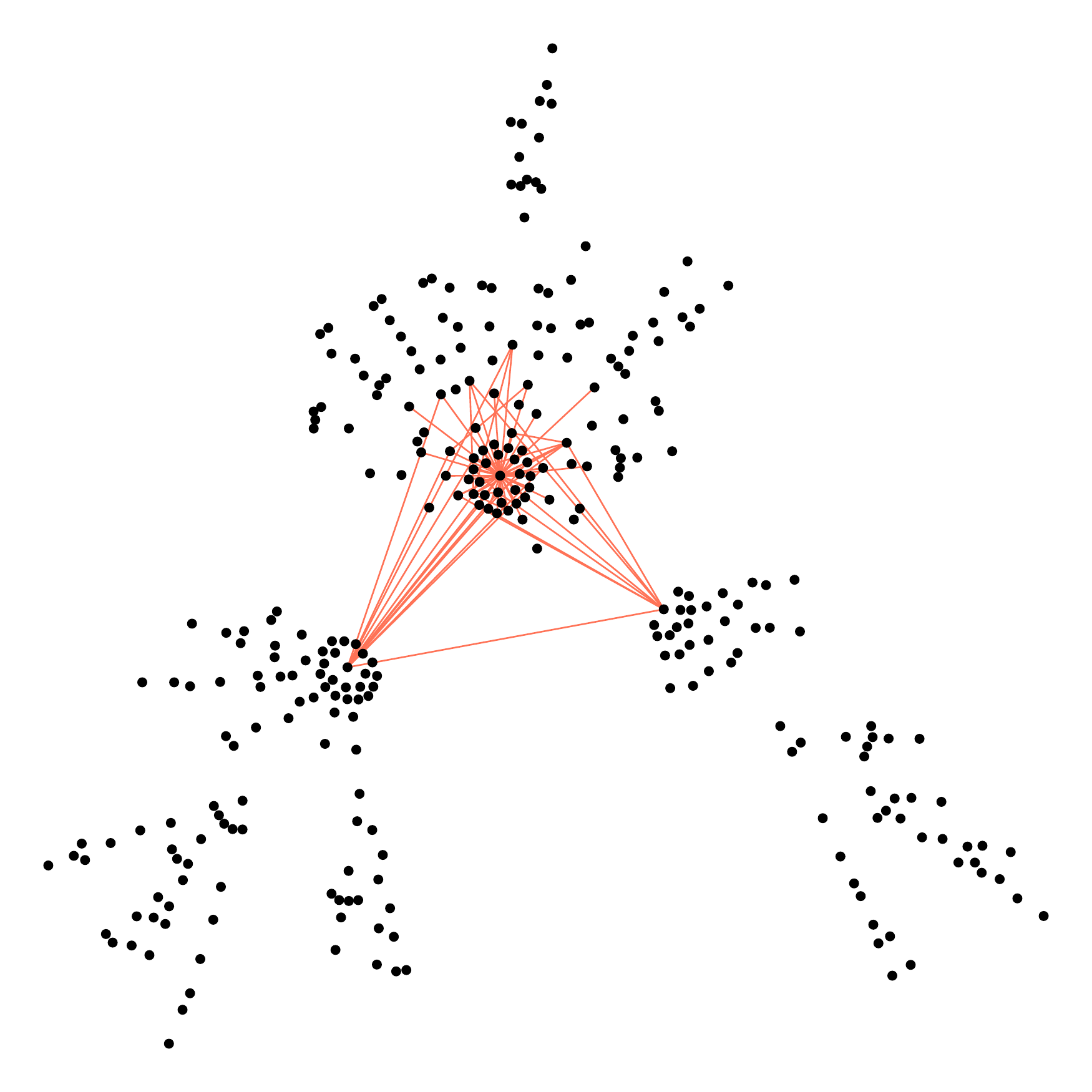}}}
 \centerline{    
     \subfigure[]{  \includegraphics[width=0.45\columnwidth, height=0.37\columnwidth]{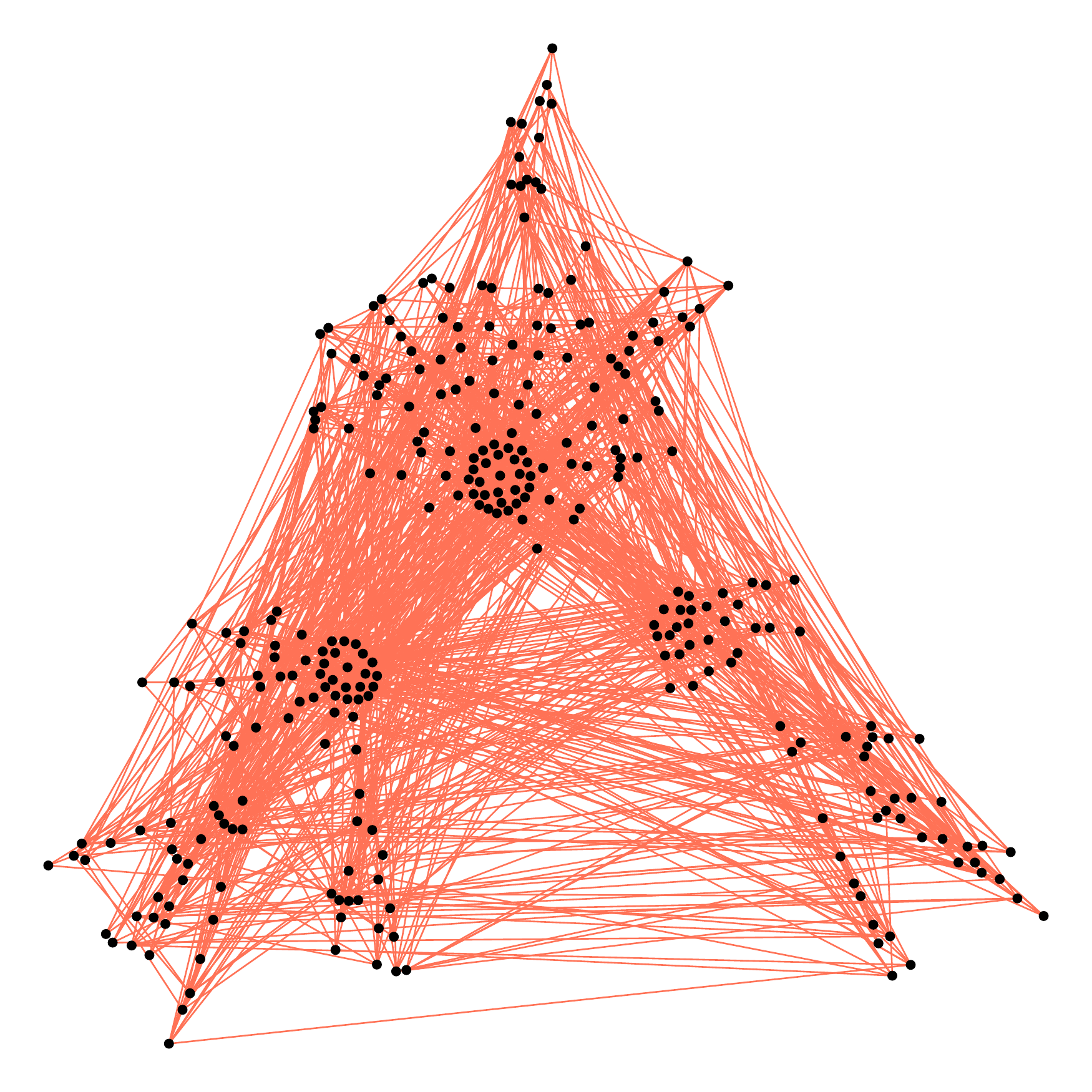}}
     \subfigure[]{  \includegraphics[width=0.45\columnwidth, height=0.37\columnwidth]{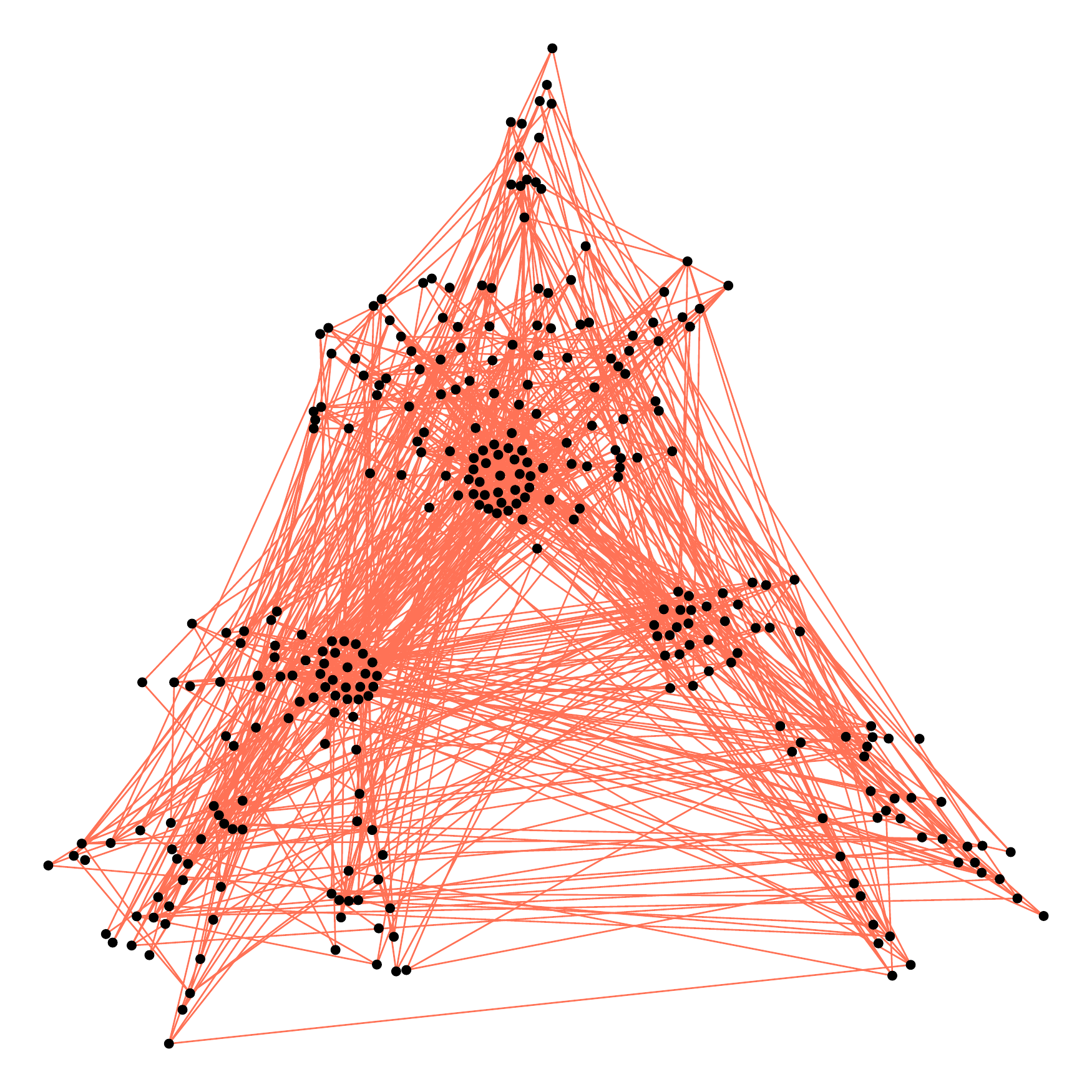}}}
 \centerline{
     \subfigure[]{  \includegraphics[width=0.45\columnwidth, height=0.37\columnwidth]{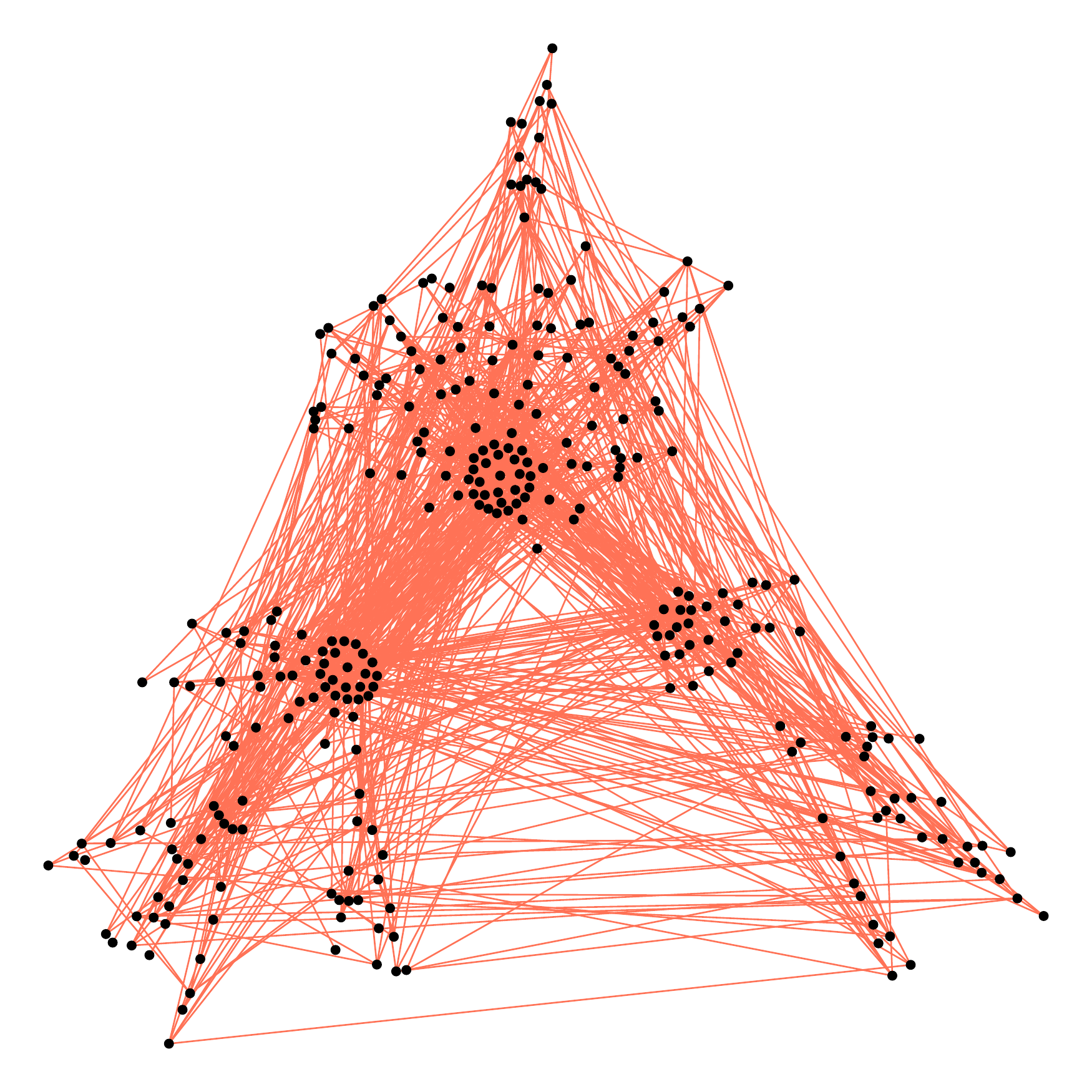}}
     \subfigure[]{  \includegraphics[width=0.45\columnwidth, height=0.37\columnwidth]{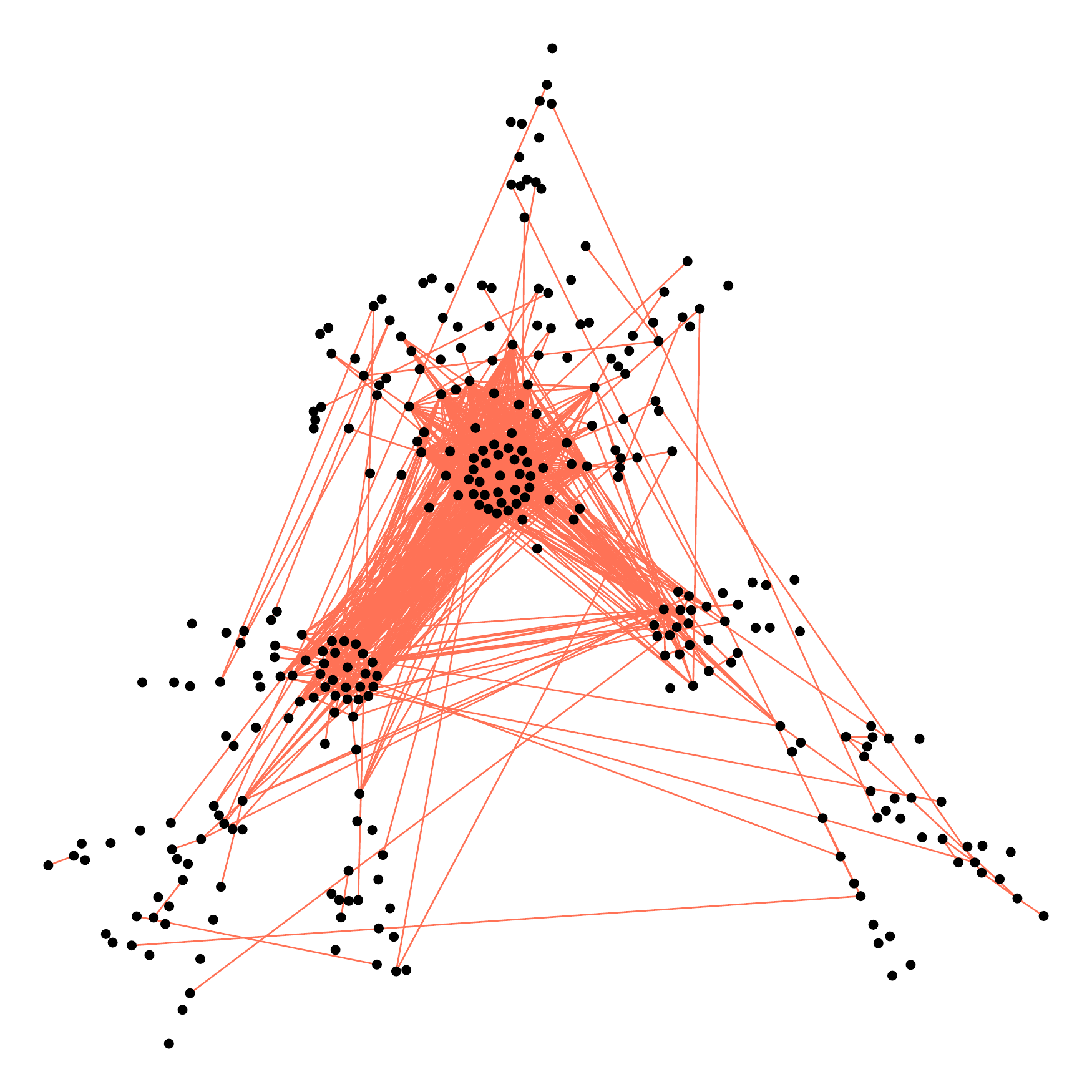}}}
     \caption{
     Examples of graph recovery for a scale-free graph of size $d=300$ and $n=200$ samples,
     where 
     (a) depicts the true graph, 
     (b) is the \oracle{}  ($F_1=0.25$), 
     (c) is the \thav{} ($F_1=0.23$), 
     (d) is the scaled lasso ($F_1=0.27$), 
     (e) is the TIGER ($F_1=0.25$), and 
     (f) is the \stars{}  ($F_1=0.19$).}
     \label{fig:exarecovery5}
 \end{figure*}
\begin{figure*}
 \centerline{
     \subfigure[]{  \includegraphics[width=0.45\columnwidth, height=0.37\columnwidth]{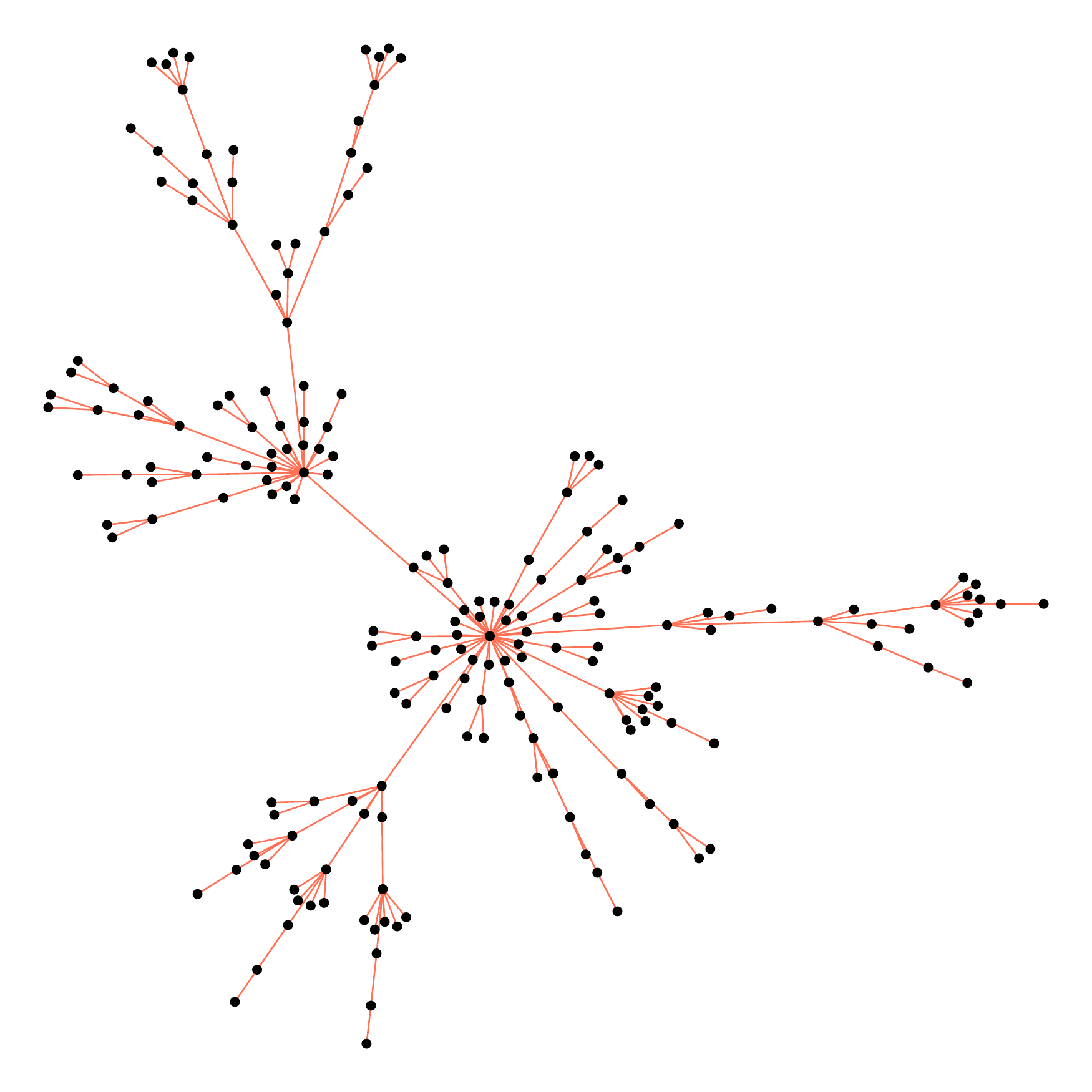}}
     \subfigure[]{  \includegraphics[width=0.45\columnwidth, height=0.37\columnwidth]{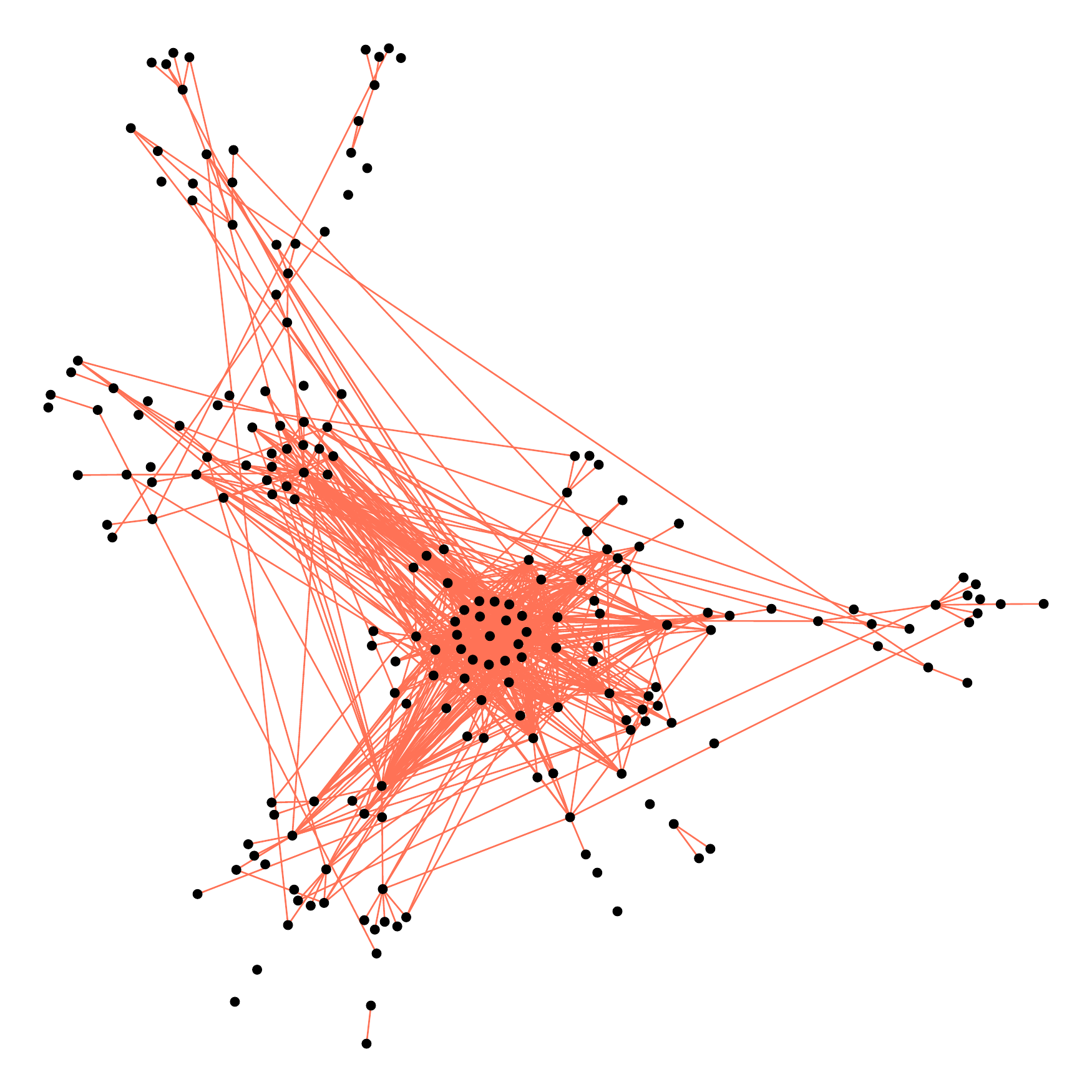} }}
 \centerline{    
     \subfigure[]{  \includegraphics[width=0.45\columnwidth, height=0.37\columnwidth]{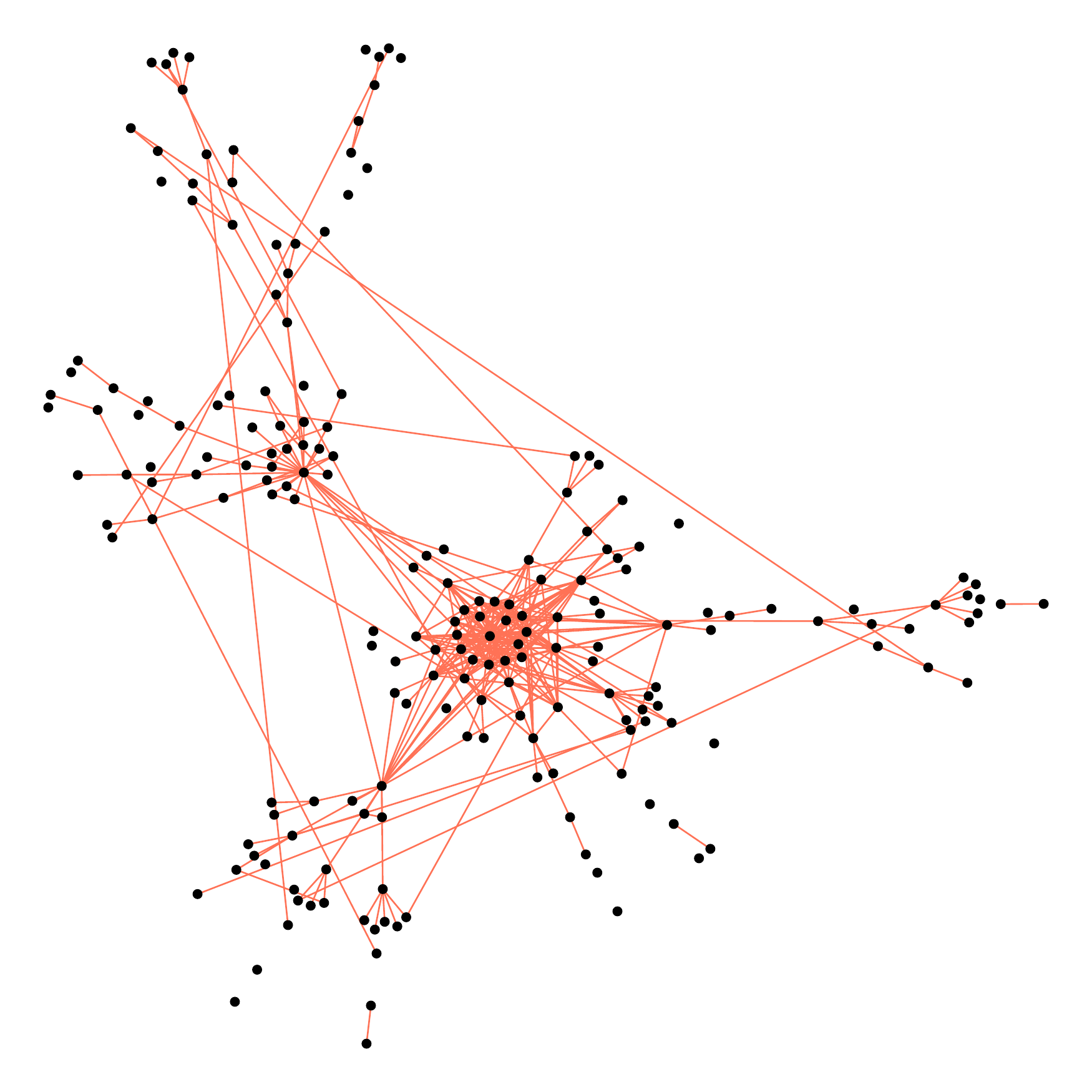}}
     \subfigure[]{  \includegraphics[width=0.45\columnwidth, height=0.37\columnwidth]{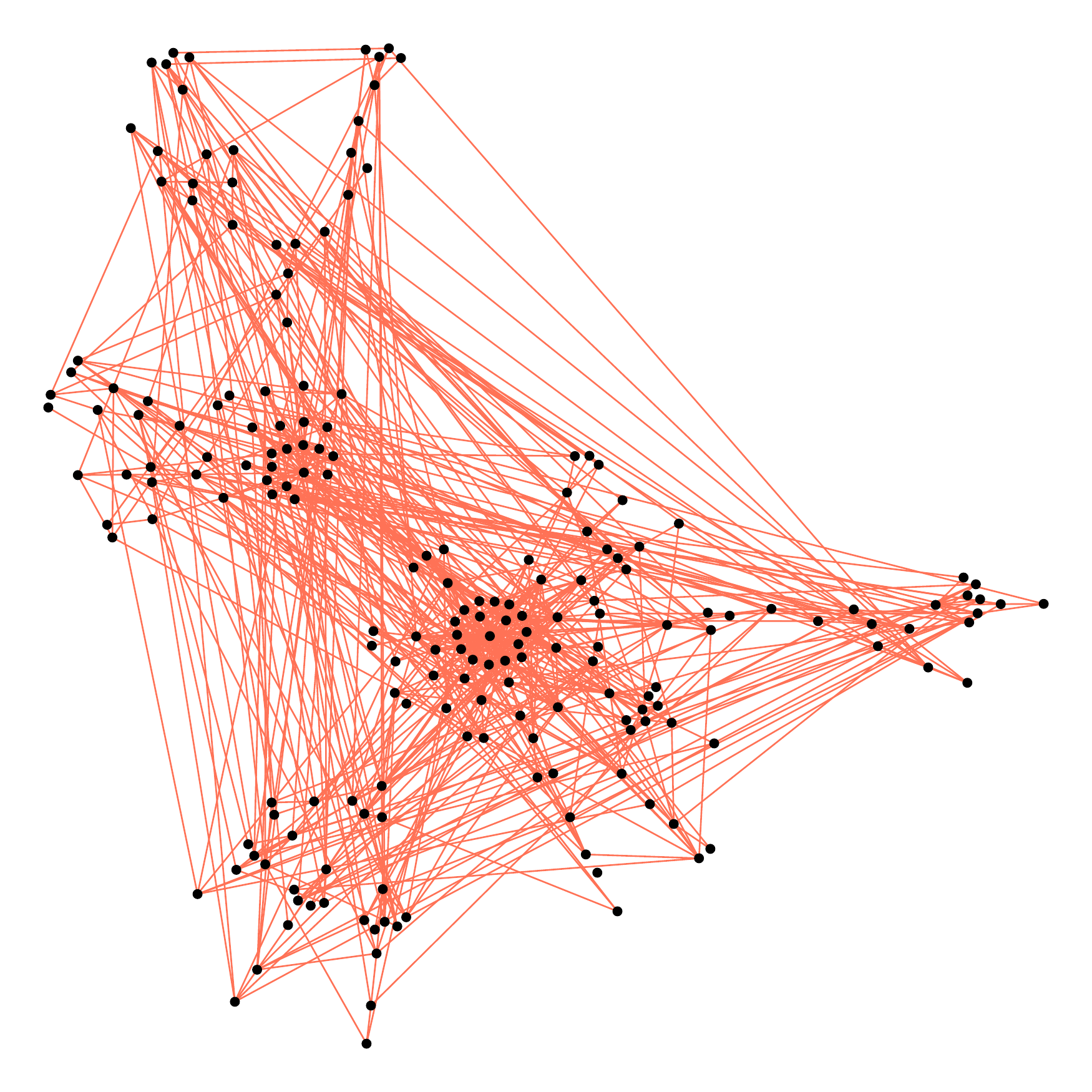}}}
 \centerline{
     \subfigure[]{  \includegraphics[width=0.45\columnwidth, height=0.37\columnwidth]{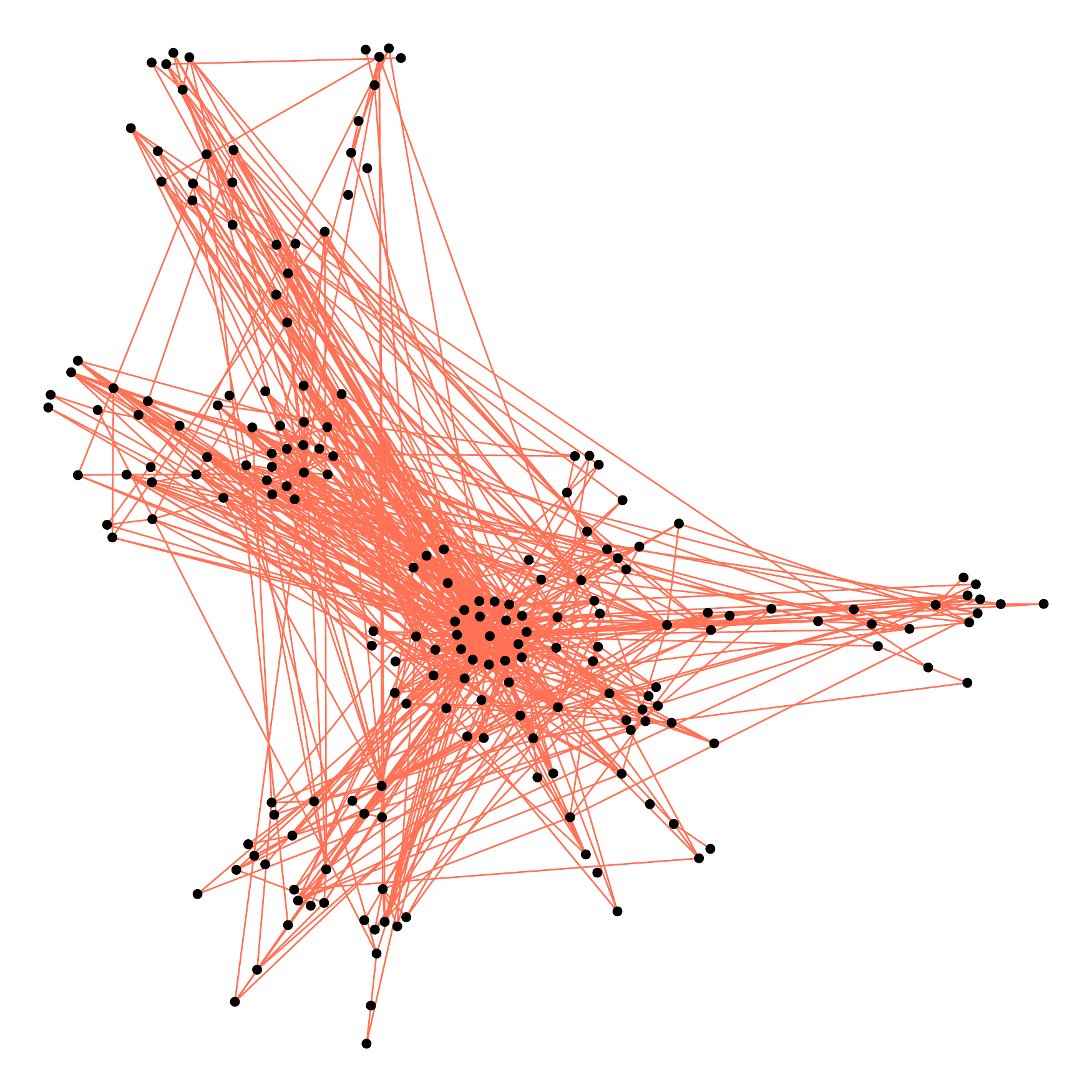}}
     \subfigure[]{  \includegraphics[width=0.45\columnwidth, height=0.37\columnwidth]{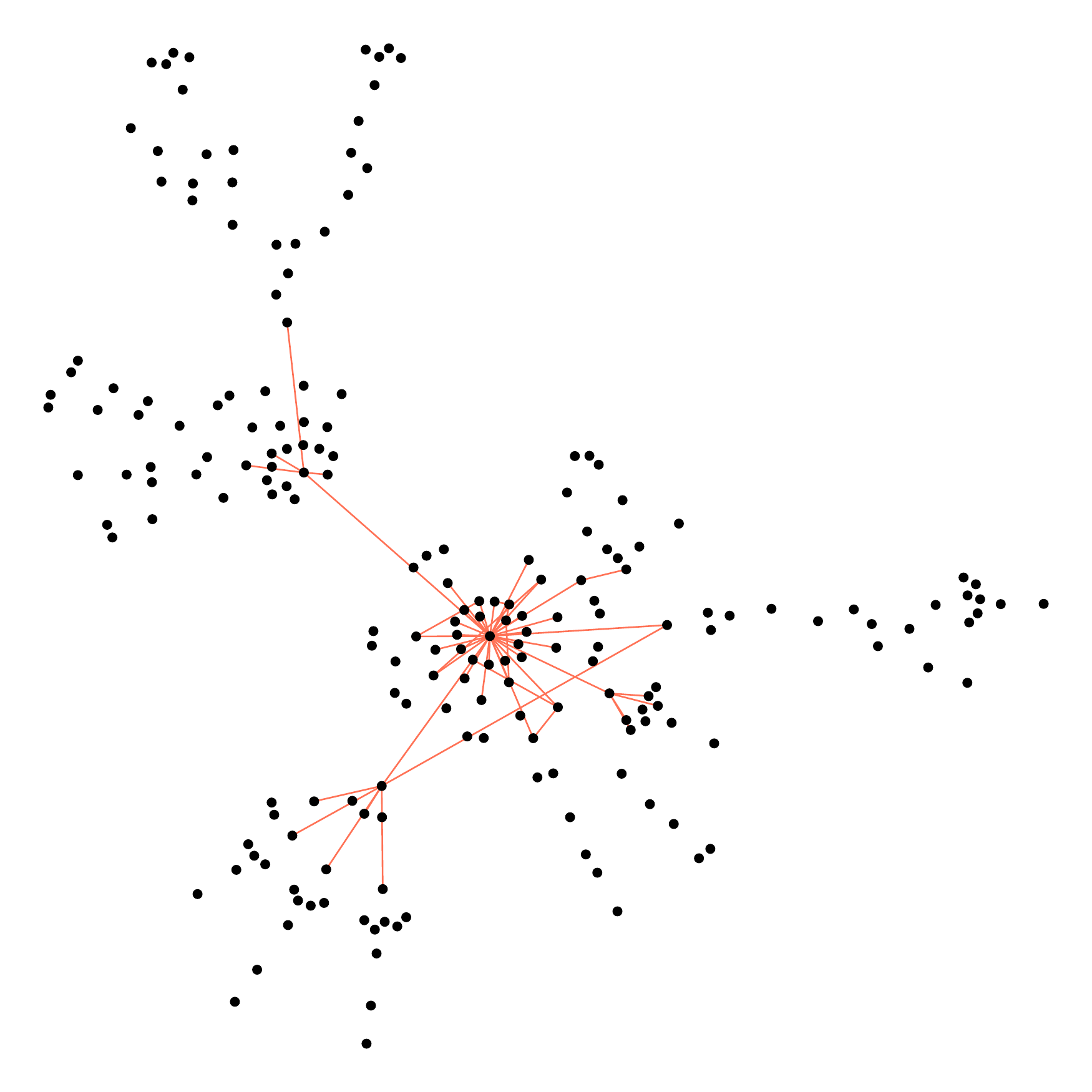}}}
     \caption{
     Examples of graph recovery for a scale-free graph of size $d=200$ and $n=400$ samples,
     where 
     (a) depicts the true graph, 
     (b) is the \oracle{}  ($F_1=0.34$), 
     (c) is the \thav{} ($F_1=0.60$), 
     (d) is the scaled lasso ($F_1=0.43$), 
     (e) is the rSME ($F_1=0.42$), and 
     (f) is the SCIO (Bregman) ($F_1=0.38$).}
     \label{fig:exarecovery6}
 \end{figure*}
 
 \begin{figure*}
     \centerline{
     \subfigure[]{
     \includegraphics[width=0.45\columnwidth]{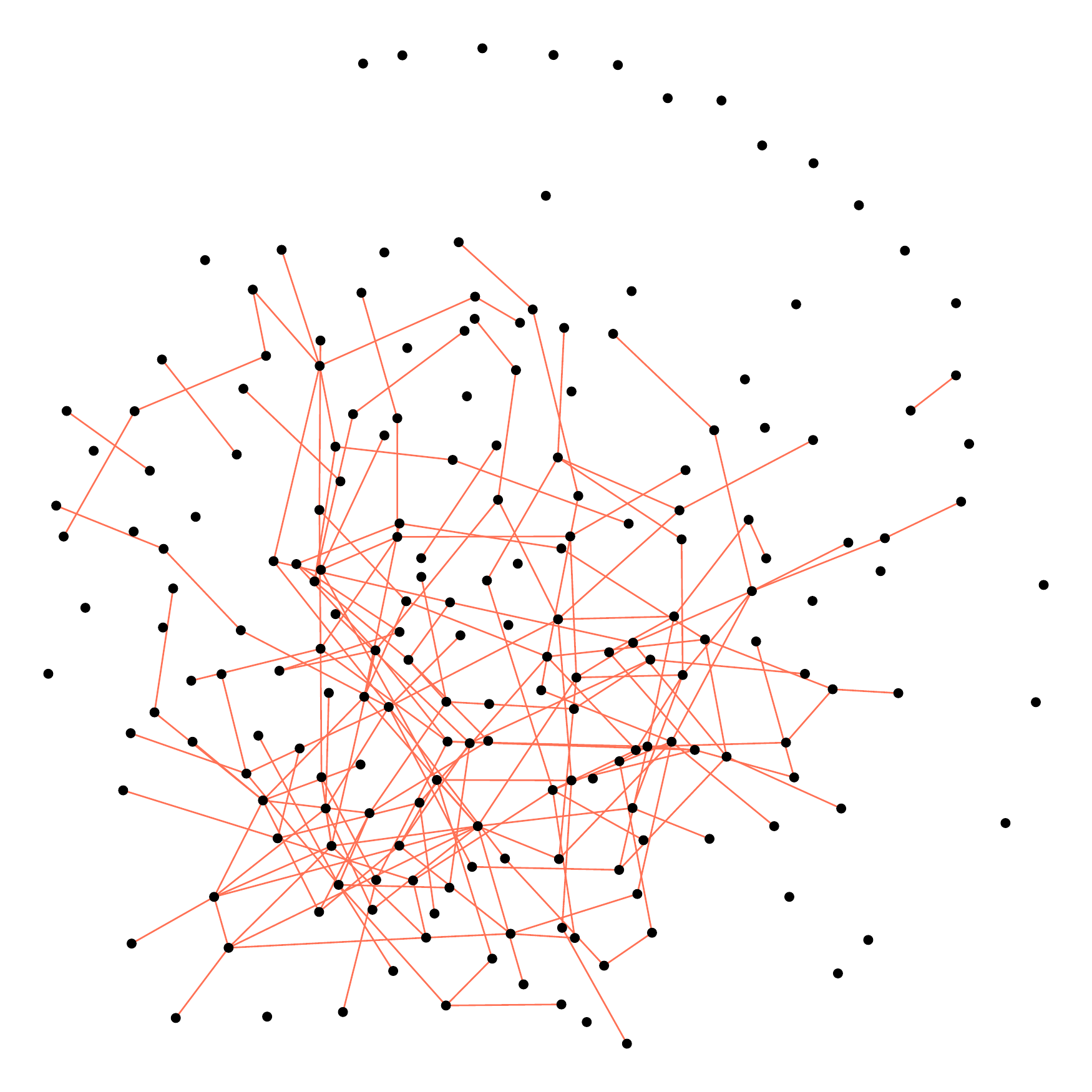}}
     \subfigure[]{
     \includegraphics[width=0.45\columnwidth]{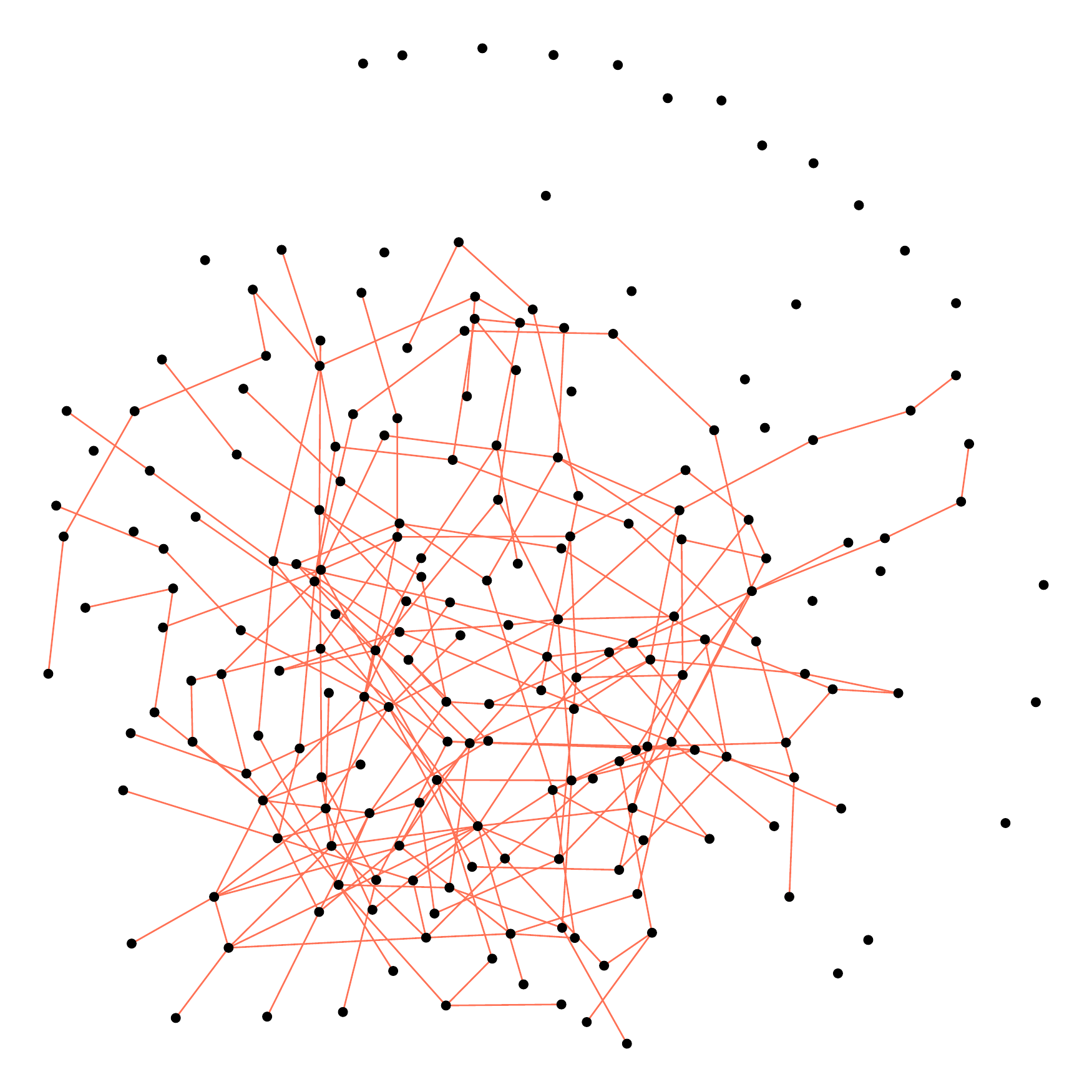}}
     }
     \centerline{
     \subfigure[]{
     \includegraphics[width=0.45\columnwidth]{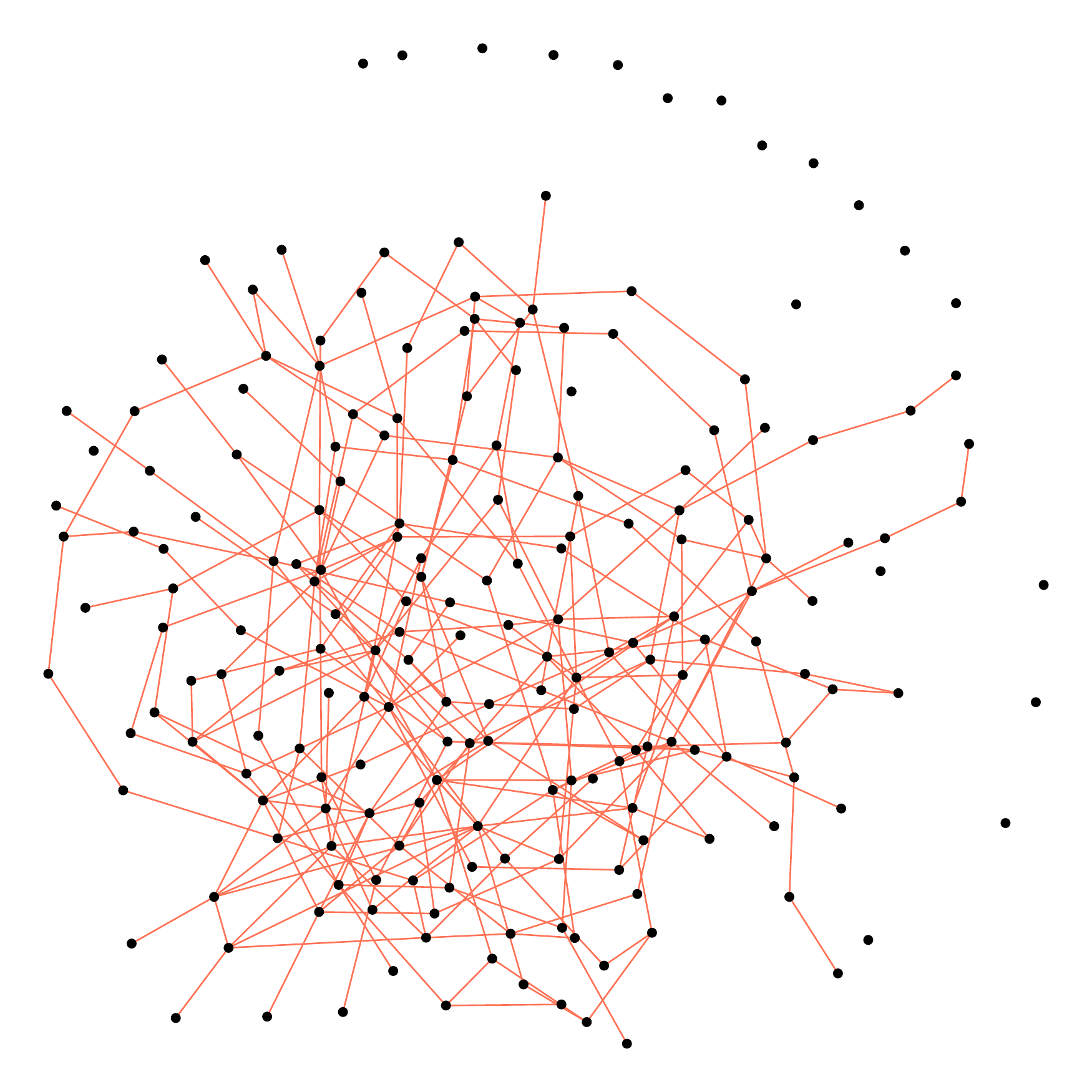}}
     \subfigure[]{
     \includegraphics[width=0.45\columnwidth]{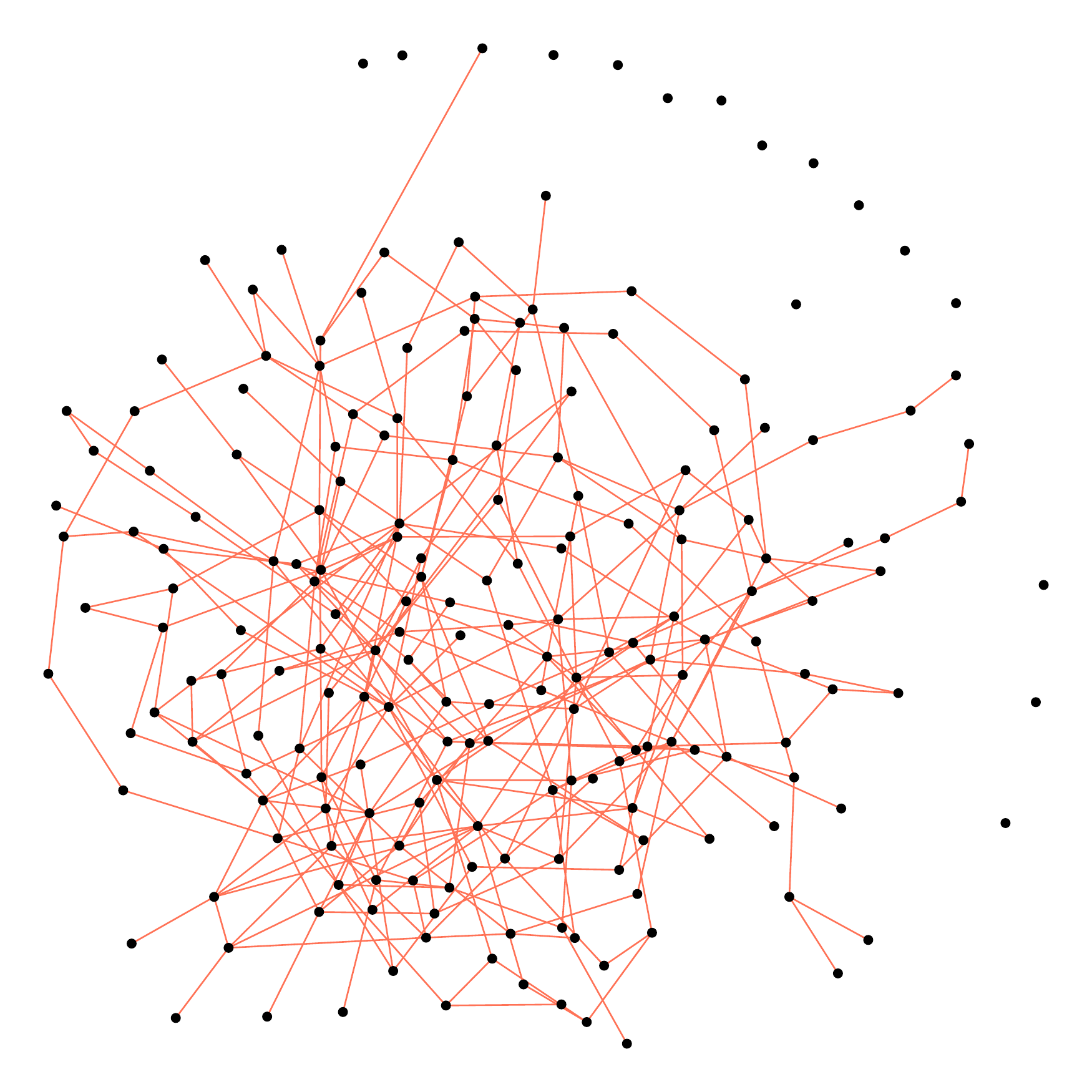}}
     }
     \caption{Recovered graphs by the \thav{} estimator using different values of $C$ for a random graph of size $d=200$ and $n=300$ samples. The graphs in (a) to (d) depict the resulting graphs using $C\in\{ 0.5,\; 0.6, \; 0.7, \, 0.8\}$, respectively. The corresponding $F_1$-scores are $0.78,\; 0.88, \; 0.94, \; 0.95$, respectively.}
     \label{fig:exadifferentC}
 \end{figure*}

 \begin{figure*}
     \centerline{
     \subfigure[]{
     \includegraphics[width=0.45\columnwidth]{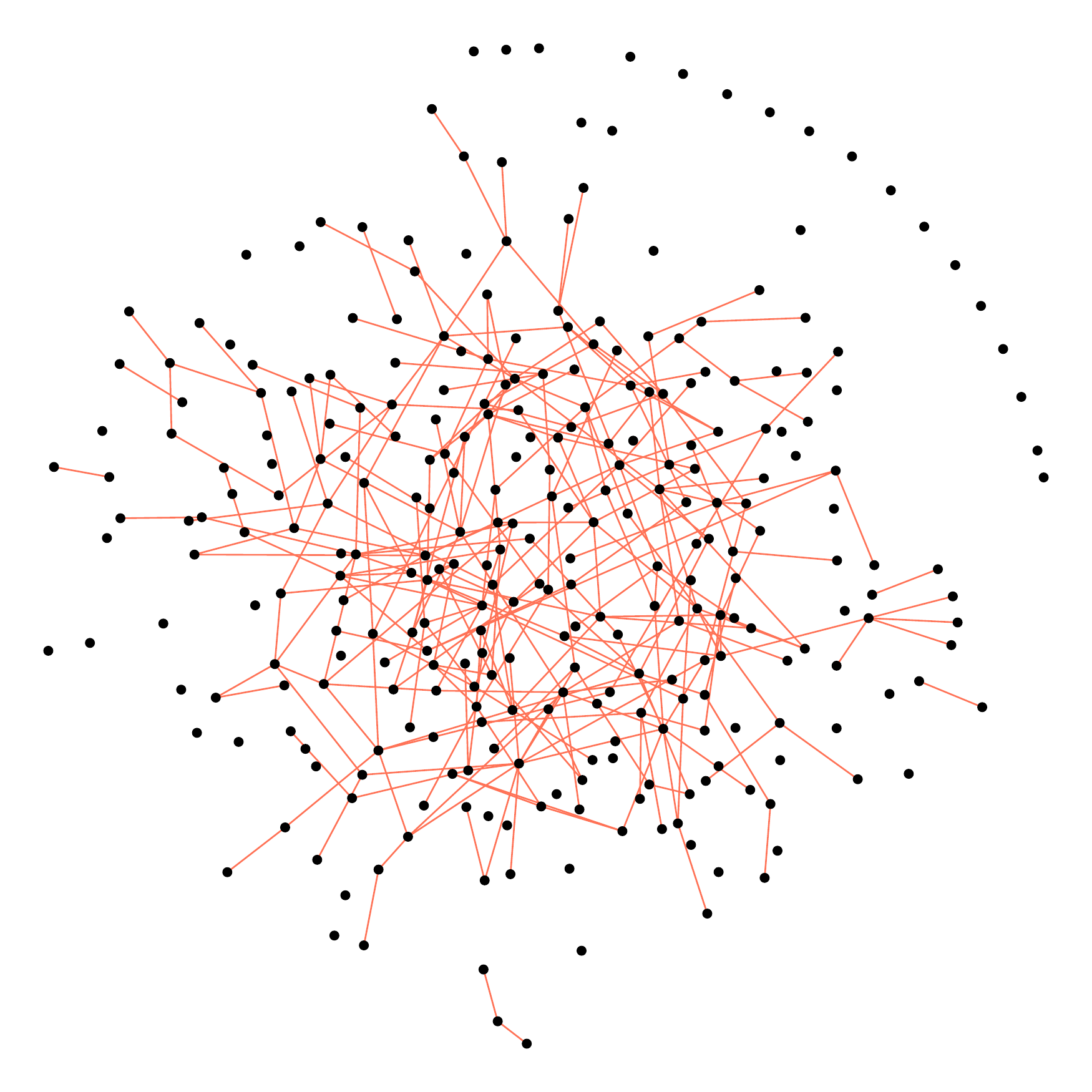}}
     \subfigure[]{
     \includegraphics[width=0.45\columnwidth]{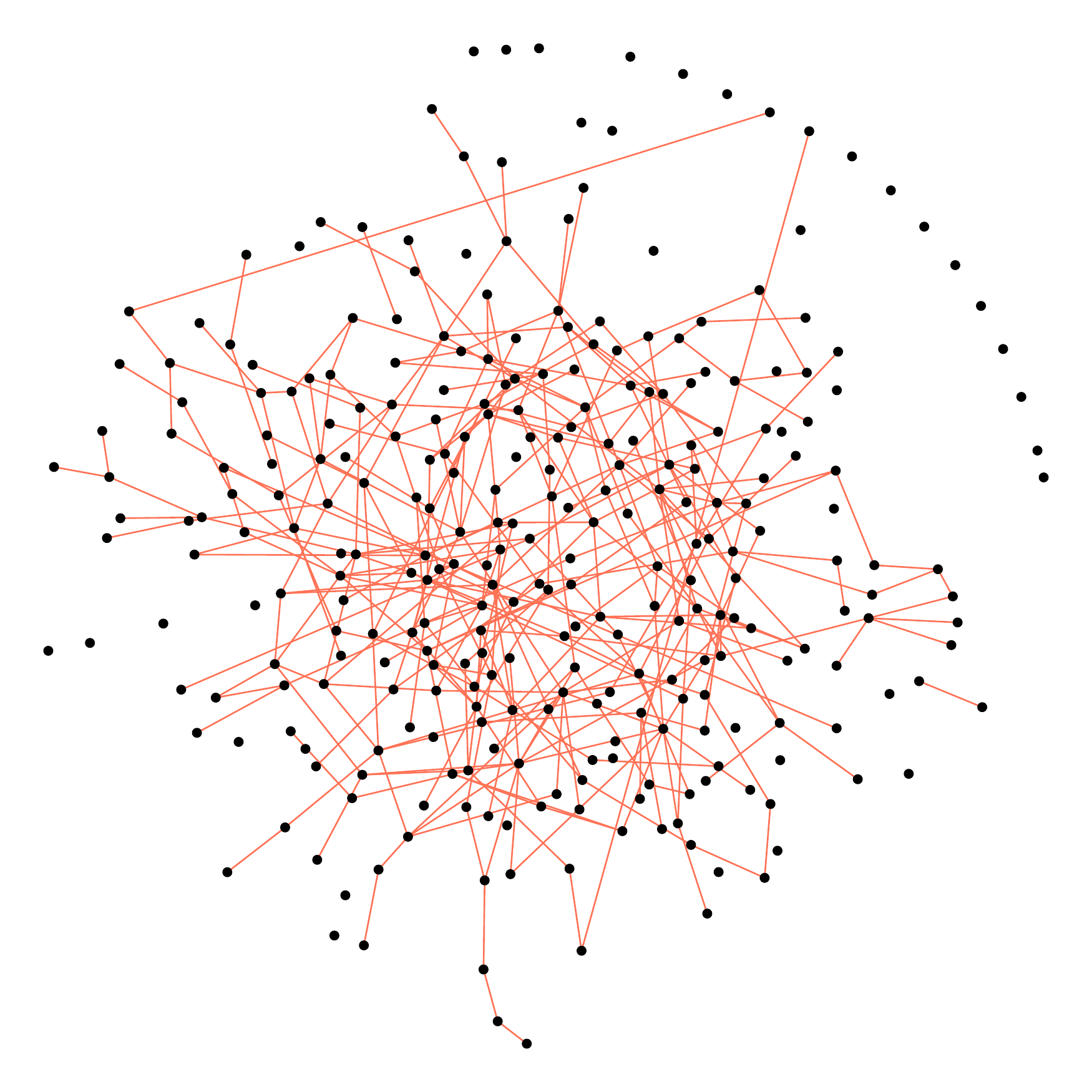}}
     }
     \centerline{
     \subfigure[]{
     \includegraphics[width=0.45\columnwidth]{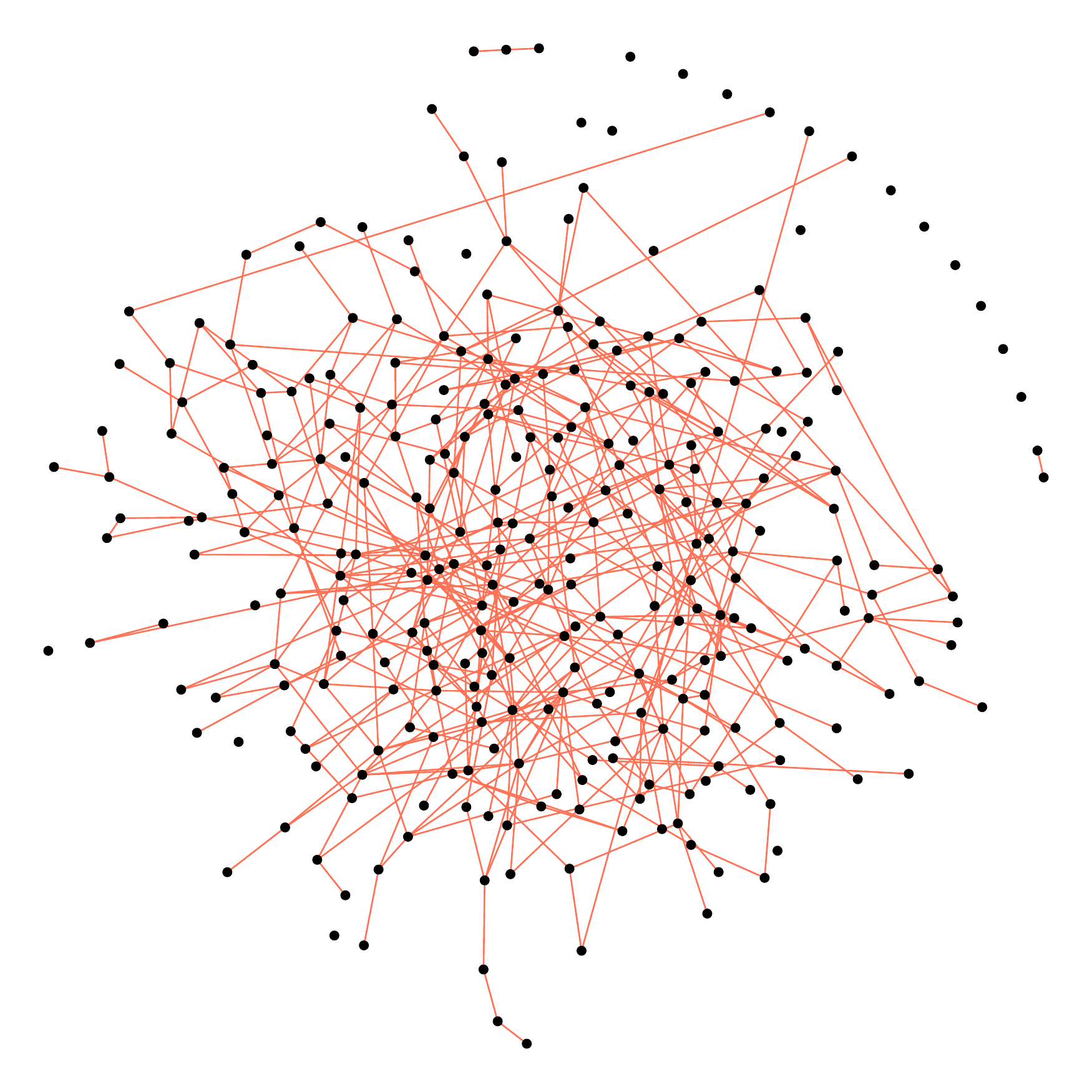}}
     \subfigure[]{
     \includegraphics[width=0.45\columnwidth]{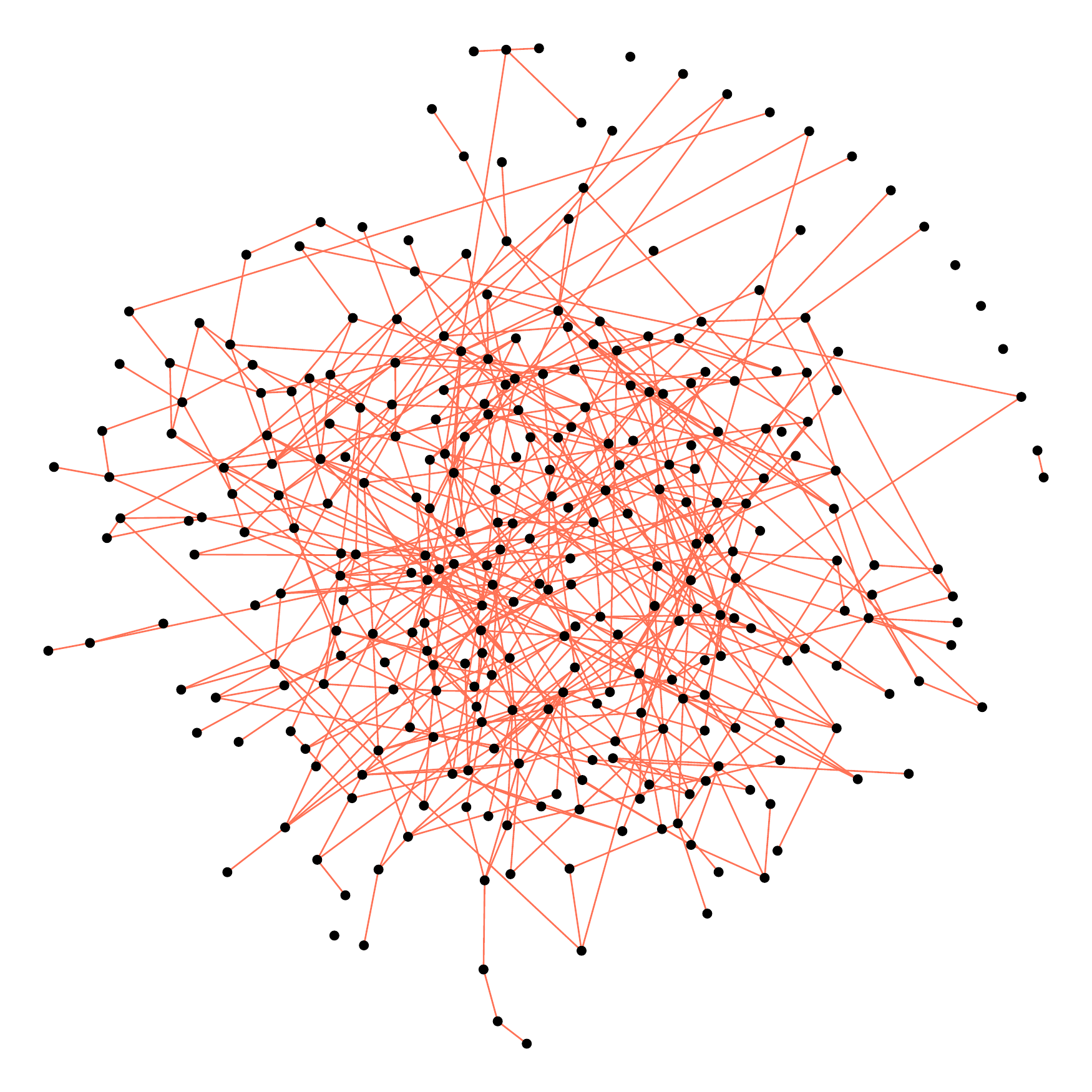}}
     }
     \caption{Recovered graphs by the \thav{} estimator using different values of $C$ for a random graph of size $d=300$ and $n=200$ samples. The graphs in (a) to (d) depict the resulting graphs using $C \in \{ 0.5,\; 0.6, \; 0.7, \, 0.8\} $, respectively. The corresponding $F_1$-scores are $0.74,\; 0.82, \; 0.86, \; 0.84$, respectively.}
 \end{figure*}

 \begin{figure*}
     \centerline{
     \subfigure[]{
     \includegraphics[width=0.45\columnwidth]{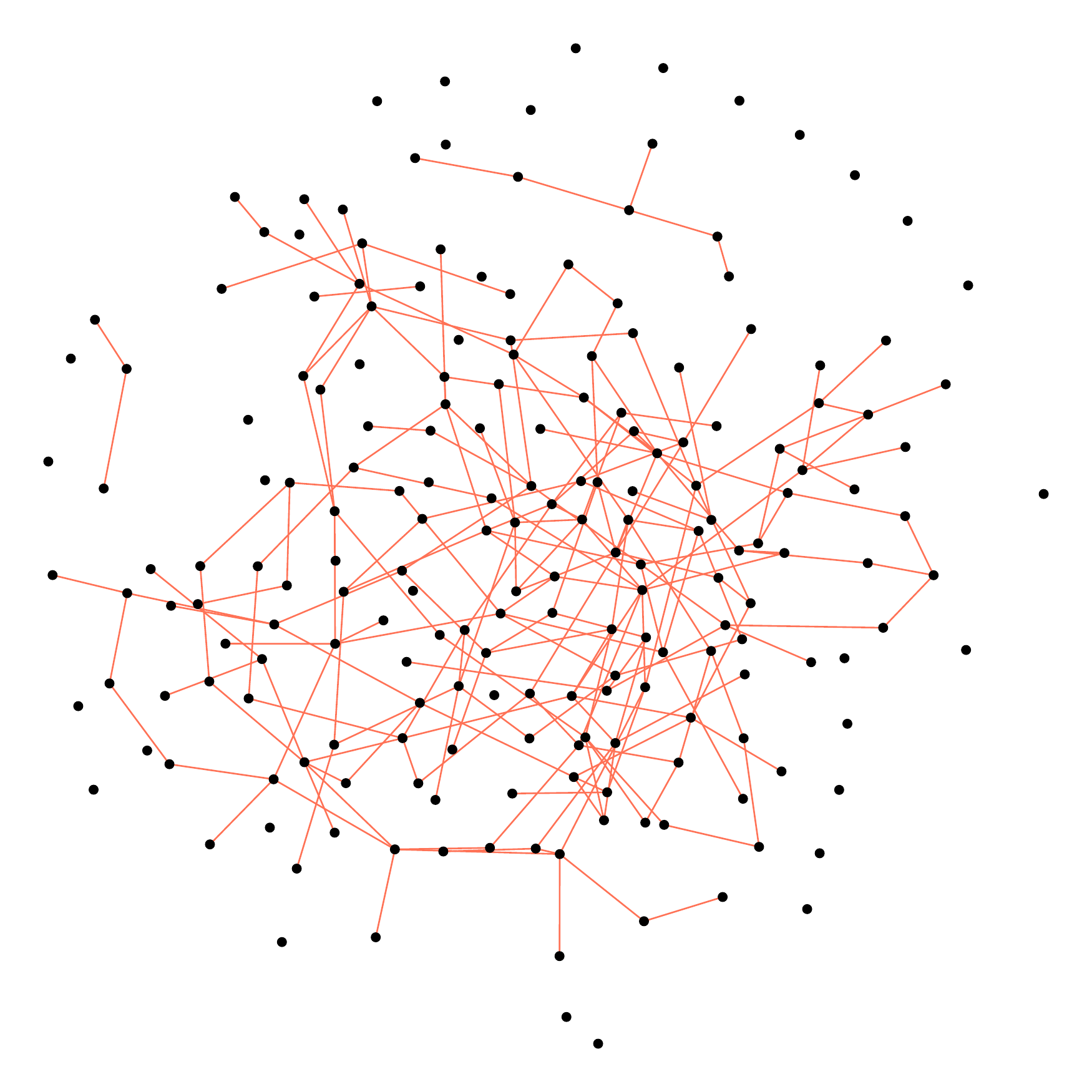}}
     \subfigure[]{
     \includegraphics[width=0.45\columnwidth]{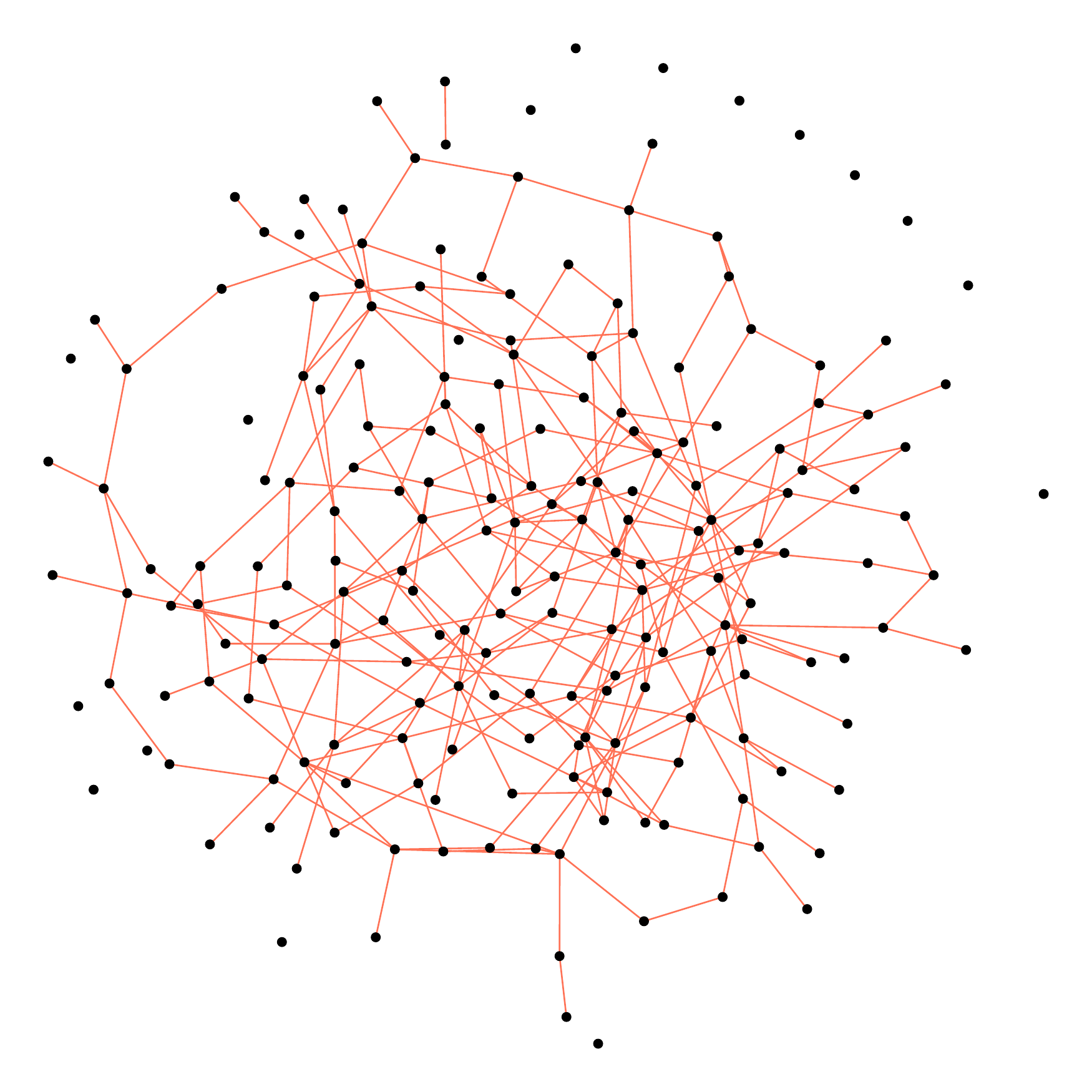}}
     }
     \centerline{
     \subfigure[]{
     \includegraphics[width=0.45\columnwidth]{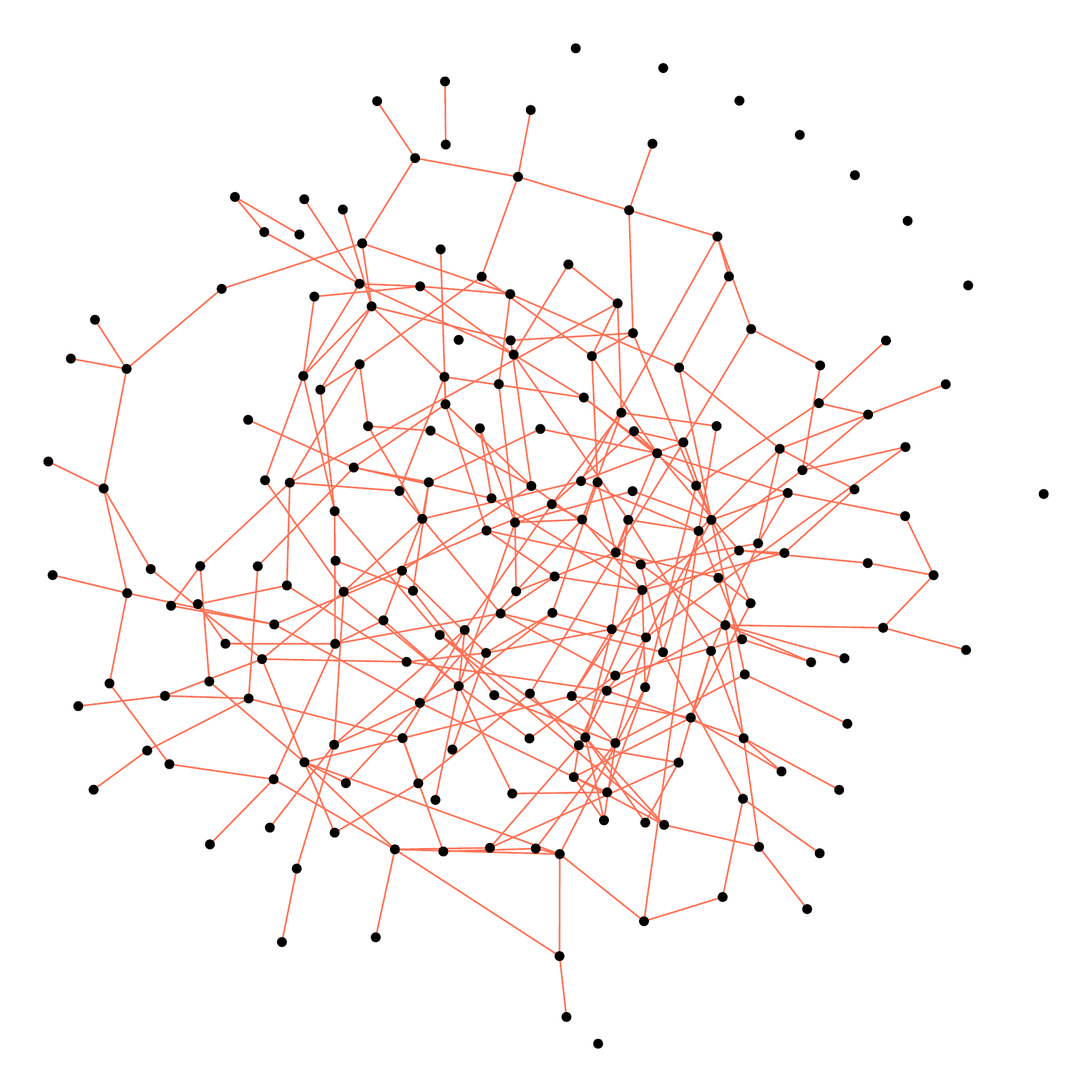}}
     \subfigure[]{
     \includegraphics[width=0.45\columnwidth]{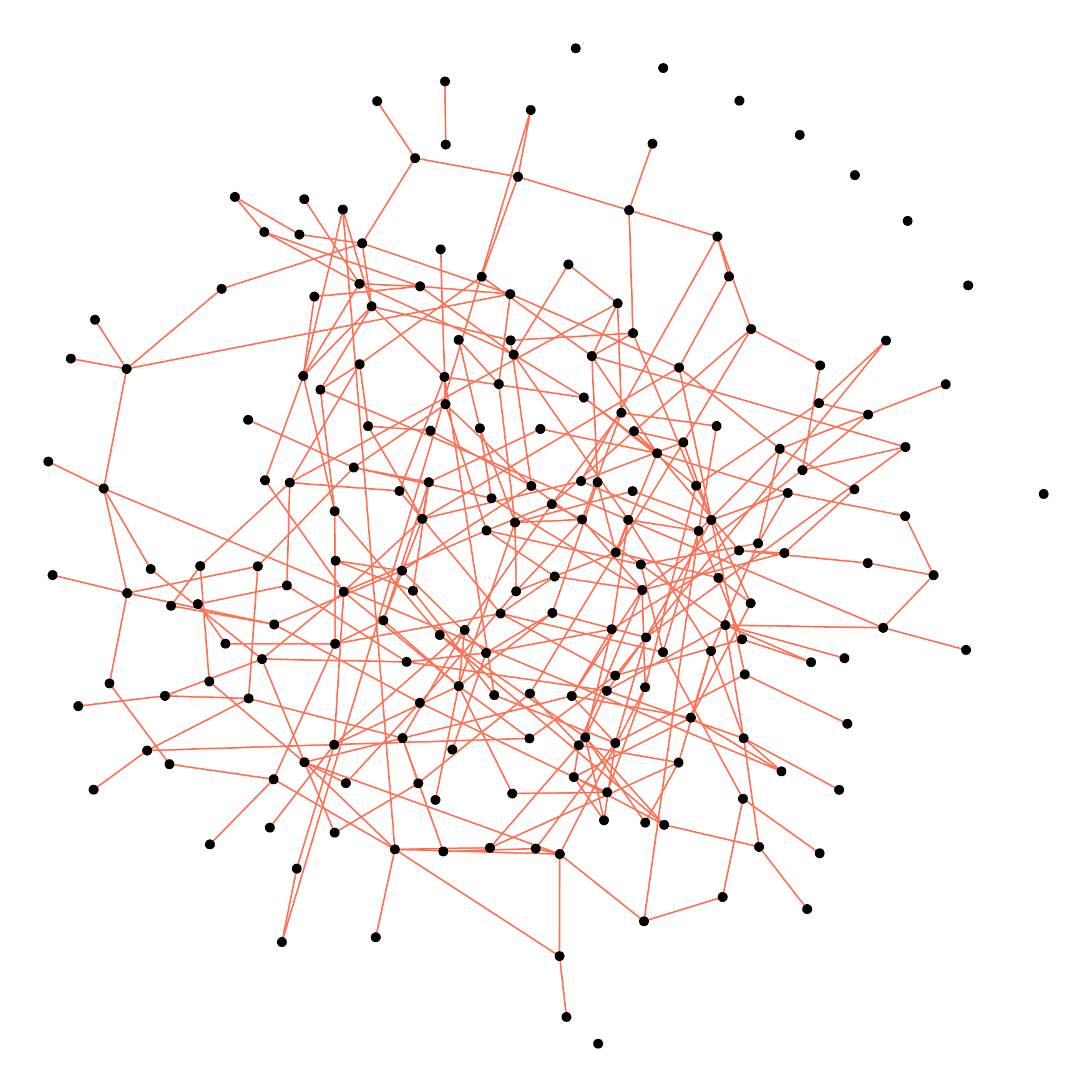}}
     }
     \caption{Recovered graphs by the \thav{} estimator using different values of $C$ for a random graph of size $d=200$ and $n=400$ samples. The graphs in (a) to (d) depict the resulting graphs using $C\in\{0.5,\; 0.6, \; 0.7, \, 0.8\}$, respectively. The corresponding $F_1$-scores are $0.82,\; 0.94, \; 0.97, \; 0.96$, respectively.}
 \end{figure*}

 \begin{figure*}
     \centerline{
     \subfigure[]{
     \includegraphics[width=0.45\columnwidth]{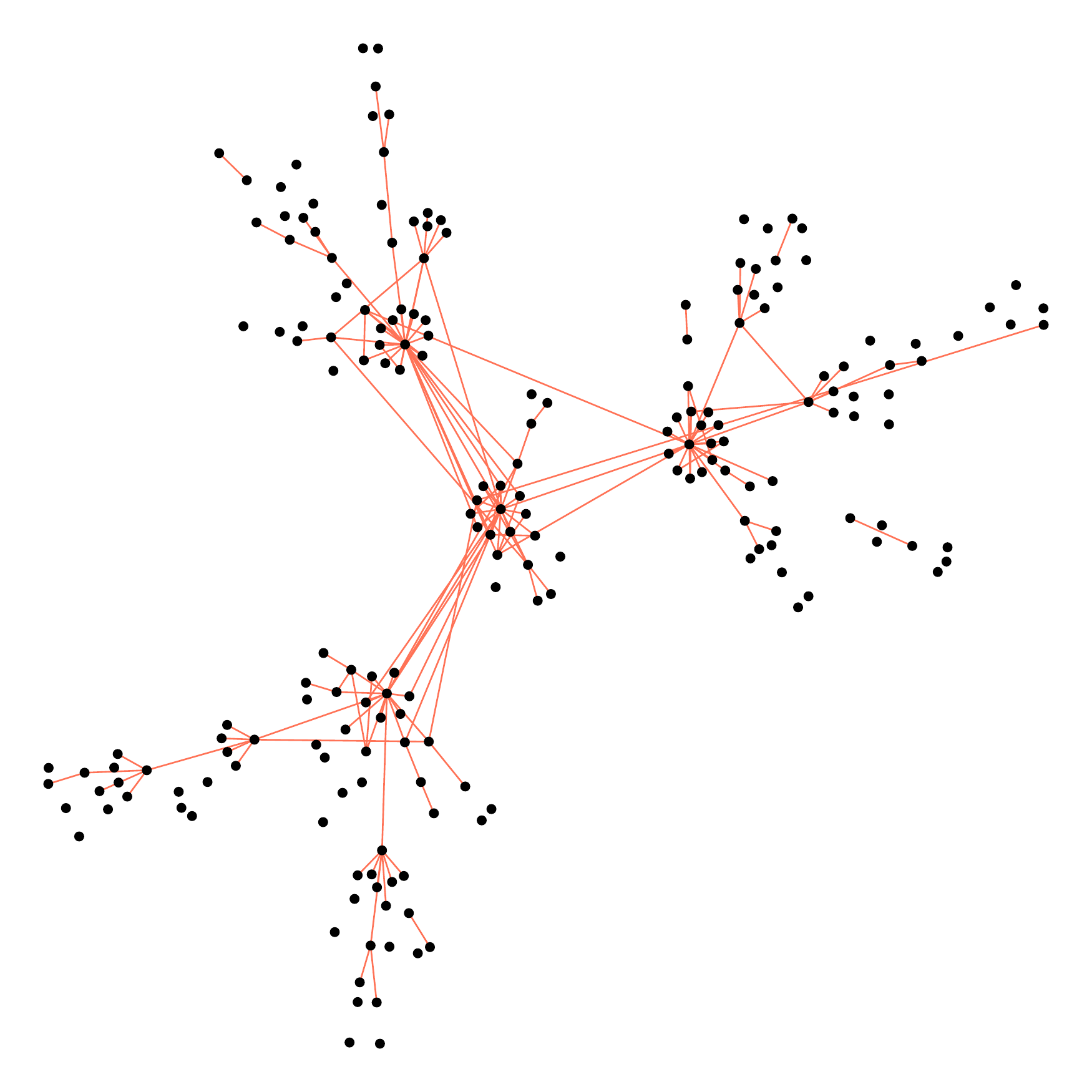}}
     \subfigure[]{
     \includegraphics[width=0.45\columnwidth]{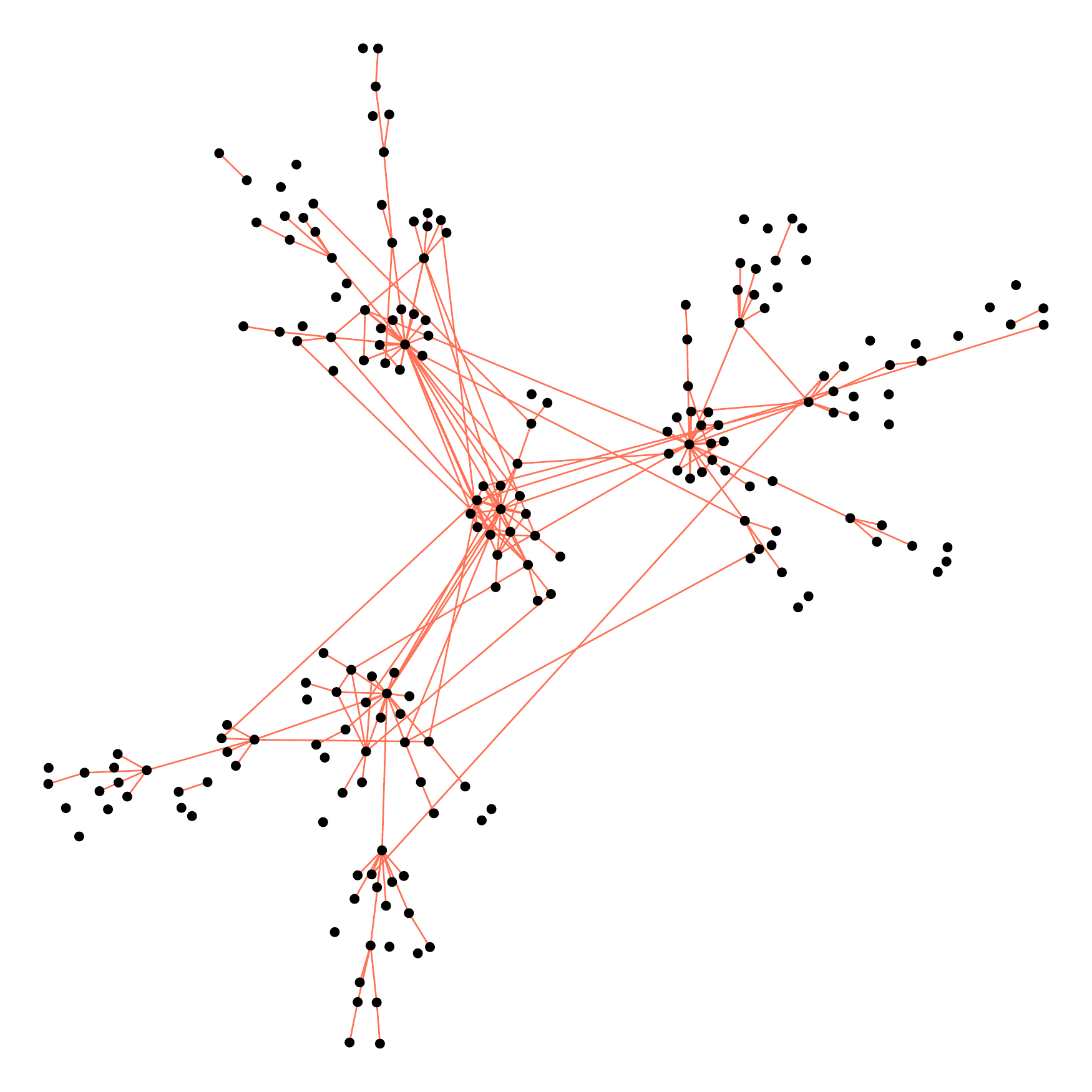}}
     }
     \centerline{
     \subfigure[]{
     \includegraphics[width=0.45\columnwidth]{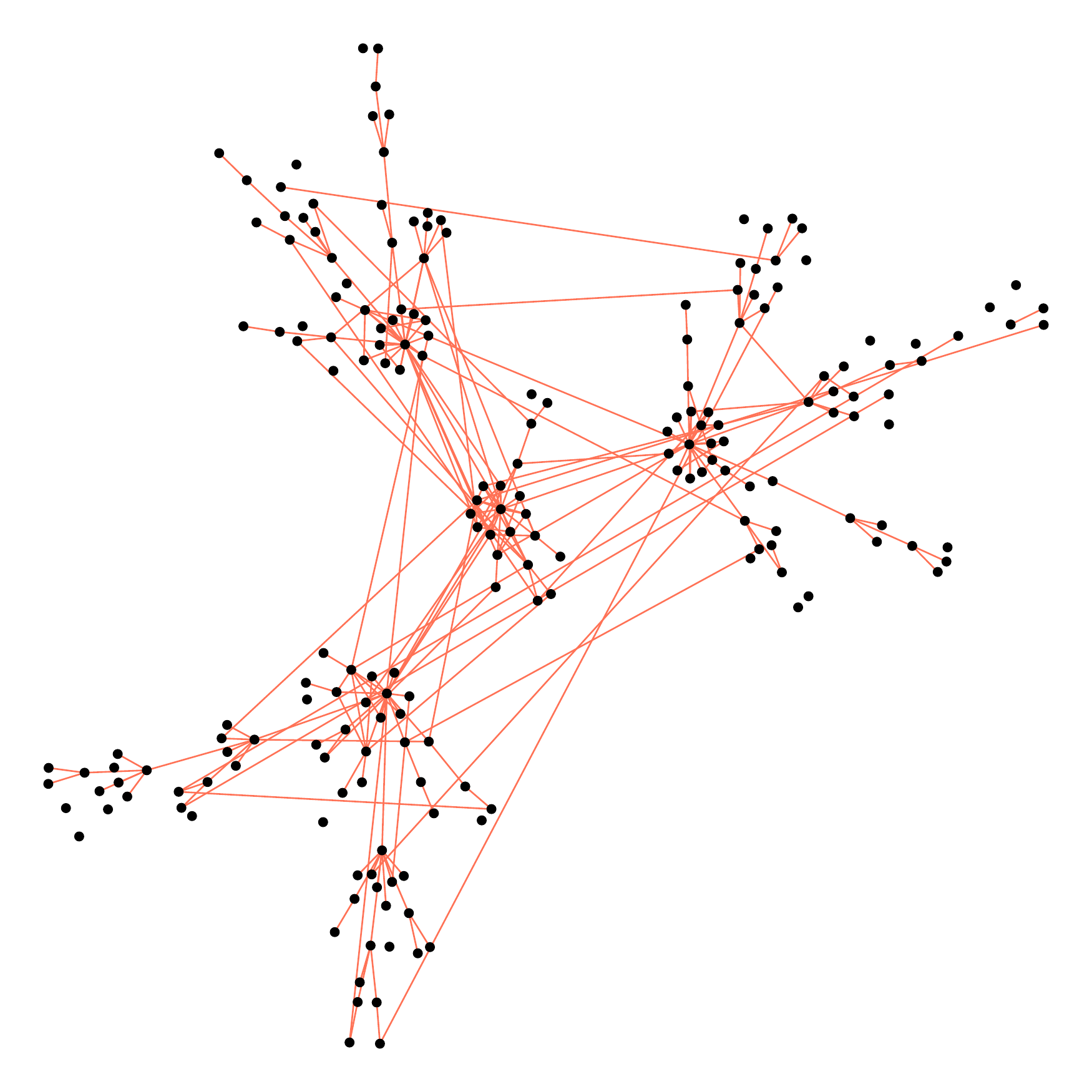}}
     \subfigure[]{
     \includegraphics[width=0.45\columnwidth]{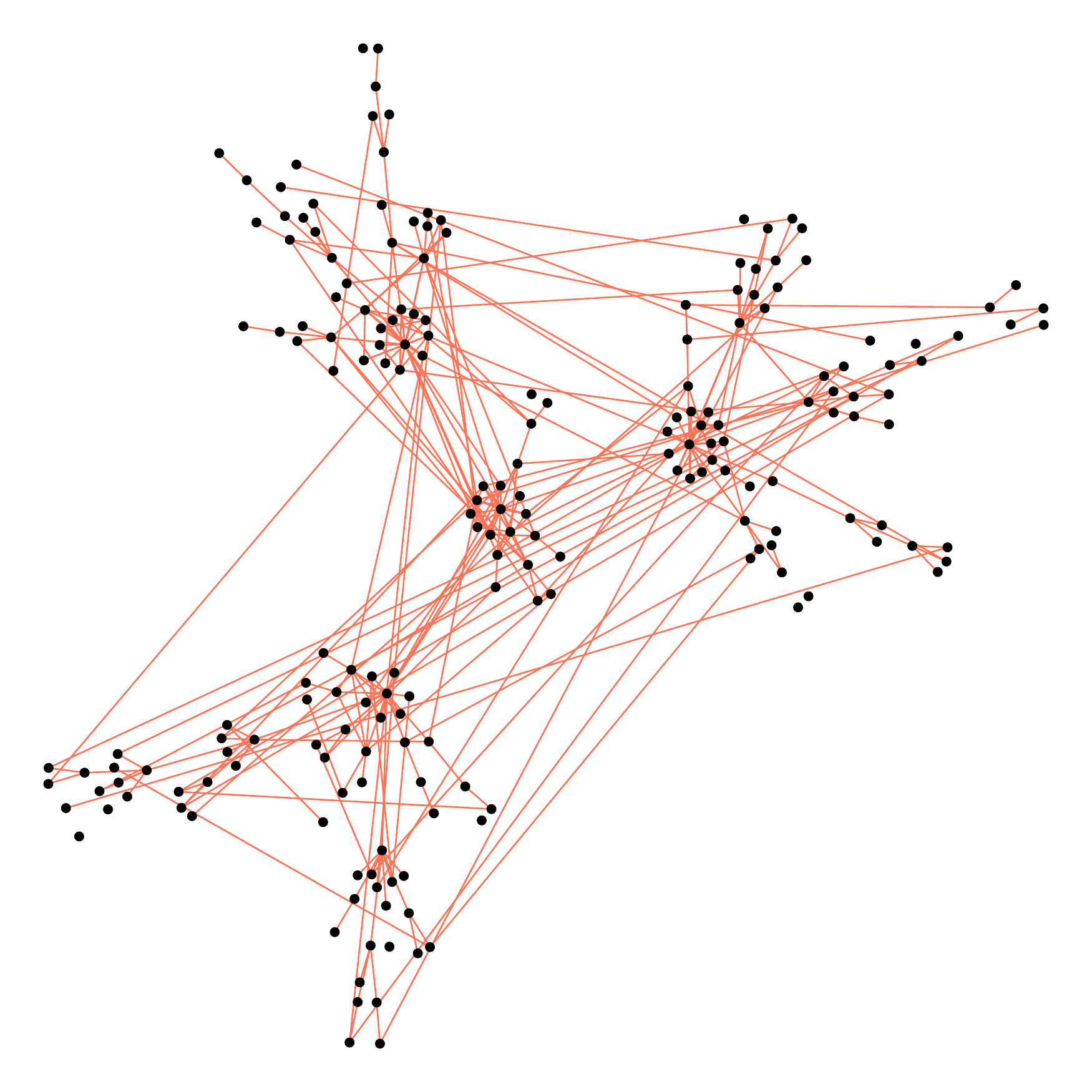}}
     }
     \caption{Recovered graphs by the \thav{} estimator using different values of $C$ for a scale-free graph of size $d=200$ and $n=300$ samples. The graphs in (a) to (d) depict the resulting graphs using $C\in\{0.5,\; 0.6, \; 0.7, \, 0.8\}$, respectively. The corresponding $F_1$-scores are $0.69,\; 0.73, \; 0.75, \; 0.71$, respectively.}
 \end{figure*}

 \begin{figure*}
     \centerline{
     \subfigure[]{
     \includegraphics[width=0.45\columnwidth]{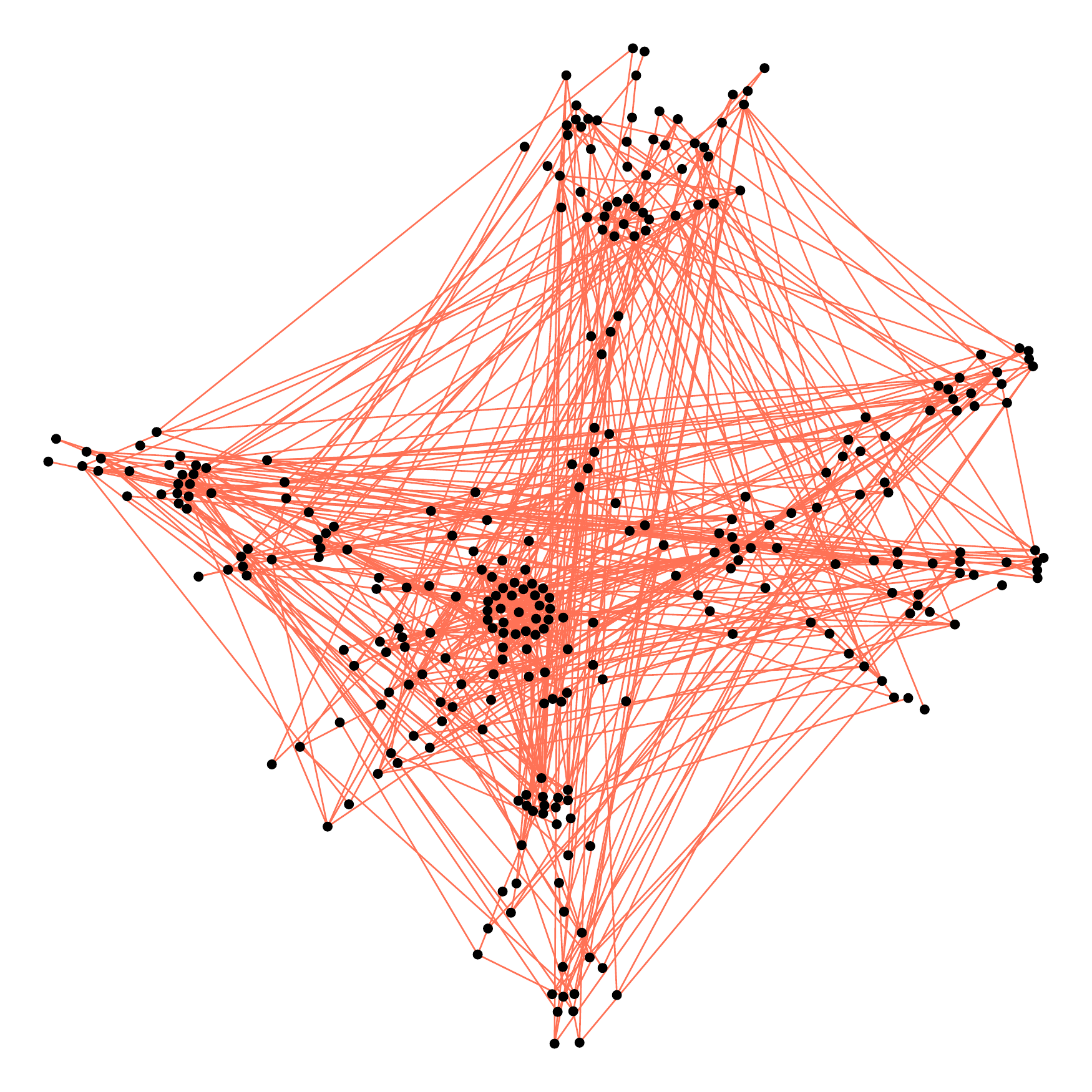}}
     \subfigure[]{
     \includegraphics[width=0.45\columnwidth]{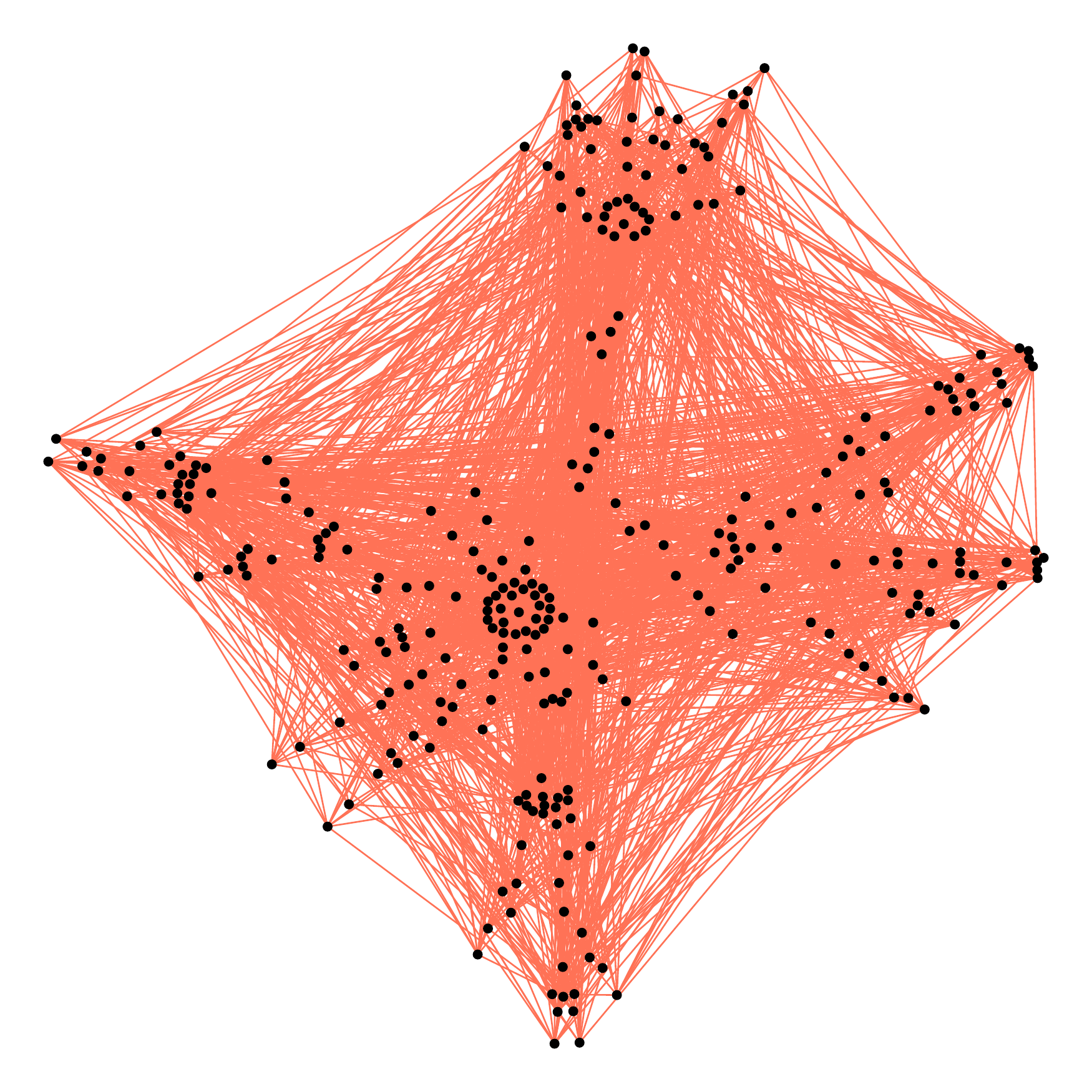}}
     }
     \centerline{
     \subfigure[]{
     \includegraphics[width=0.45\columnwidth]{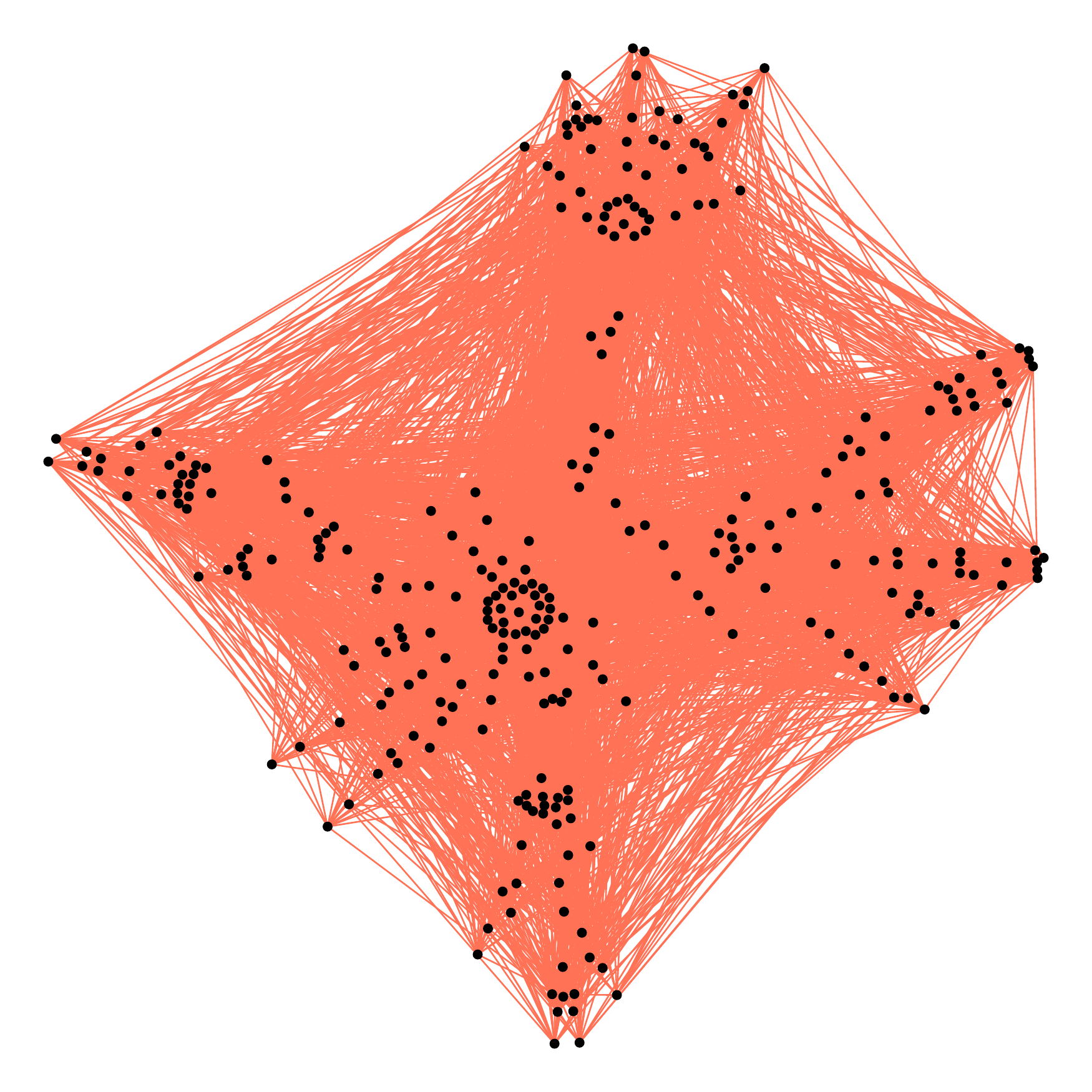}}
     \subfigure[]{
     \includegraphics[width=0.45\columnwidth]{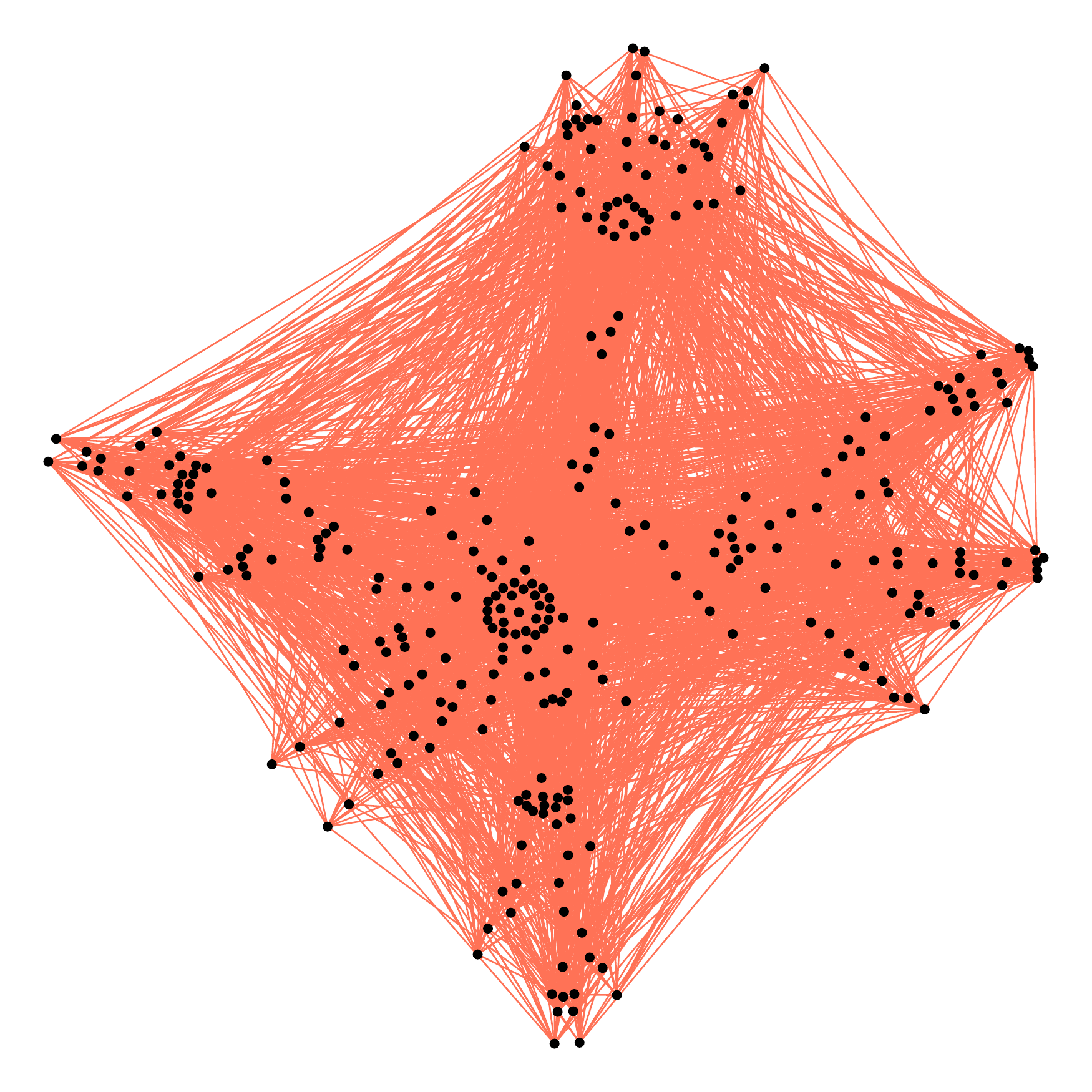}}
     }
     \caption{Recovered graphs by the \thav{} estimator using different values of $C$ for a scale-free graph of size $d=300$ and $n=200$ samples. The graphs in (a) to (d) depict the resulting graphs using $C\in\{0.5,\; 0.6, \; 0.7, \, 0.8\}$, respectively. The corresponding $F_1$-scores are $0.30,\; 0.17, \; 0.11, \; 0.13$, respectively.}
 \end{figure*}

 \begin{figure*}
     \centerline{
     \subfigure[]{
     \includegraphics[width=0.45\columnwidth]{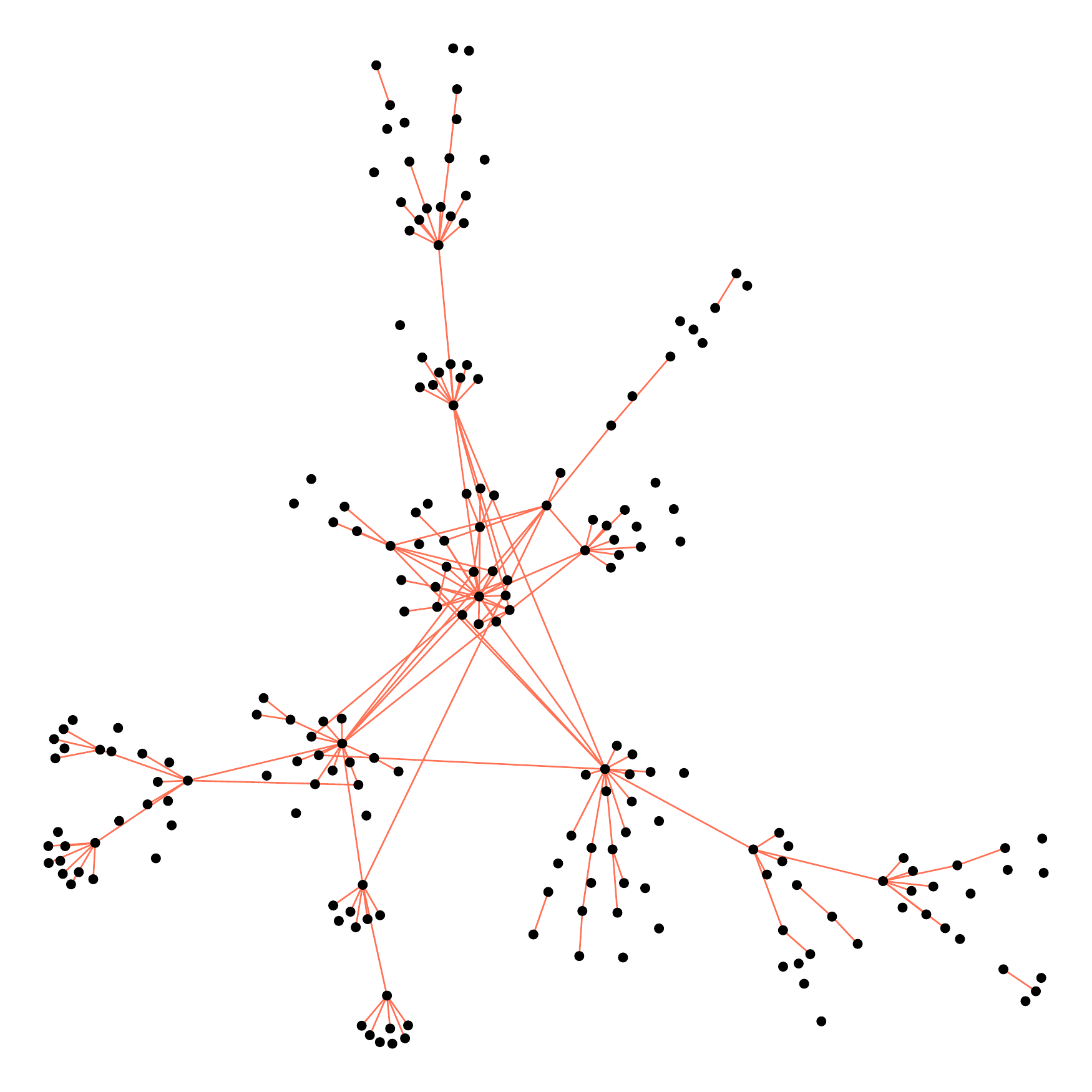}}
     \subfigure[]{
     \includegraphics[width=0.45\columnwidth]{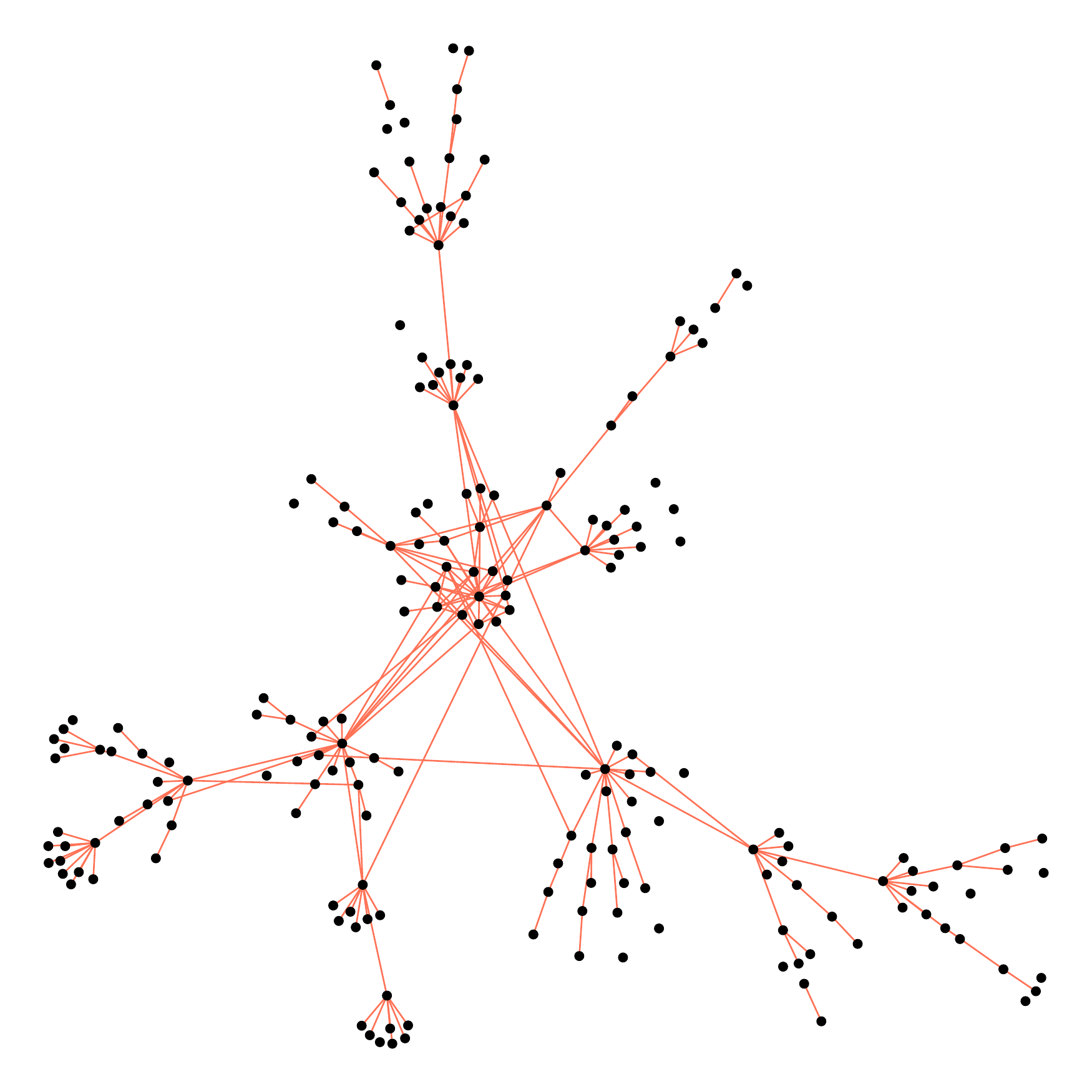}}
     }
     \centerline{
     \subfigure[]{
     \includegraphics[width=0.45\columnwidth]{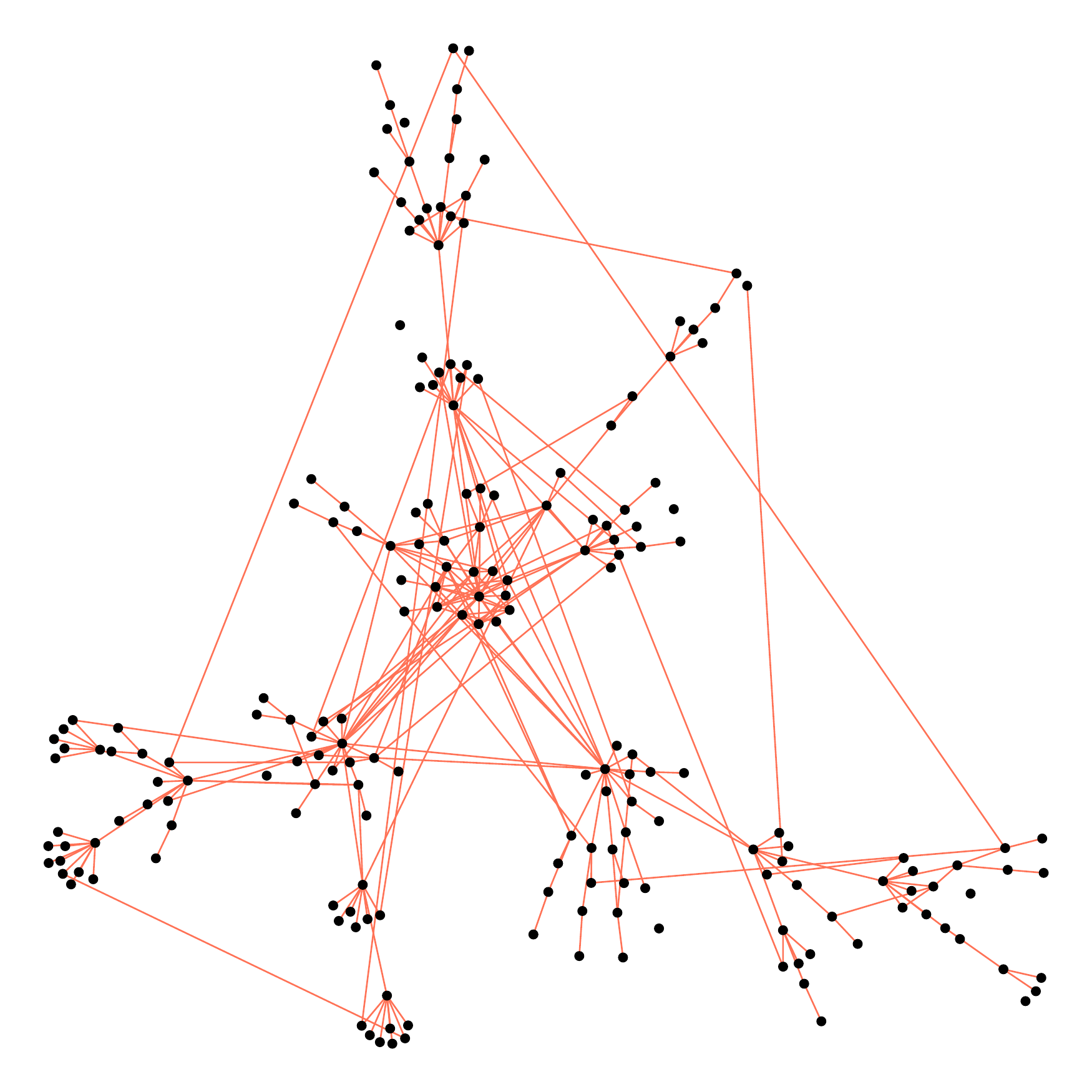}}
     \subfigure[]{
     \includegraphics[width=0.45\columnwidth]{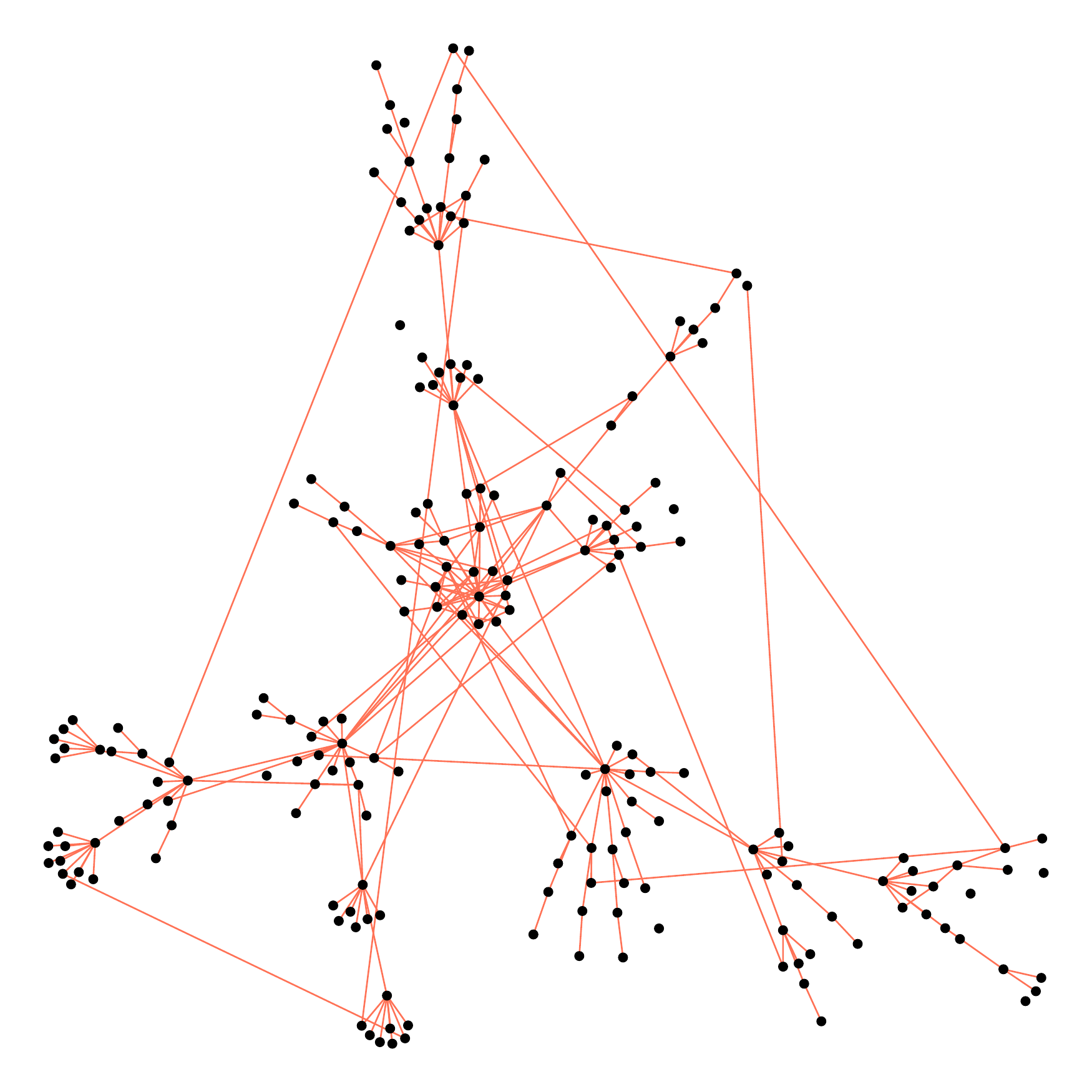}}
     }
     \caption{Recovered graphs by the \thav{} estimator using different values of $C$ for a random graph of size $d=200$ and $n=400$ samples. The graphs in (a) to (d) depict the resulting graphs using $C\in\{0.5,\; 0.6, \; 0.7, \, 0.8\}$, respectively. The corresponding $F_1$-scores are $0.75,\; 0.83, \; 0.80, \; 0.85$, respectively.}
     \label{fig:exadifferentClast}
 \end{figure*}
 \begin{figure}
    \centering
    \includegraphics[width=\linewidth]{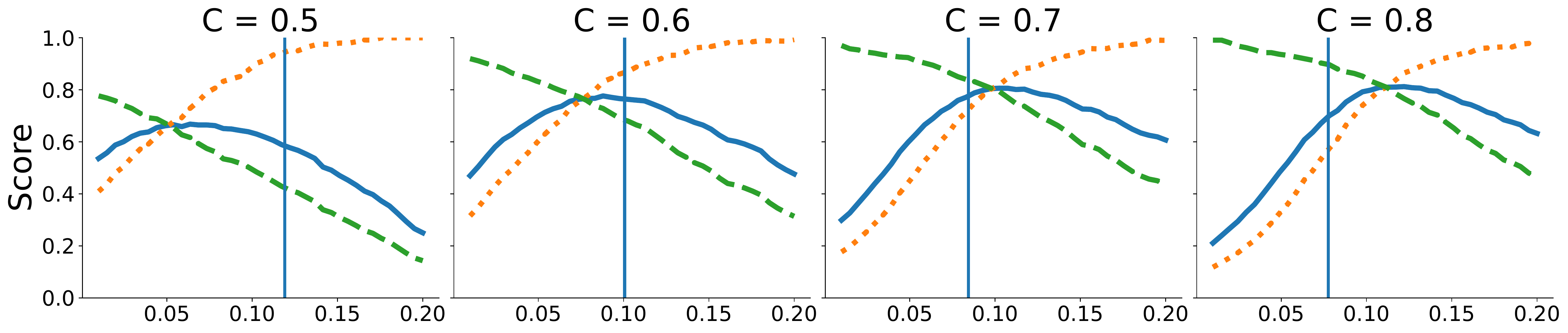}
    \includegraphics[width=\linewidth]{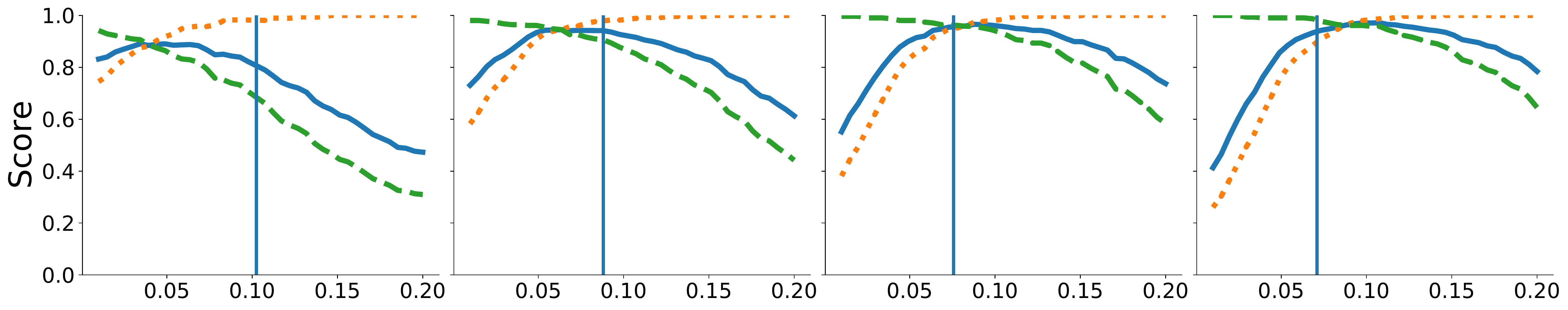}
    \includegraphics[width=\linewidth]{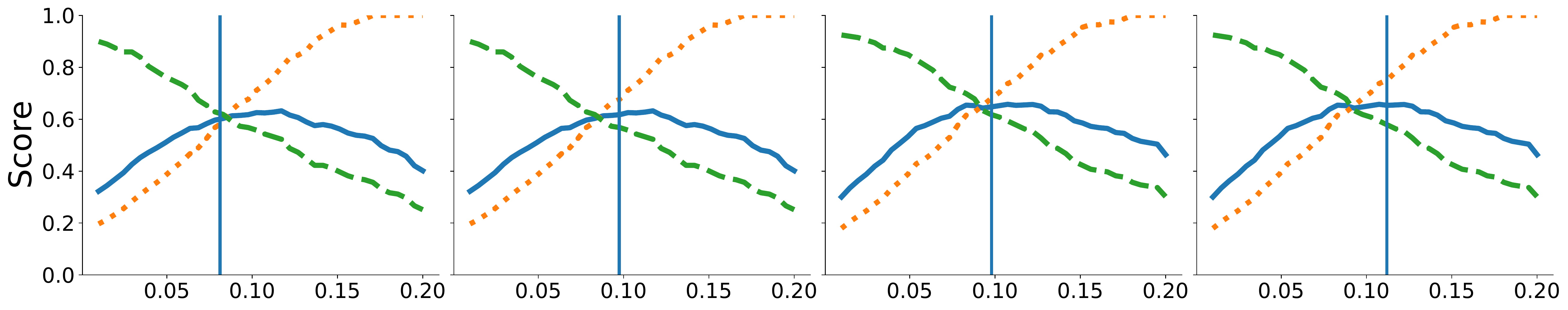}
    \includegraphics[width=\linewidth]{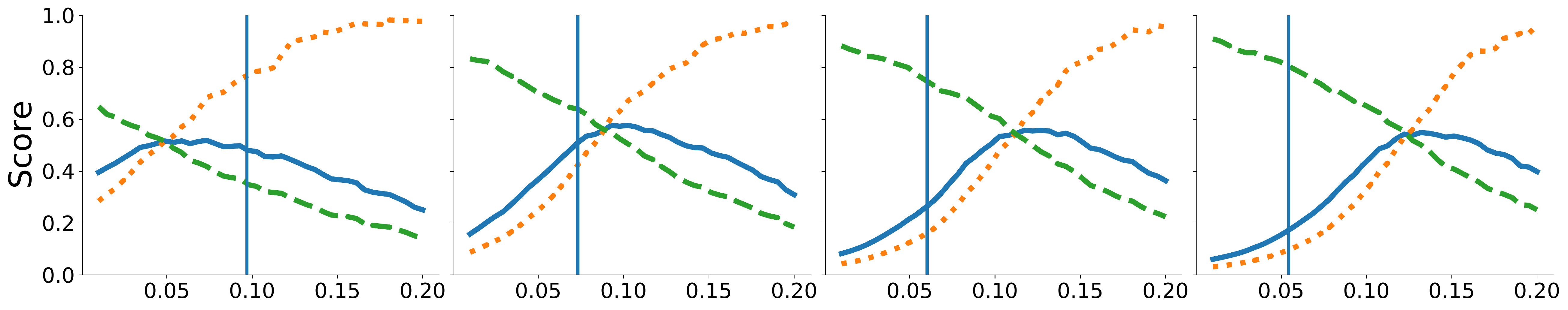}
    \includegraphics[width=\linewidth]{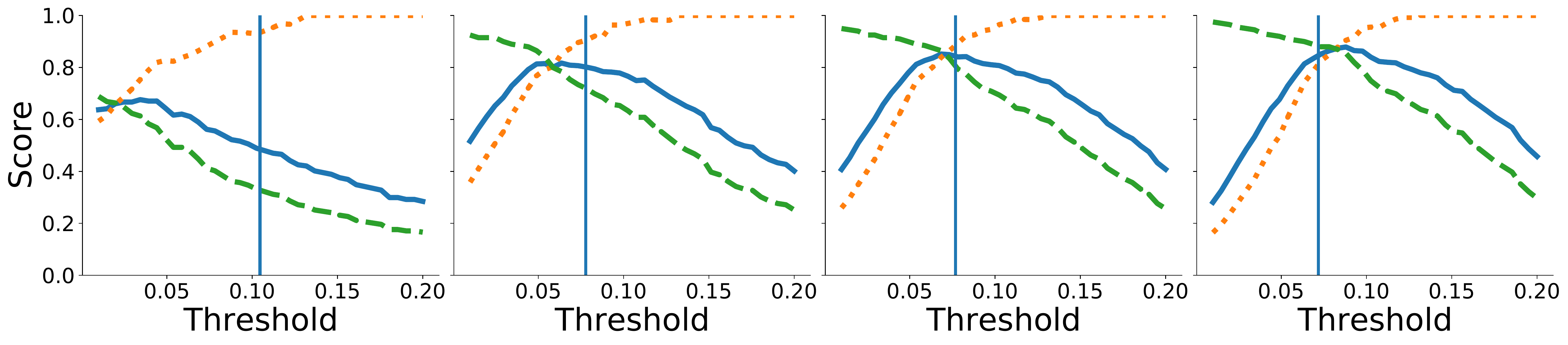}
    
    \caption{The $F_1$-score (blue, solid), precision (yellow, dotted), and recall (green, dashed) of a thresholded \av{} estimator in various settings. From top to bottom, we consider the settings $n=200,\, d=300$, and $n=400, \, d=200$ for a random graph, $n=300, \, d=200$, and $n=200,\,  d=300$, and $n=400,\,  d=200$ for a scale-free graph in dependence of the thresholds. The vertical line depicts the proposed threshold $t=C\hat{r}$ corresponding to the \thav{} estimator.}
    \label{bestthreshold_supp}
\end{figure}

\begin{figure*}
    \centerline{
    \subfigure[Random graph]{\includegraphics[width=\linewidth]{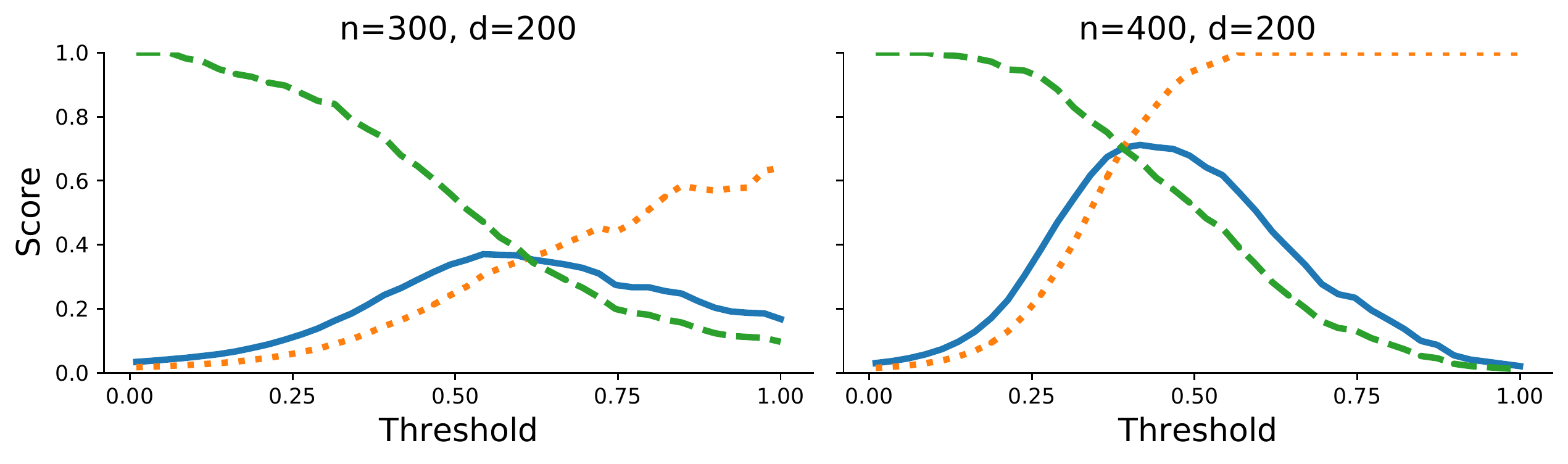}}
    }
    \centerline{
    \subfigure[Scale-free graph]{\includegraphics[width=\linewidth]{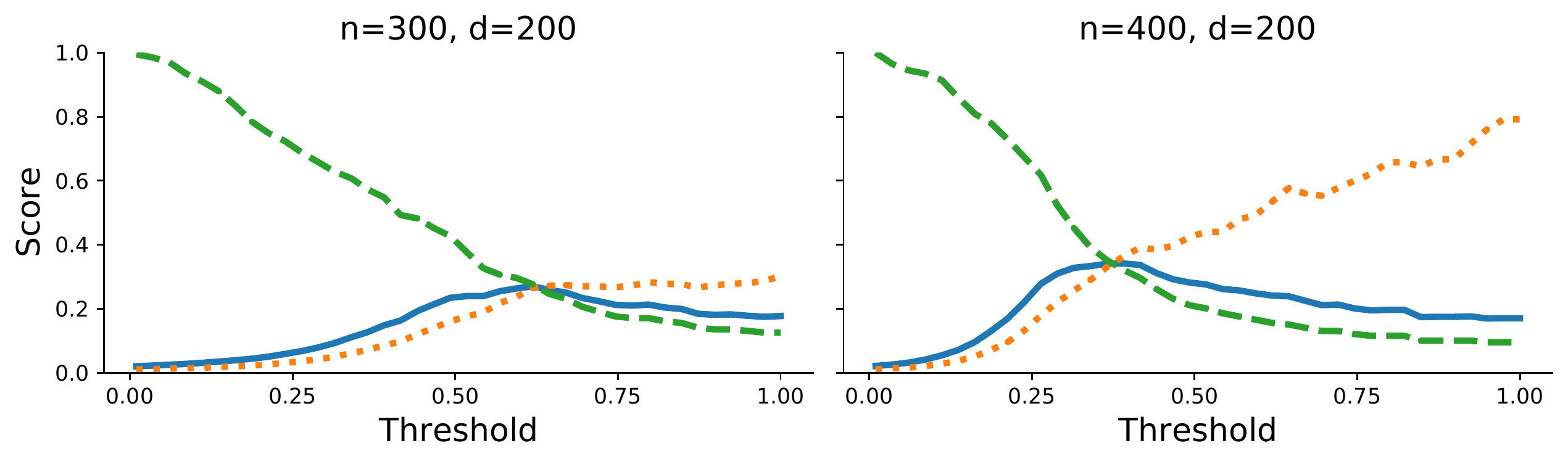}}
    }
    
    \caption{The $F_1$-score (blue, solid), precision (yellow, dotted), and recall (green, dashed) of a thresholded unregularized estimator in various settings. Note that the unregularized estimator does not exist in the case $d>n$, hence we exclude the setting $d=300,\, n=200$ from this experiment. }
    \label{fig:mlethresholding}
\end{figure*}

\begin{figure*}
 \centerline{
 \subfigure[$n=300, \, d=200.$ The resulting $F_1$-scores are $0.86$ for the \thav{} rSME and $0.84$ for the \thav{} sf-glasso.]{
    \includegraphics[width=0.3\columnwidth]{plots_AISTATS/random/true_graphn300d200.pdf}  \includegraphics[width=0.3\columnwidth]{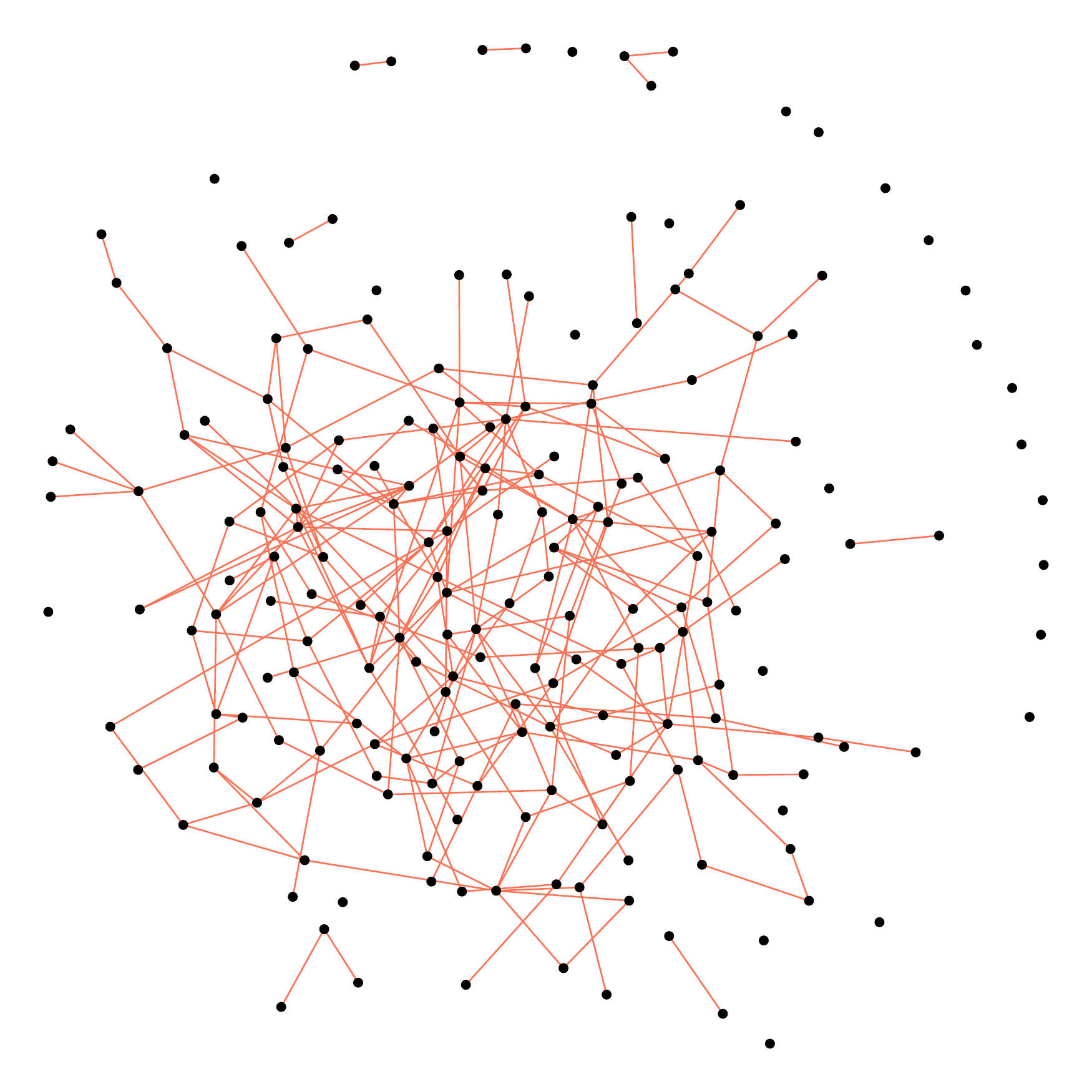} \includegraphics[width=0.3\columnwidth]{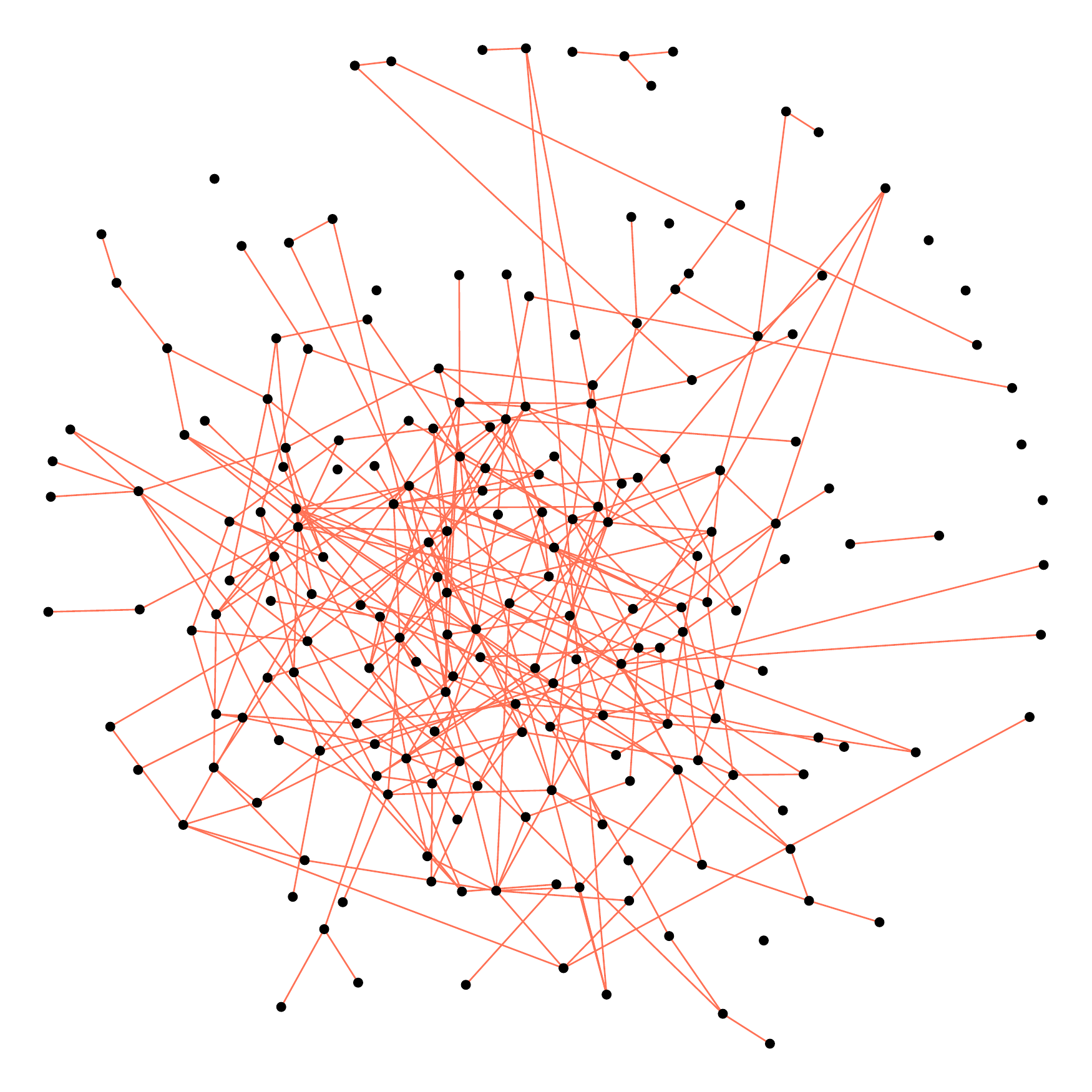}  }}
\centerline{
\subfigure[$n=200, \, d=300.$ The resulting $F_1$-scores are $0.82$ for the \thav{} rSME and $0.12$ for the \thav{} sf-glasso.]{
     \includegraphics[width=0.3\columnwidth]{plots_AISTATS/random/true_graphn200d300.pdf}  \includegraphics[width=0.3\columnwidth]{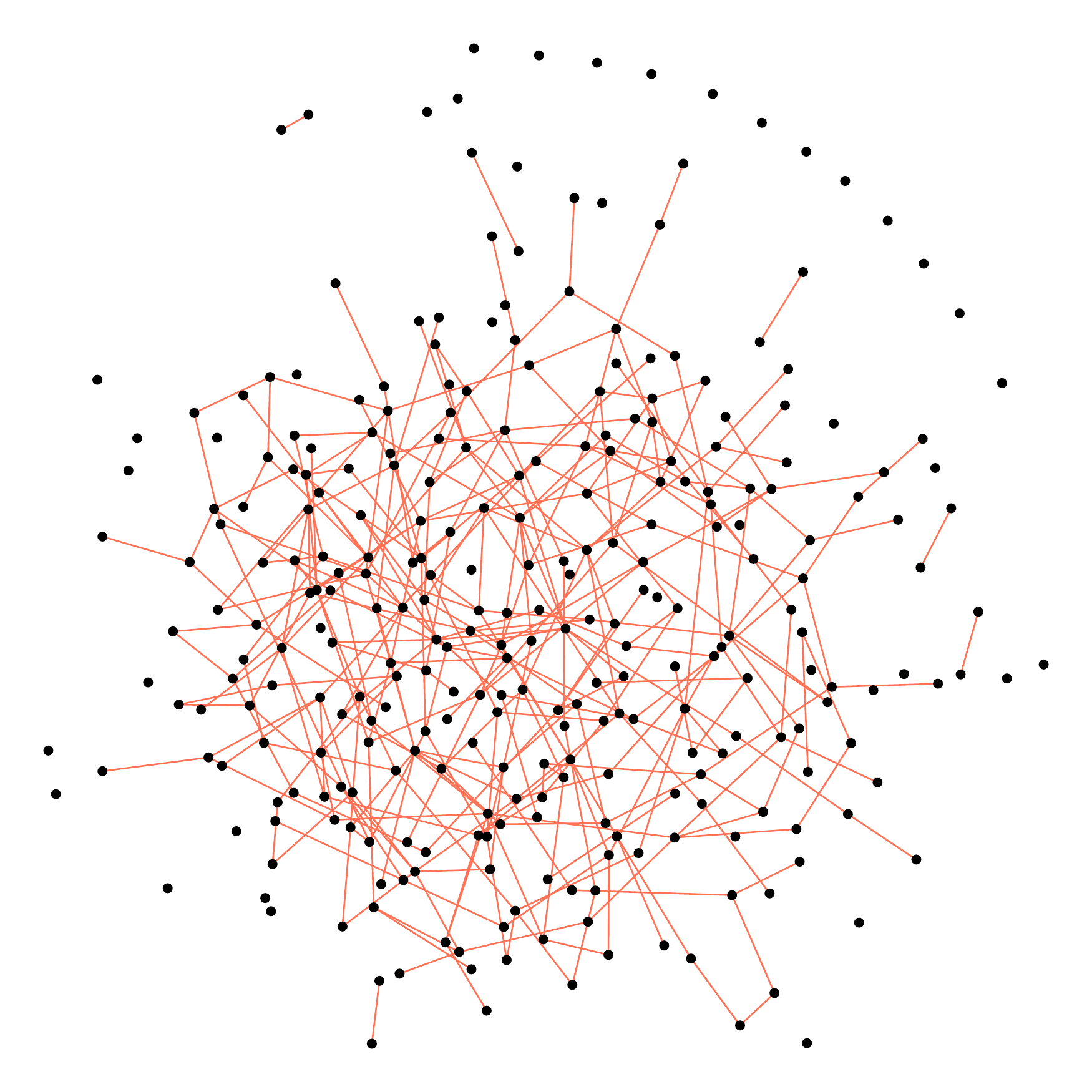} \includegraphics[width=0.3\columnwidth]{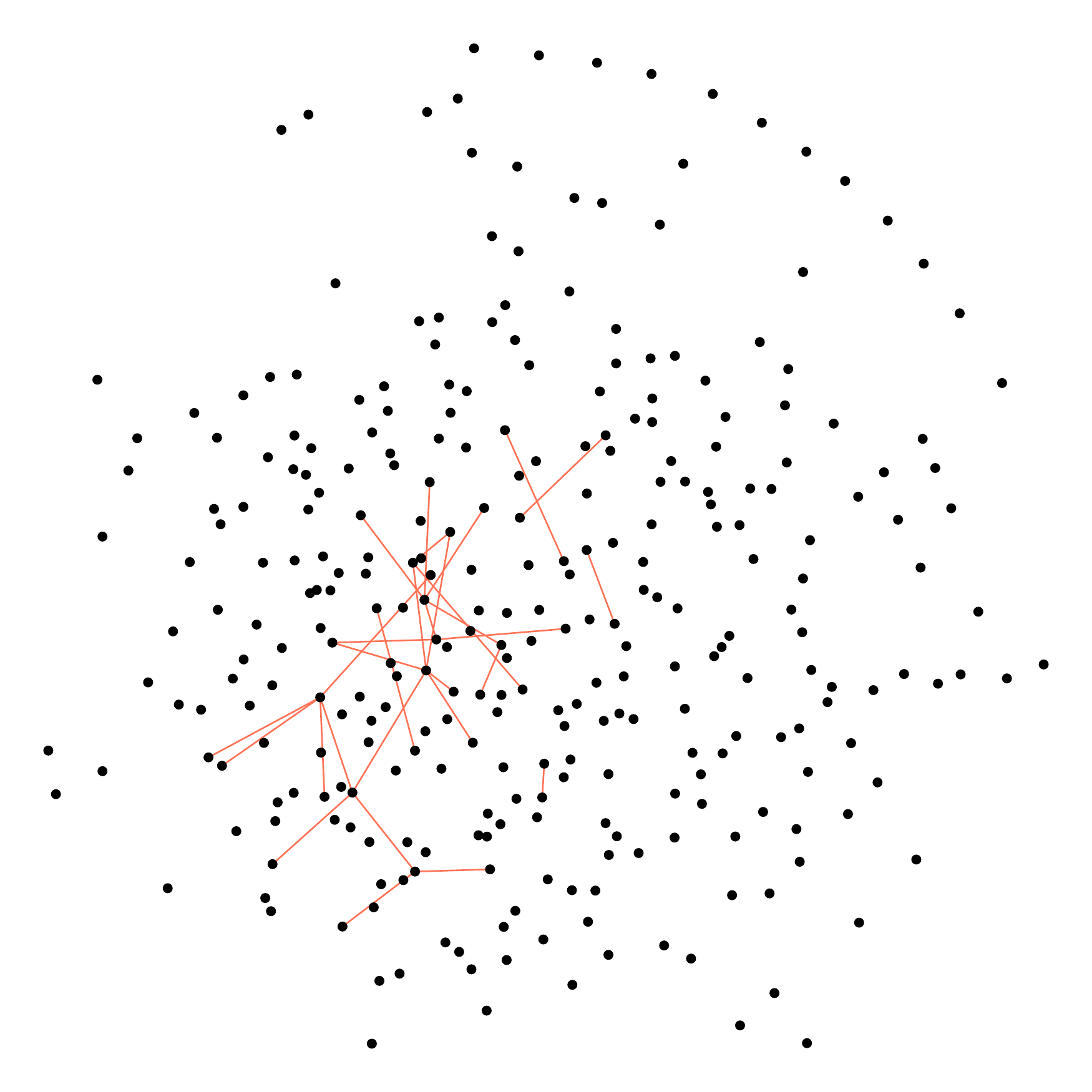} }
     }
\centerline{
\subfigure[$n=400, \, d=200.$ The resulting $F_1$-scores are $0.93$ for the \thav{} rSME and $0.62$ for the \thav{} sf-glasso.]{
     \includegraphics[width=0.3\columnwidth]{plots_AISTATS/random/true_graphn400d200.pdf}  \includegraphics[width=0.3\columnwidth]{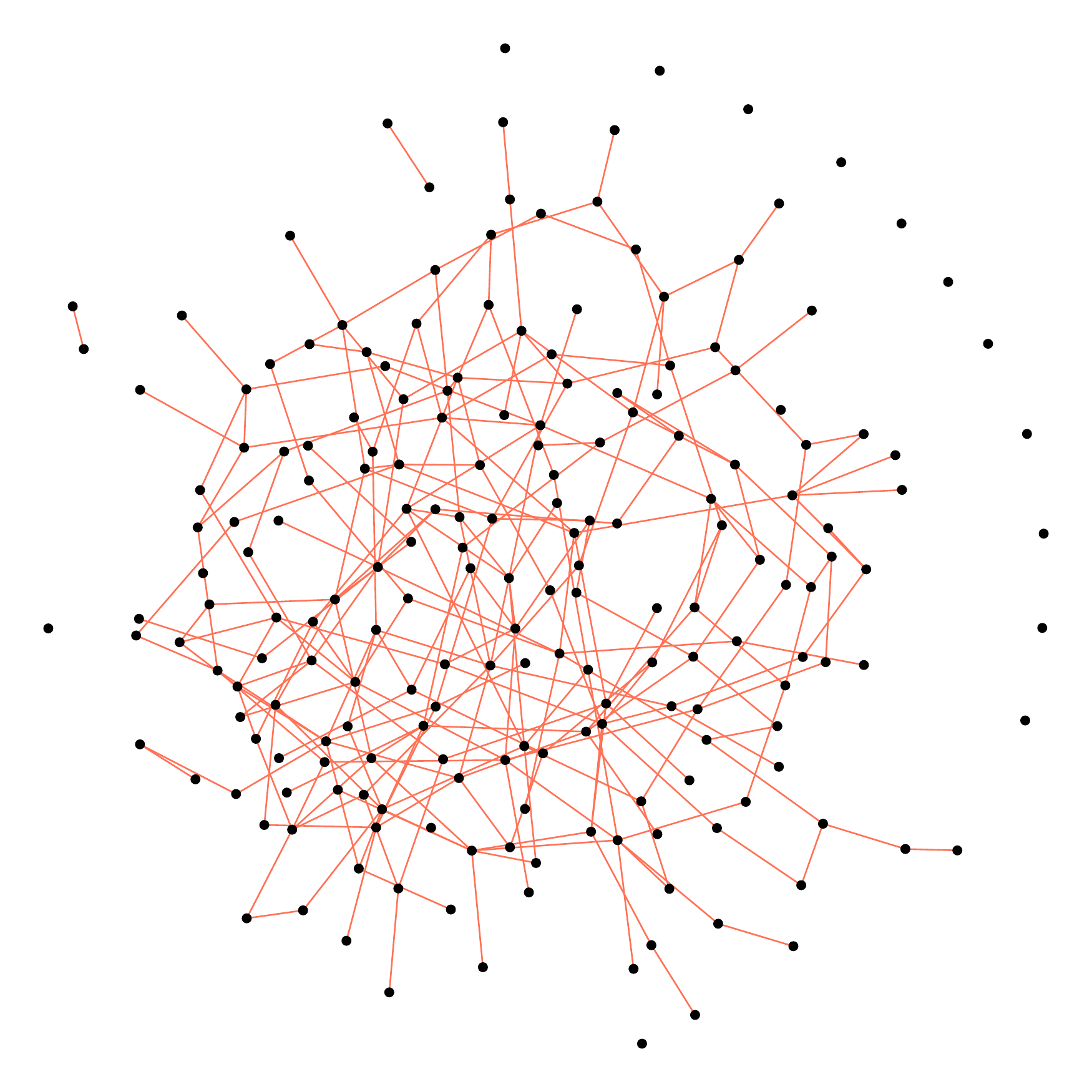} \includegraphics[width=0.3\columnwidth]{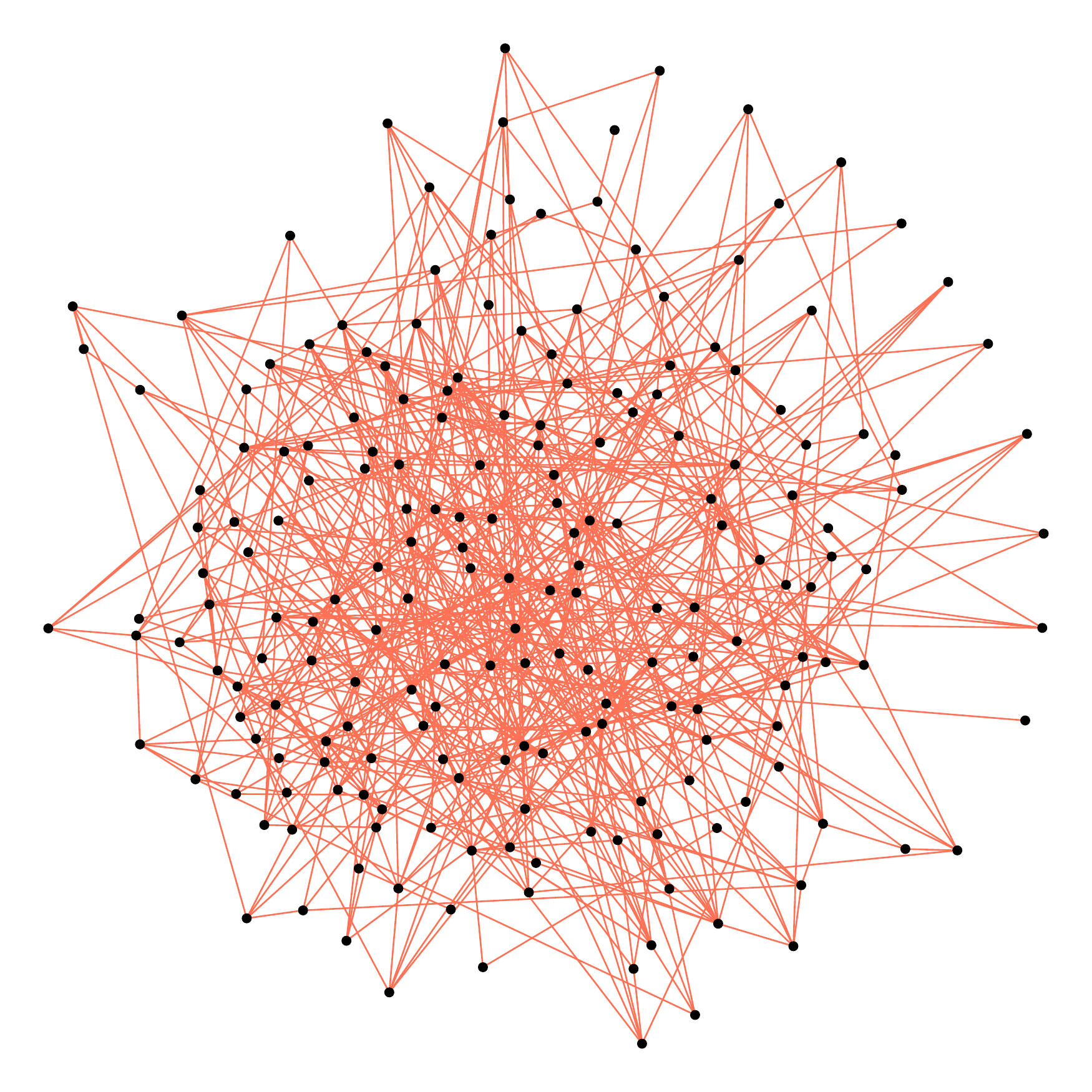} }}
     \caption{
     Examples of graph recovery for a random graph of varying size and number of samples. Each row depicts (from left to right) the true graph, the \thav{} rSME, and the \thav{} sf-glasso, respectively.}
     \label{fig:otherav_rand}
\end{figure*}

\begin{figure*}
 \centerline{
     \subfigure[$n=300, \, d=200.$ The resulting $F_1$-scores are $0.77$ for the \thav{} rSME and $0.72$ for the \thav{} sf-glasso.]{  \includegraphics[width=0.3\columnwidth]{plots_AISTATS/scale-free/true_graphn300d200.pdf}  \includegraphics[width=0.3\columnwidth]{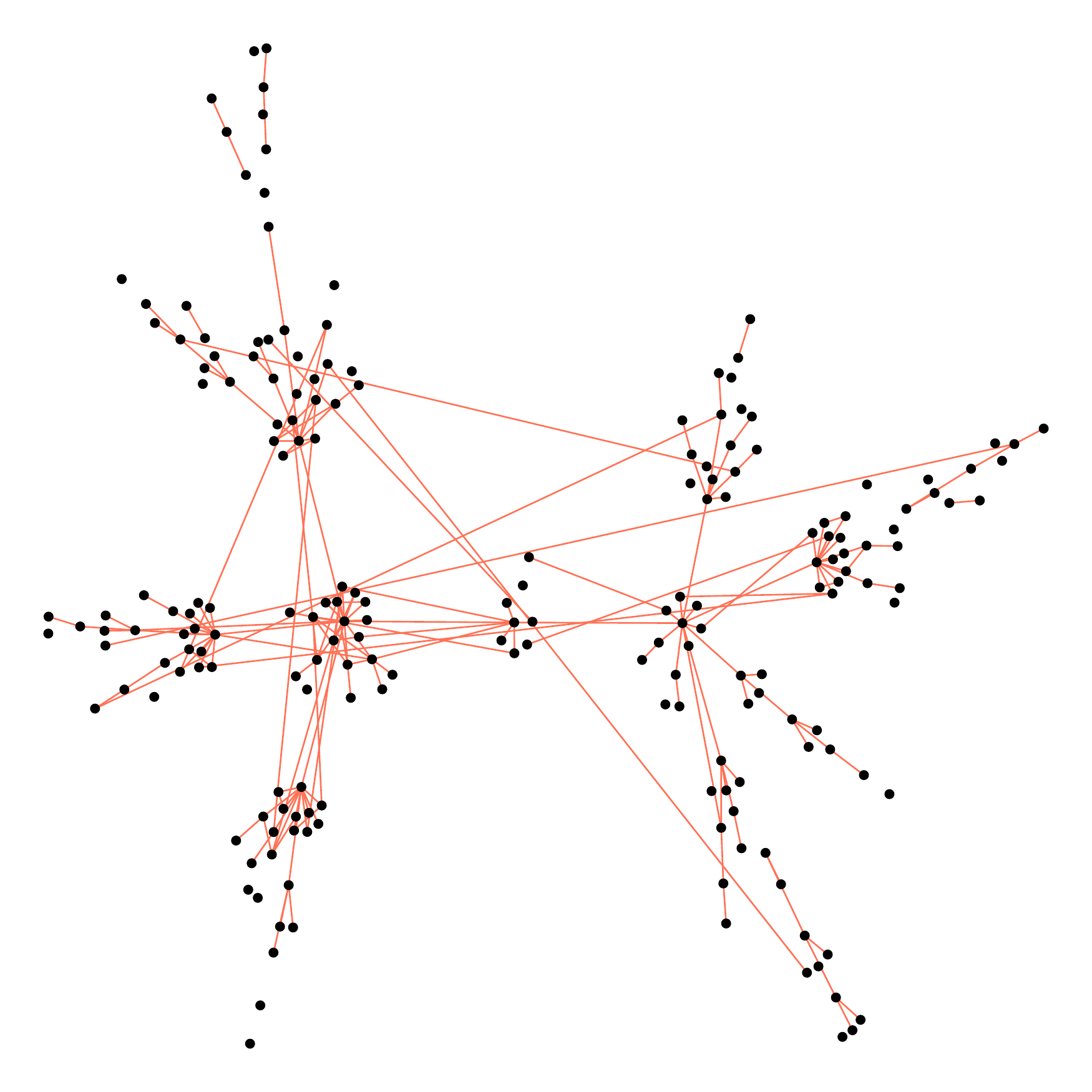} \includegraphics[width=0.3\columnwidth]{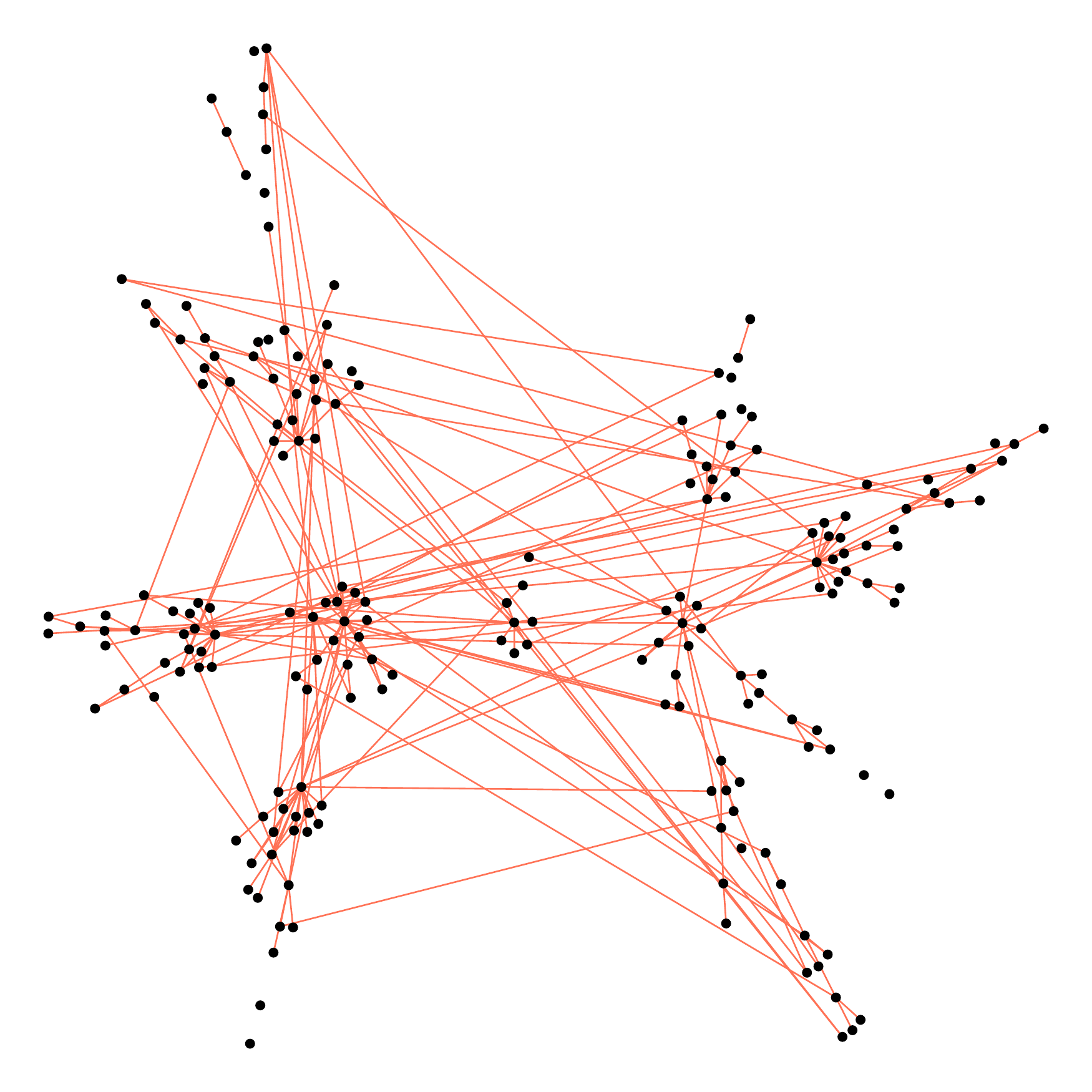} }}
 \centerline{
     \subfigure[$n=200, \, d=300.$ The resulting $F_1$-scores are $0.21$ for the \thav{} rSME and $0.23$ for the \thav{} sf-glasso.]{  \includegraphics[width=0.3\columnwidth]{plots_AISTATS/scale-free/true_graphn200d300.pdf}  \includegraphics[width=0.3\columnwidth]{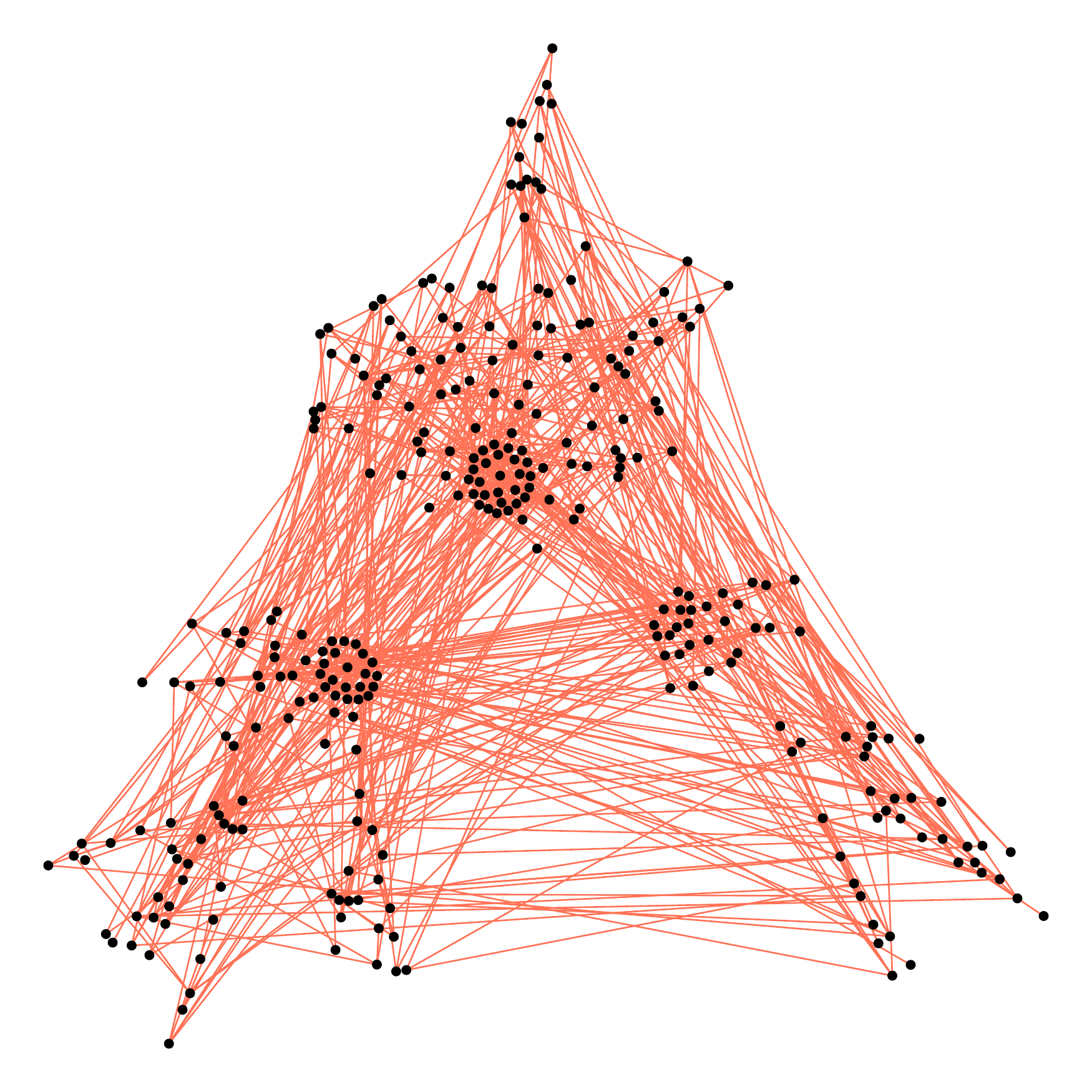} \includegraphics[width=0.3\columnwidth]{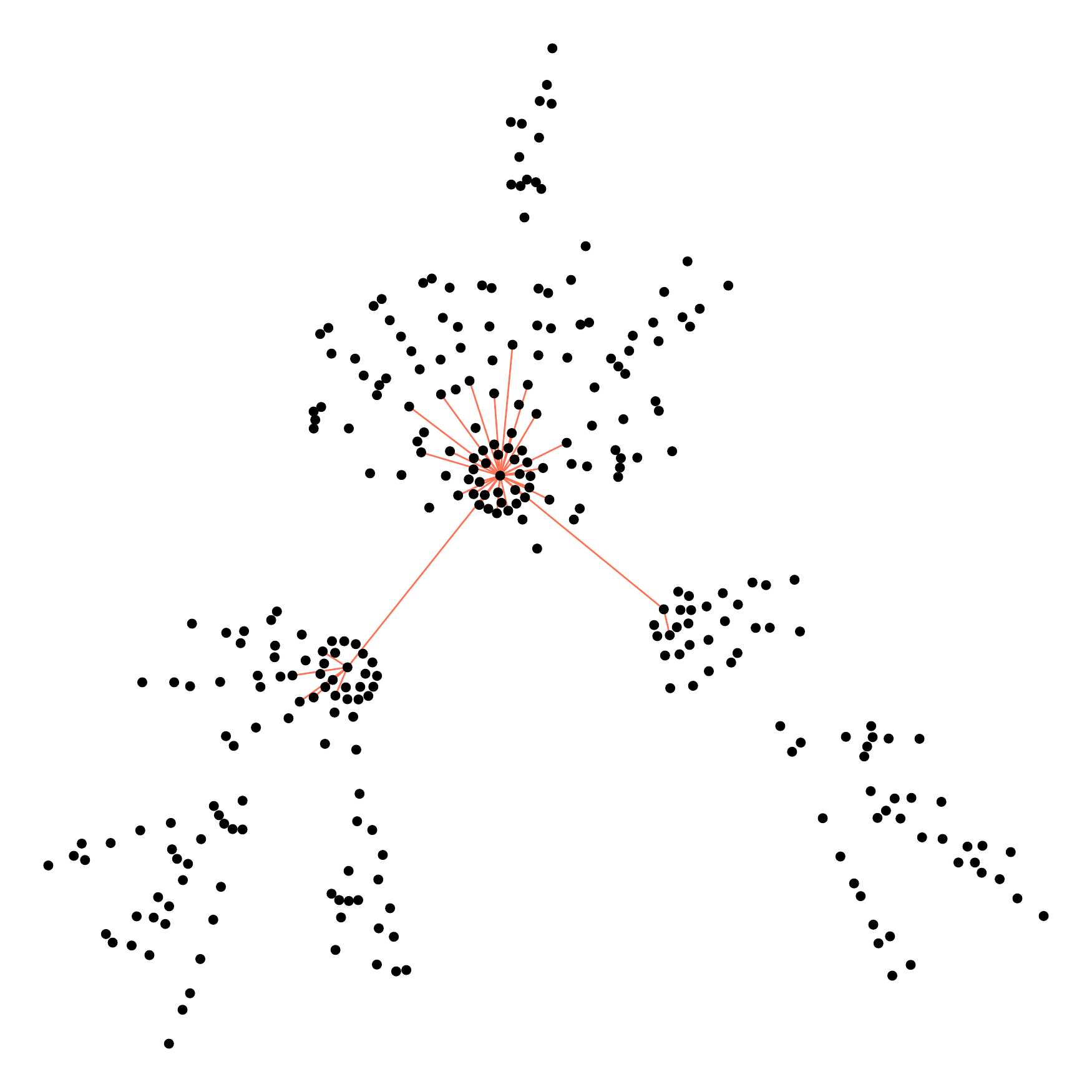} }}
 \centerline{
     \subfigure[$n=400, \, d=200.$ The resulting $F_1$-scores are $0.65$ for the \thav{} rSME and $0.70$ for the \thav{} sf-glasso.]{  \includegraphics[width=0.3\columnwidth]{plots_AISTATS/scale-free/true_graphn400d200.pdf}  \includegraphics[width=0.3\columnwidth]{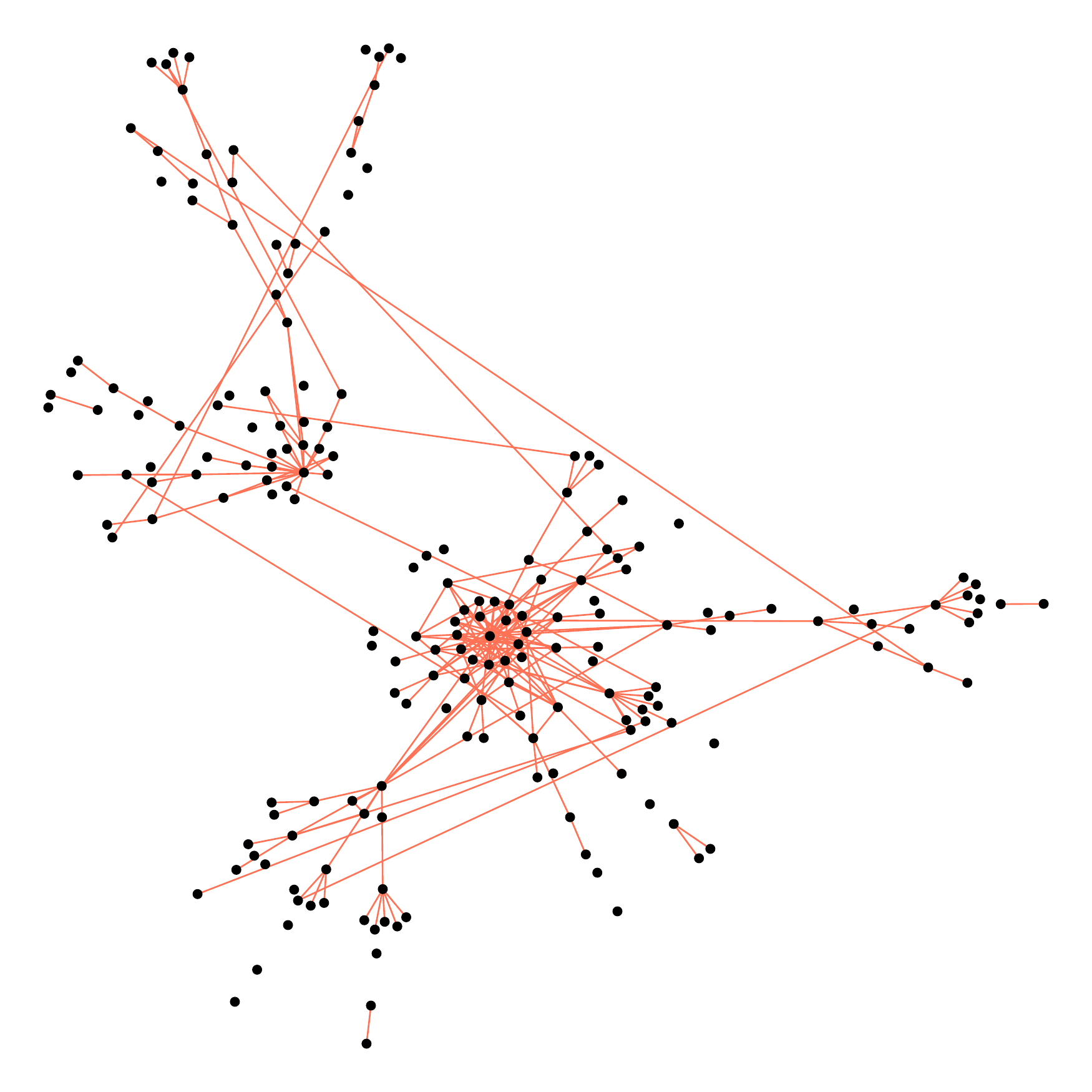} \includegraphics[width=0.3\columnwidth]{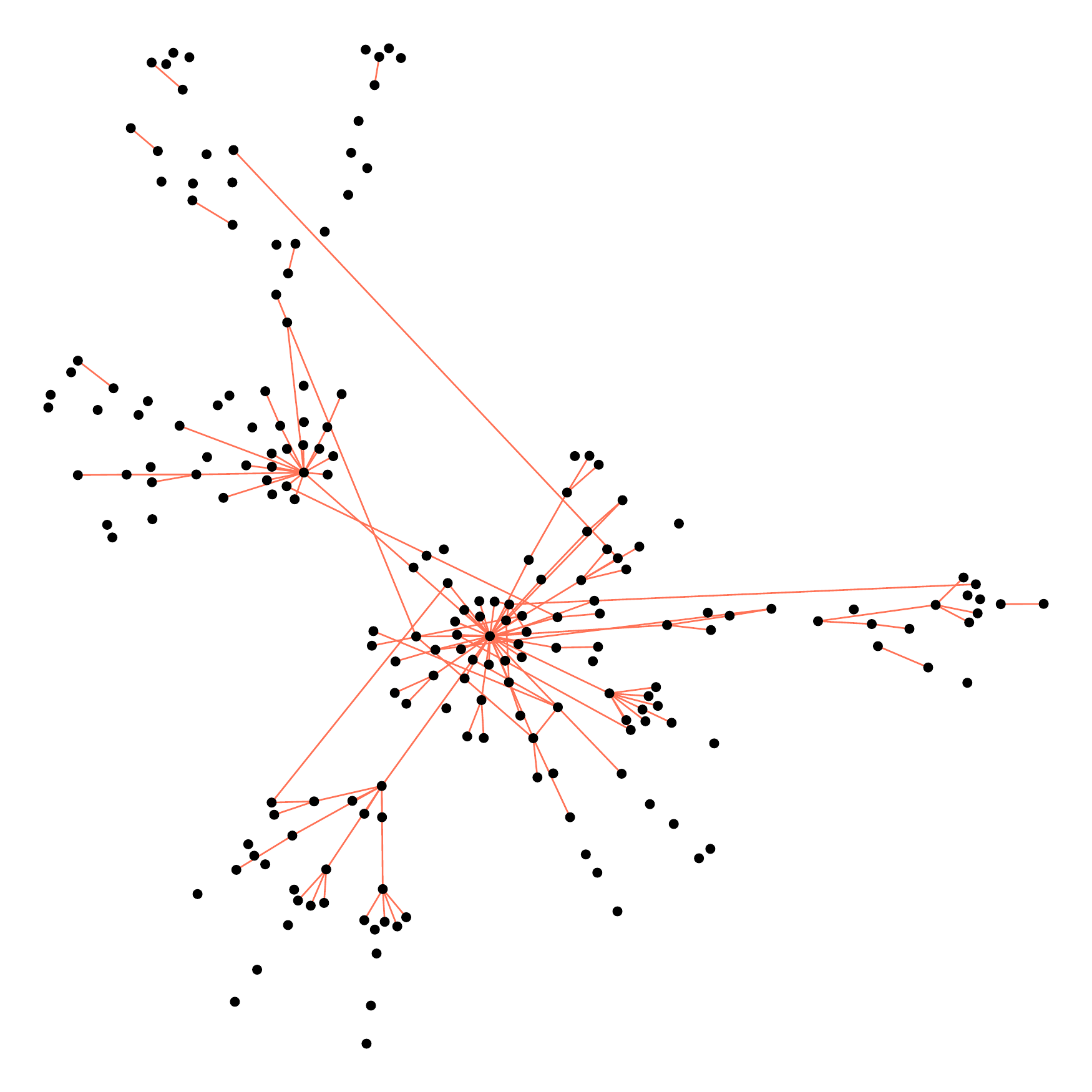} }}
     \caption{
     Examples of graph recovery for a scale-free graph of varying size and number of samples. Each row depicts (from left to right) the true graph, the \thav{} rSME, and the \thav{} sf-glasso, respectively.}
     \label{fig:otherav_sf}
 \end{figure*}
 
\begin{figure}[ht]
    \vskip 0.2in
    \centerline{
         \includegraphics[width=\columnwidth]{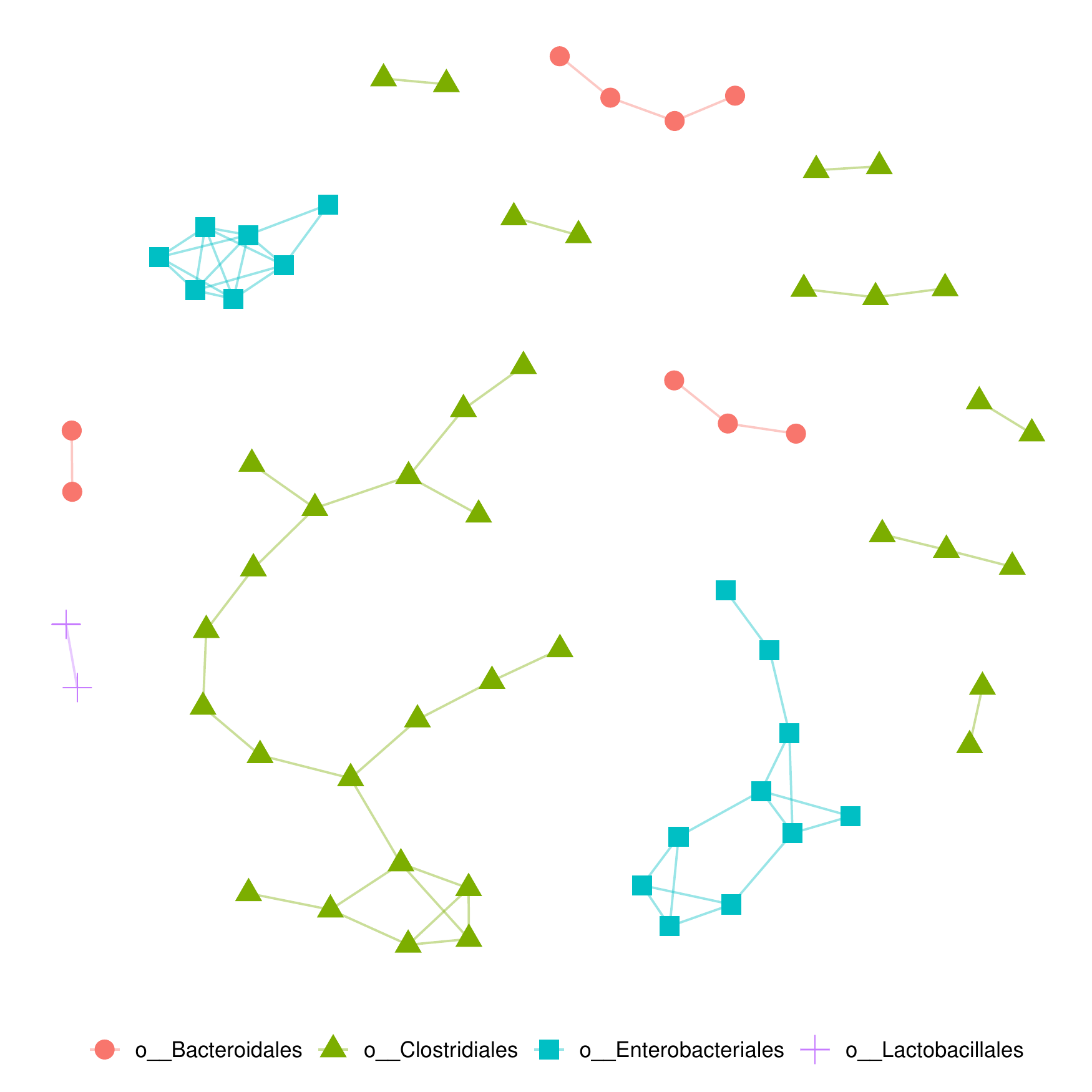}}
    \caption{Recovered microbial network based on the American gut data using \thav{} with the threshold $t:=C\hat{r}$. To avoid too large graphics, we exclude isolated vertices. The color and the shape of a node imply the biological cluster of each OTU.}
    \label{fig:gutthav2} 
\end{figure}

\end{document}